\documentclass[smallextended]{svjour3}       
\smartqed  
\usepackage[]{natbib}
\usepackage{url}
\usepackage{graphicx}
\usepackage{subcaption}
\usepackage{amsfonts}
\usepackage[fleqn]{amsmath}
\usepackage{mathtools}
\usepackage{amssymb}
\usepackage[utf8]{inputenc}
\usepackage[T1]{fontenc}
\usepackage[ruled]{algorithm2e}
\usepackage{algorithmic}
\SetAlFnt{\small}
\SetAlCapFnt{\small}
\SetAlCapNameFnt{\small}
\SetAlCapHSkip{0pt}
\IncMargin{-\parindent}

\newcommand{\given}{\,|\,}

\renewcommand{\vec}[1]{\ensuremath{\mathbf{#1}}}

\newcommand{\mat}[1]{\ensuremath{\mathbf{#1}}}
\DeclareMathOperator*{\argmax}{arg\,max}

\newcommand{\newSkip}{-0.1in}
\newcommand{\newSpace}{-0.1in}

\begin{document}

\title{Active Learning Methods based on Statistical Leverage Scores
}


\author{Cem Orhan         \and
        Oznur Tastan 
}


\institute{C. Orhan \at
              Department of Computer Engineering, Bilkent University, 06800 Ankara \\
              Tel.: +90-538-2038077\\
              \email{cem.orhan@yahoo.com}           
           \and
           	  O. Tastan \at
              Faculty of Natural Sciences and Engineering, Sabanci University, 34956 Istanbul \\
              Tel.: +90-216-483883\\
              \email{otastan@sabanciuniv.edu}           
}

\date{Received: date / Accepted: date}

\maketitle

\begin{abstract}
In many real-world machine learning applications, unlabeled data are abundant whereas class labels are expensive and scarce. An active learner aims to obtain a model of high accuracy with as few labeled instances as possible by effectively selecting useful examples for labeling. We propose a new selection criterion that is based on statistical leverage scores and  present two novel active learning methods based on this criterion: ALEVS for querying single example at each iteration and DBALEVS for querying a batch of examples. To assess the representativeness of the examples in the pool, ALEVS and DBALEVS use the statistical leverage scores of the kernel matrices computed on the examples of each class. Additionally, DBALEVS selects a diverse a set of examples that are highly representative but are dissimilar to already labeled examples through maximizing a submodular set function defined with the statistical leverage scores and the kernel matrix computed on the pool of the examples. The submodularity property of the set scoring function let us identify batches with a constant factor approximate to the optimal batch in an efficient manner. Our experiments on diverse datasets show that querying based on leverage scores is a powerful strategy for active learning. 

\end{abstract}
\section{Introduction}
\label{intro}
\vspace{-2mm}
Learning a supervised model with good predictive performance requires a sufficiently large amount of labeled data. In many real-world applications, while the unlabeled data is abundant and easily obtained,  acquiring class labels is time-consuming, costly or requires expert knowledge,  i.e. medical image analysis requires pathology experts. Additionally, labeling all data is often redundant, as some examples do not add further information to the already labeled ones. An active learner aims to learn a model using as few training examples as possible to achieve good model performance \cite{Settles2009}. Active learning has been shown to be effective in a variety of domains, including text categorization \cite{Tong2001}, computer vision \cite{kapoor2007active} and medical image analysis \cite{hoi2006batch}.

One common setting for active learning is the pool-based active learning, where a large number of unlabeled examples together with a small number of labeled examples are initially available \cite{Settles2009}. The learner interacts with an oracle (i.e., human expert) that provides labels when queried. At each step, the active learner chooses  example(s) intelligently from the unlabeled pool and request labels of these queries from the oracle. Next, the training data is augmented with the newly labeled data, and the classifier is retrained. This iterative procedure is repeated until a stopping criterion (i.e., budget constraint, desired accuracy, etc.) is met.  In the sequential mode, the active learner solicits the label for a single instance \cite{Lewis94}. In cases where the training process is hard and/or there are multiple annotators available that could work in parallel, retraining the learner at each step would be inefficient. In the batch mode active learning, the learner requests the labels of a set of examples at once at each step. In both sequential and batch mode active learning, the critical step is the proper selection of label(s) to probe. 

In this work, we focus on binary classification in a supervised learning setting for the pool-based learning scenario. We propose a new criterion to assess the representativeness of an example in the pool that is based on statistical leverage scores. We  develop two novel active learning algorithms which we shall abbreviate as ALEVS and DBALEVS.  ALEVS (\underline{A}ctive Learning by Statistical \underline{Lev}erage \underline{S}cores) is a sequential active learning algorithm that queries one example at a time. \underline{D}iverse \underline{B}atch-mode \underline{A}ctive Learning by Statistical \underline{Lev}erage \underline{S}ampling (\texttt{DBALEVS}) is a batch mode alogrithm which selects not only a batch of examples that are influential in the class but also a set that is diverse with respect to the already labeled examples and to the other examples in the batch. We achieve this by encoding these properties in a set function that is submodular and monotonically non-decreasing; therefore, we can utilize a greedy submodular maximization algorithm that is provably near-optimal. 
 
 \vspace{-5mm}
\section{ Problem set up and Notations}
\label{sec:oroblemsetup}
\vspace{-2mm}
To explore the use of leverage scores for querying examples, we focus on supervised learning binary classification in the pool-based active learning scenario, where a small number of labeled examples are provided along with a large pool of unlabeled examples. The objective is to learn an accurate classifier, $h: \mathcal{X} \rightarrow \mathcal{Y}$, where $\mathcal{X}$ denotes the instance space and $\mathcal{Y}$ is the set of class labels. We denote the training data with $ \mathcal{D}= \{ (\vec{x}_i, y_i)\}_{i=1}^{n}$wherein $i$-the example's feature vector is denoted by $\vec{x}_i \in \mathbb{R}^{d}$ and the class label with $y_i \in \{-1, 1\}$. We propose solutions for the sequential learning and the batch mode setting based on leverage scores.

In the sequential learning scenario, the active learner iteratively selects one example from the unlabeled pool and queries its label. We denote the queried example with $q$ and its feature vector with $\vec{x}_\mathrm{q}$, The labeling oracle, $\mathcal{O}$, upon receiving the labeling request for $\vec{x}_\mathrm{q}$ responds with the true label $y_\mathrm{q}$. We assume uniform cost for labeling across examples. We denote the labeled set of training examples at iteration $t$ with $\mathcal{D}_\mathrm{l}^t$ and the set of unlabeled examples with $\mathcal{D}_\mathrm{u}^t$ while $\mathcal{D}_\mathrm{l}^t$ comprises labeled $(\mat{x}_i, y_i)$ pairs, $\mathcal{D}_\mathrm{u}^t$ include only $\mat{x}_i$. 

The objective is to attain a good accuracy classifier $h^*$ by minimizing the number of examples queried  thereby reducing the labeling cost. In the batch mode active learning, instead of selecting a single example at each iteration, batches of size $b$ are sequentially picked where $b$ is specified \textit{a priori} by the user. We will refer the set of examples queried at iteration $t$ with the set $S_q^t$. Notations used throughout the paper is provided in Appendix \ref{app:notation} Table ~\ref{tbl:notations}.

The rest of this article is organized as follows.  In the following section, we briefly review the related work. In Section \ref{sec:lev}, we introduce statistical leverage scores, discuss their use in the literature and present the idea of querying based on statistical leverage scores. In Section \ref{sec:alevs}, \texttt{ALEVS} presents active learning method for querying  a single example at each active learning while Section \ref{sec:dbalevs} presents the batch mode active learning algorithm \texttt{DBALEVS}. Experimental results are reported in Sections \ref{sec:alevsresults} and in \ref{sec:dbalevsresults}. Lastly, Section \ref{sec:conclusion} concludes.

\vspace{-4mm}
\section{Related Work}
\label{sec:rw}
\vspace{-2mm}
%

For selecting a query, there are two main approaches proposed in the literature \cite{Settles2009}. The first one is to query examples based on their informativeness. Uncertainty sampling, in which the learner queries the example with the most uncertain class label, is one of the most used such methods \cite{Lewis94}. The uncertainty can be assessed by the distance to the decision boundary \cite{Tong2001}, through label entropy \cite{Lewis94} or by the disagreement of the ensemble of classifiers trained with the current label set \cite{seung1992query,Freund1997}. One common drawback for these algorithms, particularly at early iterations, is that the classifier is uncertain about many points and the decision boundary formed with the classifier is not reliable as it relies on a limited set of examples available for training \cite{roy2001}. Furthermore, these approaches introduce a sampling bias, and the methods fail to exploit the unlabeled data distribution \cite{dasgupta2011two}. Others choose the most informative example
that minimizes the model variance  \cite{mackay1992information}. To assess the informativeness of an example, \cite{yu:icml06expdesign} extends the classical experimental design to active learning and aims at finding examples that will lead to best predictions.The second set of approaches select instances that are representative of the data distribution \cite{xu2003representative,nguyen2004active,xu2007incorporating,dasgupta2008hierarchical}. These algorithms’ success heavily depends on the clustering algorithm employed.  There are also hybrid approaches that combine the representative strategies in a single framework. \cite{settles2008analysis} uses a weighting strategy that incorporates the similarity of the example to the other points based on its informativeness. A similar method, using density and entropy, is applied to a text classification problem  \cite{zhu:2008}.  QUIRE \cite{quire} optimizes an objective function wherein both the informativeness and representativeness of the examples are considered simultaneously.

Adapting the single query selection to batch mode setting by simply choosing the top $b$ examples, where $b$ is the number of elements in batch, does not account for the fact that there can be redundant information among the selected set of examples. Several batch mode active learning strategies have been proposed \cite{xu2007incorporating,zhu:2008,guo2008discriminative,hoi2009batch,guo2010active,chattopadhyay2013batch,gu2014batch,yang2015multi,wang2015querying}.
Among these, there are methods that directly optimize an objective function that represents a good quality batch. The method introduced in \cite{zhu:2008} selects top $b$ examples that satisfies an objective function combining density and entropy. Guo and Schuurmans \cite{guo2008discriminative}  select a batch of examples which achieves the best discriminative classification performance. For assembling a good batch, Guo \cite{guo2010active} selects instances that maximize the mutual information between labeled and unlabeled examples. In the work   
 \cite{yang2015multi}, the most uncertain and representative queries are selected by minimizing the empirical risk. In the batch mode setting, selecting a diverse set of examples is critical. Brinker et al. \cite{brinker2003incorporating} selects a diverse batch  using SVMs, where the diversity is measured as the angle between the hyperplane induced by the currently selected point and the hyperplanes  induced  by the previously selected points. \cite{hoi2006batch} proposes a framework that minimizes the Fisher information and solves this optimization problem using the submodular properties of the set selection function. Chen and Krause \cite{chen2013near} similarly employ submodular optimization and their approach asks for the batch with the maximum marginal gain.

\vspace{-4mm}
\section{Statistical Leverage Scores and Our Motivation}
\label{sec:lev}
\vspace{-2mm}
Statistical leverage scores of a matrix $\mat{A}$ are the squared row-norms of the matrix containing its (top) left singular vectors.    For a symmetric positive semi-definite (SPSD) matrix, the statistical leverage scores relative to the best rank-$k$ approximation to the input matrix are defined as follows \normalfont{\cite{gittens2013}}:

\begin{definition}[Leverage scores for an SPSD matrix] \label{def:lev}
Let $\mat{A}$, an arbitrary $m\times m$ SPSD matrix with the eigenvalue decomposition $\mat{A}=\mat{U}\mat{\Sigma} \mat{U}^\mathrm{T}$. $\mat{U}$ can be partitioned as   $\mat{U} =  (\mat{U}_1 \quad \mat{U}_2 )$
where $\mat{U}_1$ comprises $k$ orthonormal columns spanning the top $k$-dimensional eigenspace of $\mat{A}$. Let $\lambda_1(\mat{A}) \geq \lambda_2(\mat{A}) \geq \dots \geq \lambda_m(\mat{A})$ be the eigenvalues of $\mat{A}$ ranked in descending order. Given $\mat{A}$ and a rank parameter $k$, the statistical leverage scores of \mat{A} relative to the best rank-$k$ approximation to $\mat{A}$ is equal to the squared Euclidean norms of the rows of the $m \times k$ matrix $\mat{U}_1$:
\begin{eqnarray}
	\label{eq:levscore}
	&\ell_i := \|({\mat{U}}_1)_{(i)}\|_2^2
\end{eqnarray}
\end{definition}
for $ i \in \{1, \ldots, m\}$, where $\ell_i \in [0,1]$, and $\sum_{i=1}^{m}\ell_i = k$.

 Intuitively, leverage scores determine which columns (or rows) are most representative with respect to a rank-$k$ subspace of $\mat{A}$. They are most recently used in low-rank matrix approximation algorithms to identify influential columns of the input matrix \cite{dkm2006,gittens2013,papailiopoulos2014provable,mahoney2009cur,yang2015,wang2015provably,wang2015empirical}. Mahoney et al. \cite{dkm2006,yang2015} show that in a low-rank matrix approximation task, the column subset selection is improved if the columns of the matrices are sampled based on a probability distribution weighted by the leverage scores of the columns.  Along with these randomized algorithms, Papailiopoulos et al. \cite{papailiopoulos2014provable} demonstrate that deterministically selecting a subset of the matrix columns with the largest leverage scores results in a good low-rank matrix approximation. In another work, CUR decomposition is improved with the use statistical leverage scores \cite{mahoney2009cur}. Gittens and Mahoney \cite{gittens2013} analyze different Nystr\"om sampling strategies for symmetric positive semi-definite (SPSD) matrices and show that sampling based on leverage scores is quite effective.

\vspace{-1mm}
\begin{figure*}[]
	\centering
	\vskip \newSkip
	\begin{subfigure}[b]{0.2\textwidth}
		\caption{$\mat{A}$}
		\vspace{\newSpace}
		\includegraphics[width=0.6\textwidth]{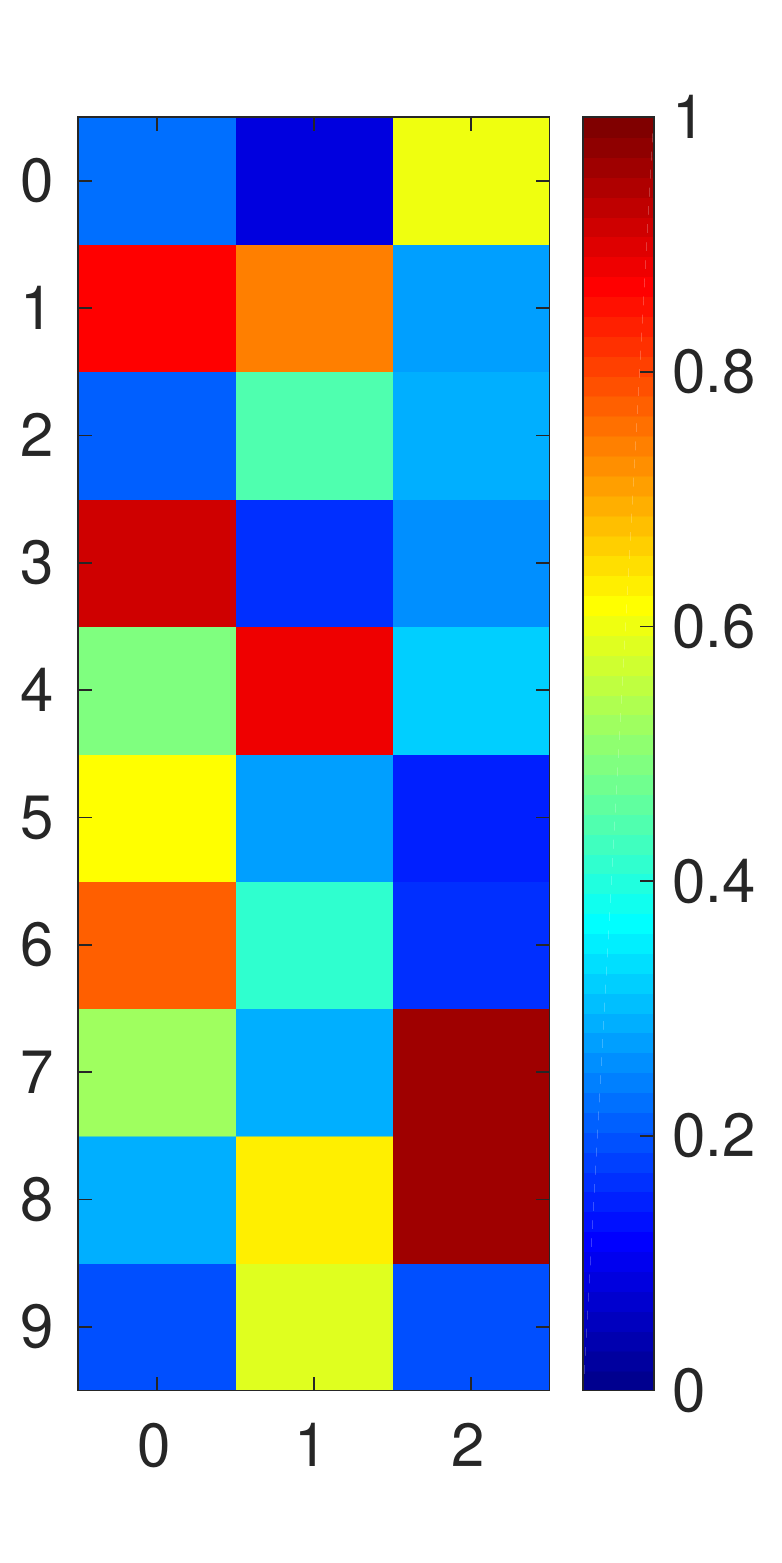}
		\label{A}
	\end{subfigure}
	~
	\begin{subfigure}[b]{0.4\textwidth}
		\caption{$\mat{K}_{\mat{A}}$, linear}
		\vspace{\newSpace}
		\includegraphics[width=0.6\textwidth]{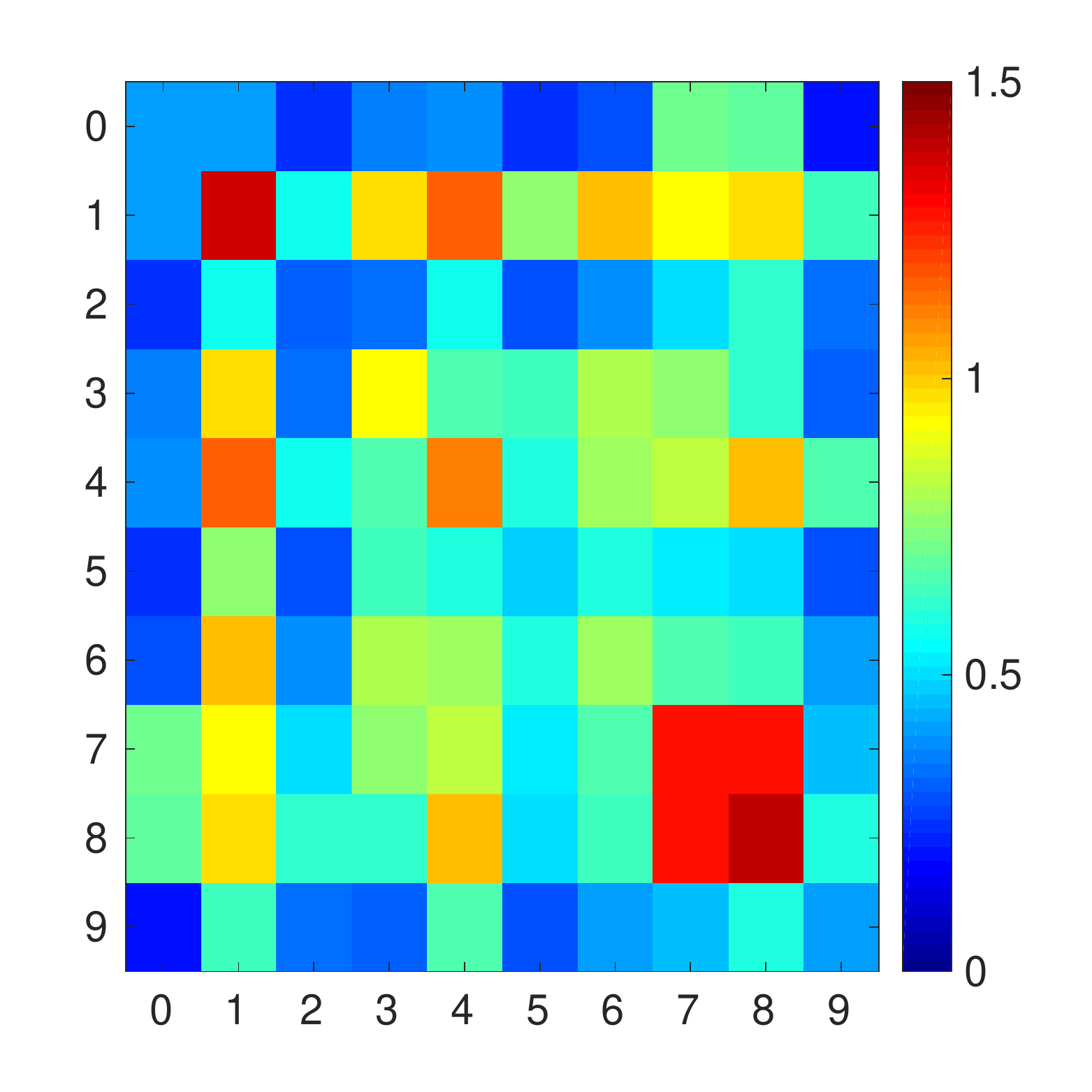}
		\label{ker_A}
	\end{subfigure}
	~
	\begin{subfigure}[b]{0.11\textwidth}
		\caption{$\ell_i$ of $\mat{A}$}
		\vspace{\newSpace}
		\includegraphics[width=0.6\textwidth]{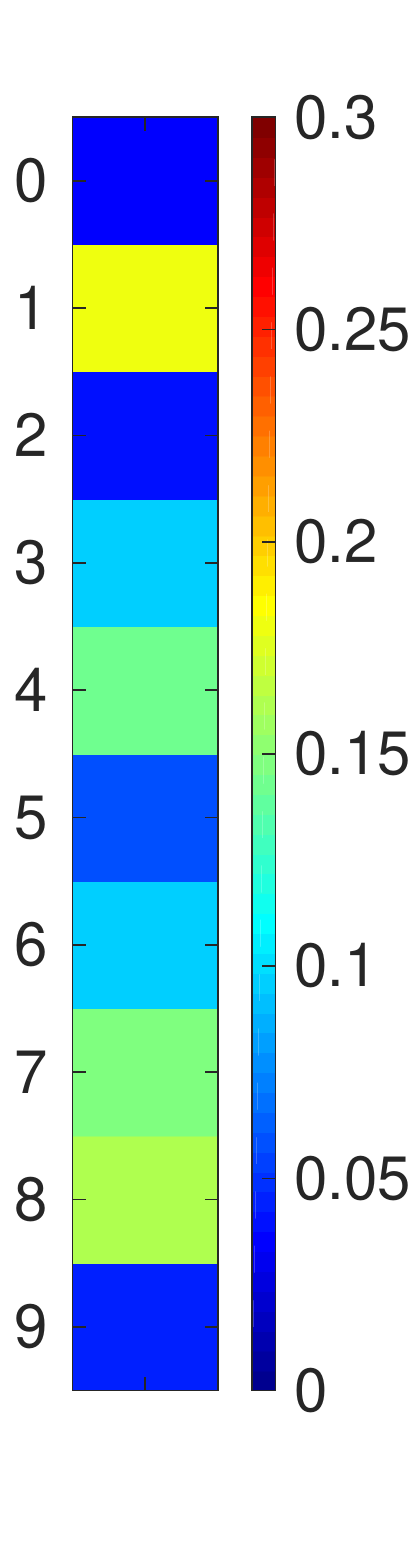}
		\label{lev_A}
	\end{subfigure}
	\vskip \newSkip
	\begin{subfigure}[b]{0.2\textwidth}
		\caption{$\mat{B}$}
		\vspace{\newSpace}
		\includegraphics[width=0.6\textwidth]{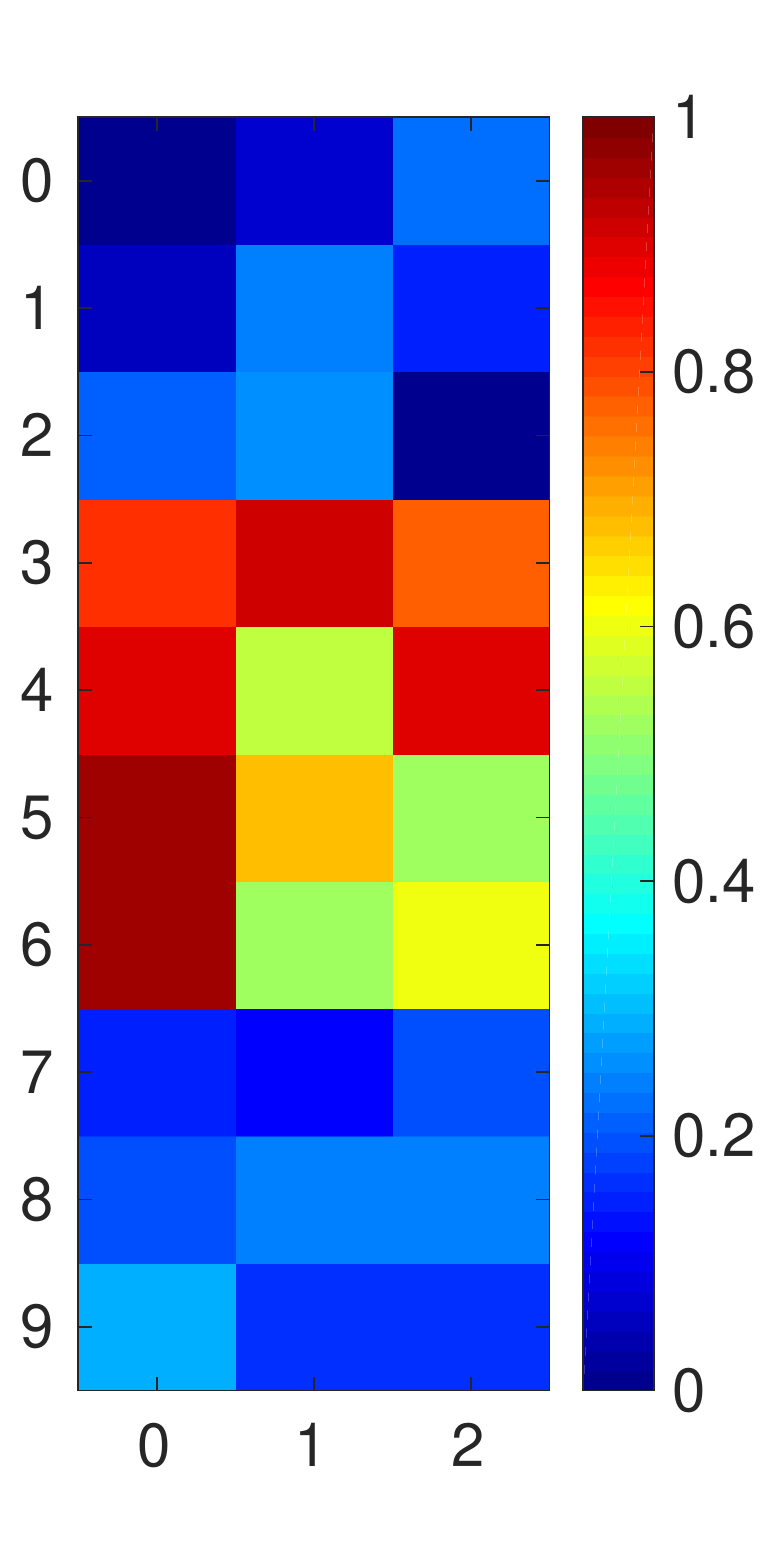}
		\label{B}
	\end{subfigure}
	~
	\begin{subfigure}[b]{0.4\textwidth}
		\caption{$\mat{K}_{\mat{B}}$, linear}
		\vspace{\newSpace}
		\includegraphics[width=0.6\textwidth]{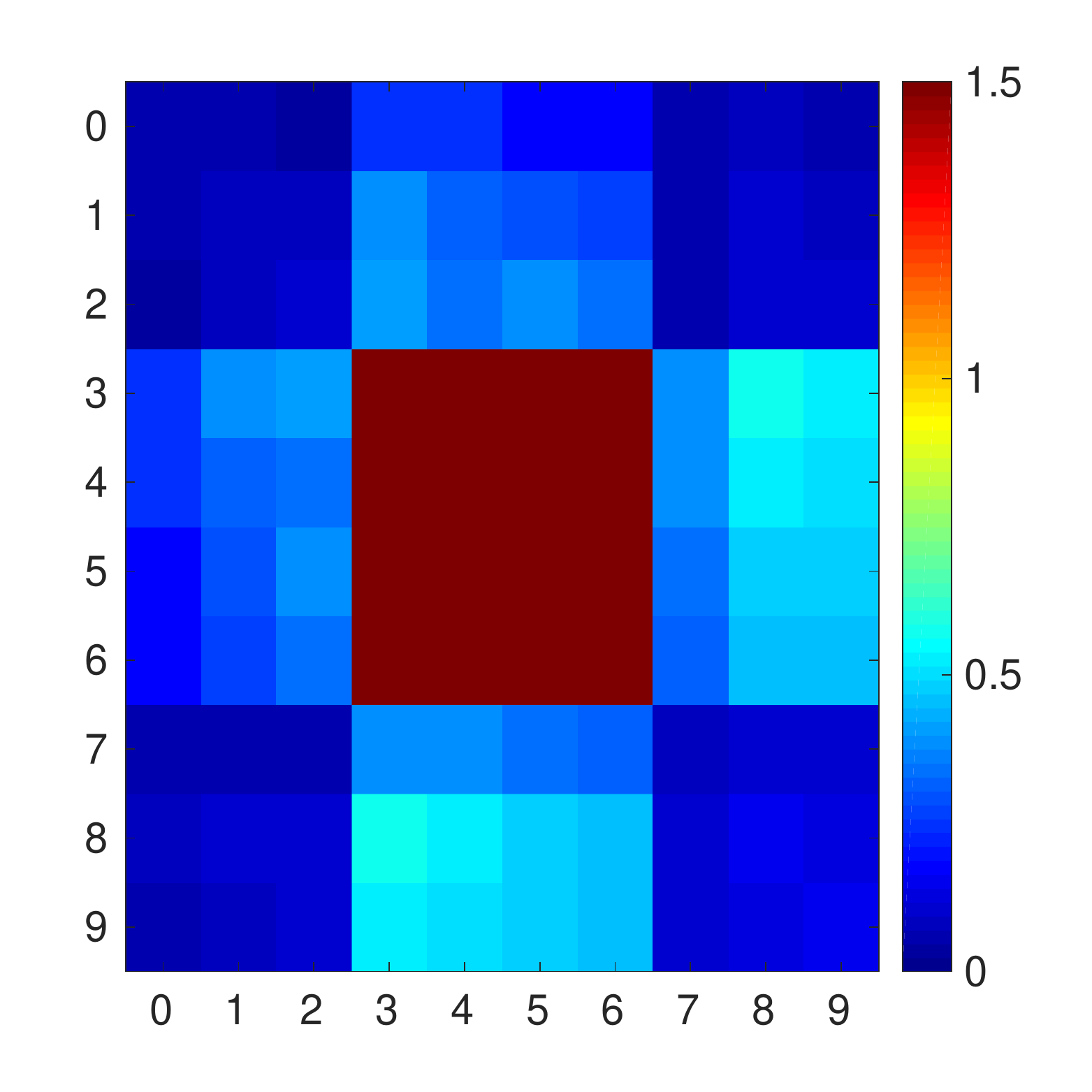}
		\label{ker_B}
	\end{subfigure}
	~
	\begin{subfigure}[b]{0.108\textwidth}
		\caption{$\ell_i$ of $\mat{B}$}
		\vspace{\newSpace}
		\includegraphics[width=0.6\textwidth]{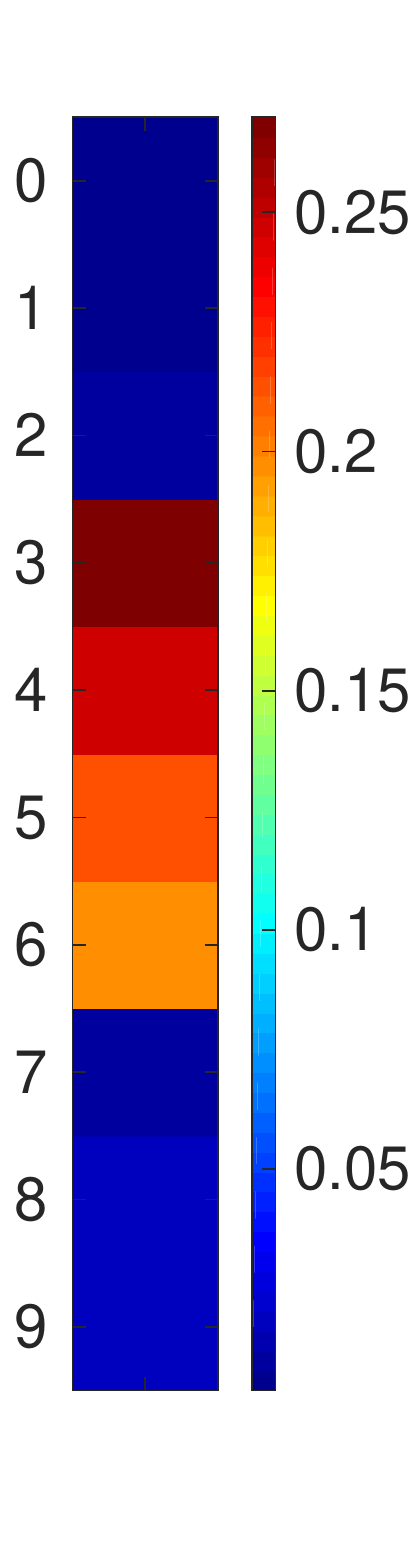}
		\label{lev_B}
	\end{subfigure}
	\caption{Leverage scores demonstrated on two toy matrices.}
	\label{figs:lev}
\end{figure*}

Motivated from this line of work which shows that statistical leverage scores are effective in finding columns (or rows) that exhibit high influence on the best low-rank fit of the data matrix, we propose to measure the representativeness of an example in a class with its leverage score in the kernel matrix computed on the examples.  A kernel function, $K: \mathcal{X} \times \mathcal{X} \rightarrow \mathbb{R}$ returns the dot product of the input vectors in a typically higher dimensional transformed feature space, $\Phi: \mathcal{X} \rightarrow \mathcal{H} $ \cite{scholkopf2001learning}. Let $K(\vec{x}_i, \vec{x}_j) = \langle \Phi(\vec{x}_i) \cdot \Phi(\vec{x}_j) \rangle_\mathcal{H}$. For a given $m$ number of examples, the kernel matrix is defined as $\mat{K} = \left[ K(\vec{x}_i,\vec{x}_j) \right]_{m\times m}$. 

Statistical leverage scores reflect the influence of the examples in a kernel matrix by capturing the most dominant part of the matrix.   Fig~\ref{figs:lev} demonstrates this idea on two toy matrices. The first matrix, $\mat{A}$, contains entries that are drawn from a uniform distribution on $[0,1]$ (Fig~\ref{figs:lev} a), whereas, $\mat{B}$, includes a submatrix that includes entries sampled uniformly at random from the $[0.6-1]$ range and the remaining entries are sampled from $[0-0.4]$ (Fig~\ref{figs:lev}d). Hence, every example is equally representative in $\mat{A}$ while few examples in $\mat{B}$ are representative.
Consider the linear kernel computed on these examples and let $\mat{K}_\mat{A}$ and $\mat{K}_\mat{B}$ denote them respectively (Fig~\ref{figs:lev}b and 1e).  The leverage scores computed on $\mat{K}_\mat{A}$ and $\mat{K}_\mat{B}$ depict the structural difference between the two matrices and  successfully identify the important rows (compare Fig~\ref{figs:lev}c and Fig~\ref{figs:lev}f ). The rows with high leverage scores of $\mat{B}$, rows 4-7 (Fig~\ref{figs:lev}f), encode most of the information in the matrix while the rows with all zero-entries have leverage scores of 0.  We use the idea that leverage scores can identify and rank the rows (examples) with most information in constructing the original kernel matrix, thus they can be used to assess the influence of the examples in the data distribution. 

\vspace{-4mm}
\section{Proposed Sequential Active Learning Method: ALEVS} 
\label{sec:alevs}
\vspace{-2mm}

For the sequential learning scenario, where at each round one example is queried for labeling, we propose \texttt{ALEVS}. The following steps are taken in deciding the example to query at each iteration $t$.

First, the training examples are divided into two subsets based on class memberships and two separate feature matrices are formed on these subsets. Let $h^t$ be the classifier at iteration $t$ that is trained with the labeled training examples $\mathcal{D}_\mathrm{l}^t$ with a supervised method, the class membership of the unlabeled examples are predicted with $h^t$. $\mat{X}_{+}^t$ is a $m \times d$ feature matrix, where the rows are the feature vectors of examples with positive class membership at iteration $t$. These examples are those whose true labels are known to be positive along with the examples for which the true labels are not known but are predicted to be in the positive class based on the prediction of $h^t$. $\mat{X}_-^t$ is similarly constructed from the negative examples.

In the second step, \texttt{ALEVS} computes kernel matrices on $\mat{X}_+^t$ and $\mat{X}_-^t$ separately.  For a given $m$ number of examples, the kernel matrix is defined as $\mat{K} = \left[ K(\vec{x}_i,\vec{x}_j) \right]_{m\times m}$.  \texttt{ALEVS} computes one kernel matrix on the positive class examples, $\mat{X}_+^t$, which we will denote with $\mat{K}_{+}^t$. Similarly, for the negatively labeled feature matrix $\mat{X}_-^t$, a kernel matrix $\mat{K}_{-}^t$ is defined. These two matrices encode the similarity of examples to other examples that are in the same class. 

We would like to find examples that carry the most information in the matrix to reconstruct the kernel matrix. \texttt{ALEVS} finds the example that imparts the strongest influence on the kernel matrices $\mat{K}_{+}^t$ and $\mat{K}_{-}^t$ through statistical leverage scores (Definition \ref{def:lev}). To be able to compare leverage scores of examples computed on matrices with different $m$ and $k$ values, we use the scaled leverage scores, which ensures that the average leverage score is 1:

\vspace{-4mm}
\begin{eqnarray}
&\ell_i = \frac{m}{k} \|({\mat{U}}_1)_{(i)}\|_2^2\enspace,
\end{eqnarray} 
 \vspace{-4mm}
 
At iteration $t$, \texttt{ALEVS} computes leverage scores for $\mat{K}_{+}^t$ and $\mat{K}_{-}^t$, and the unlabeled example that corresponds to the highest leverage score row in these matrices is selected for query:

\vspace{-4mm}
\begin{eqnarray}
&q= \argmax_{i \in \mathcal{D}_\mathrm{u}^{t} } \ell_i\enspace.
\end{eqnarray}
\vspace{-4mm}

These steps are repeated at each round of the active learning iterations. An important parameter in \texttt{ALEVS} is the target rank parameter $k$.  Let $\tau$ be the proportion of variance explained by the top first $k$ eigenvalues. We select the minimum possible low-rank parameter $k$, where  the sum of the top eigenvalues is at least as large as $\tau$. The overall procedure of \texttt{ALEVS} is summarized in  Algorithms~\ref{alg:ALEVSKernel}, ~\ref{alg:comp_lev}, and ~\ref{alg:k_est}.

\vspace{-5mm}

\begin{algorithm}[H]
	\caption{\texttt{ALEVS}: Active Learning with Leverage Score Sampling}
	\label{alg:ALEVSKernel}
	\begin{algorithmic}
		\STATE {\bf Input:} $\mathcal D$: a training dataset of $n$ instances; $\mathcal{O}$: labeling oracle; $\tau$: eigenvalue threshold; $p$: kernel parameters.
		\STATE {\bf Output:} $h^*$: final classifier.
		\STATE {\bf Initialize:}\\
		\STATE $\mathcal{D}_\mathrm l^{0}$\quad\quad\quad\quad \quad\quad\quad// initial set of labeled instances
		\STATE$\mathcal{D}_\mathrm u^{0} \gets \mathcal{D} \setminus \mathcal{D}_\mathrm l^{0}$ \quad\quad // the pool of unlabeled instances
		\STATE $t \gets 0$
		\REPEAT
		\STATE ------------------ {\bf Classification} ---------------------------
		\STATE $h^t \gets$ train($D_\mathrm l^t$)
		\STATE $\vec{\hat{y}_\mathrm{u}^t} \gets$ predict($h^t, \mathcal{D}_\mathrm u^t$)
		\STATE ------------------ {\bf Sampling} ---------------------------------
		\STATE Based on $\vec{\hat{y}}_\mathrm u^t$ and $\vec{y}_\mathrm l^t$, construct $\mat{X}_{+}^t$ and $\mat{X}_{-}^t$ 
		\STATE $\mat{K}_{+}^t \gets$ \texttt{ComputeKernel}($\mat{X}_+^t,p$) 
		\STATE $\mat{K}_{-}^t \gets$ \texttt{ComputeKernel}($\mat{X}_-^t,p$)
		\STATE $\ell_+^t \gets$ \texttt{ComputeLeverage}($\mat{K}_+^t$,$\tau$)
		\STATE $\ell_-^t \gets$ \texttt{ComputeLeverage}($\mat{K}_-^t$,$\tau$)
		\STATE $\ell^t \gets \ell_+^t \cup \ell_-^t$
		\STATE $\vec{x}_\mathrm q^t = \arg\max_{\vec{x}_j \in \mathcal{D}_\mathrm u^{t}} \ell_j^t$
		\STATE $y_\mathrm q^t \gets$ query($\mathcal{O}$,$\vec{x}_\mathrm q^t$)
		\STATE ------------------ {\bf Update} ----------------------------------- 
		\STATE $\mathcal{D}_\mathrm l^{t+1} \gets \mathcal{D}_\mathrm l^t \cup (\vec{x}_\mathrm q^t, y_\mathrm q^t)$ 
		\STATE $\mathcal{D}_\mathrm u^{t+1} \gets \mathcal{D}_\mathrm u^t \setminus \vec{x}_\mathrm q^t$ 
		\STATE $t \gets t+1$
		\UNTIL stopping criterion
		\STATE $h^* \gets h^t$
	\end{algorithmic}
\end{algorithm}
\vspace{-5mm}

\begin{algorithm}[H]
	\caption{\texttt{ComputeLeverage}}
	\label{alg:comp_lev}
	\begin{algorithmic}
		\STATE {\bf Input:} $\mat K$: $m\times m$ kernel matrix; $\tau$: eigenvalue threshold.
		\STATE {\bf Output:} $\ell$: leverage scores.
		\STATE $\mat{K}=\mat{U}\mat{\Sigma} \mat{U}^T$
		\STATE $\mat \lambda \gets diag(\Sigma)$
		\STATE $k \gets$ \texttt{RankSelector}($\mat \lambda$,$\tau$)
		\STATE $\mat{U} = \bigg (\mat{U}_1 \quad \mat{U}_2 \bigg)$, where $\mat{U}_1$ spans top $k$-eigenspace of $\mat{K}$ and is $m \times k$
		\FOR{$i=1$ \TO $m$ } \STATE{$\ell_i = \frac{m}{k} \|({\mat{U}}_1)_{(i)}\|_2^2$} \ENDFOR
	\end{algorithmic}
\end{algorithm}
\vspace{-2mm}
\begin{algorithm}[H]
	\caption{\texttt{RankSelector}}
	\label{alg:k_est}
	\begin{algorithmic}
		\STATE {\bf Input:} $\vec \lambda$: $m\times 1$ vector containing eigenvalues; $\tau$: eigenvalue threshold.
		\STATE {\bf Output:} $k$: target rank.
		\STATE $\vec \lambda \leftarrow$ sort($\vec \lambda$,`descend')
		\STATE $k \leftarrow 1$
		\WHILE {$\frac{\sum_{i=1}^{k}\vec \lambda_i}{\sum_{ i=1}^{m}\vec \lambda_i} < \tau$}
		\STATE $k \leftarrow k+1$
		\ENDWHILE
	\end{algorithmic}
\end{algorithm}

\vspace{-2mm}
\section{Proposed Batch-Mode Active Learning Method: DBALEVS}
\label{sec:dbalevs}
\vspace{-2mm}
A high-quality batch should contain highly influential examples in the data distribution. On the other hand, as some of the examples can be highly influential on an individual basis, they might contain redundant information and can form poor batches if they are queried together.  \texttt{DBALEVS} aims to select a batch not only diverse within the current batch but also with respect to the already labeled examples. We encode these properties in a set scoring function and use it to select batches at each iteration.  The sum of leverage scores of the examples in the batch assesses the total usefulness of a set of examples. To select a diverse set, we incorporate a term that penalizes the selection of examples that are similar to each other. For evaluating the similarity of examples, we use the kernel function.  We define the following set scoring function:



\begin{definition}[Set scoring function]
	\label{set-score}
	Given a set $S$ that is a subset of the ground set $V$, $S\subseteq V$ the scoring function, $F:2^V\rightarrow \mathbb{R}$, is defined as follows:
	\begin{eqnarray}
	\label{eq:setscore}
	& F(S) = \sum\limits_{i\in S}(\ell_i+1) - \frac{\alpha}{M} \sum\limits_{\substack{i,j \in S\\i \neq j}} K(i,j)
	\end{eqnarray}
	Here, $M \geq |S|$ is a cardinality constraint on the selectable set size of a batch. $\ell_i$ denotes the leverage score of point $i$, and $K(i,j)$ denotes the kernel function evaluation of points $i$ and $j$ with the assumption that $0 \leq K(i,j) \leq 1$. $\alpha \in [0,1]$ is a parameter.
\end{definition}
\vspace{-1mm}

The first part of this function evaluates the individual representatives of  the examples in the set while the second part of the function penalizes the selection of highly similar instances. The influence of the diversity term can be adjusted by the trade-off parameter $\alpha$.  We would like to select a batch that maximizes the set function, $F$:
 \vspace{-1mm}
\begin{eqnarray}
S^* = \arg\max F(S)\\
s.t.\quad |S| = b \nonumber
\end{eqnarray}
\vspace{-2mm}

This is a subset selection problem and except for small sets and small values of $b$, the exhaustive search for the optimal batch will be intractable. To tackle this computational challenge, we exploit the fact that the suggested set function is submodular. Although submodular maximization is also NP-hard in general \cite{Krause05near-optimalnonmyopic}, Nemhauser et al. \cite{nemhauser1978analysis} showed that the greedy algorithm for selecting a subset of size $b$  is guaranteed to return a solution close to the optimal value within a constant bound (Theorem~\ref{thm:nemhauser}). 

\begin{theorem}
	\label{thm:nemhauser} For a monotone, non-negative, submodular function $f:2^{V}\rightarrow \mathbb{R}$, and a cardinality constraint $b$, the greedy approximation yields to:
	\begin{eqnarray}
	f(S_b) \geq (1-\frac{1}{e}) \max_{|S|\leq b}f(S)
	\end{eqnarray}
	where $S_b$ denotes the greedily selected set with cardinality $b$ \normalfont{\cite{nemhauser1978analysis}}.
\end{theorem}



The greedy algorithm adds elements to the solution that gives the maximum increase at each step. To be able to use this greedy algorithm with the aforementioned approximation bound, we need to show that $F$ is a submodular, monotonically non-decreasing and non-negative function.   Below we first define submodularity and then prove $F$ is submodular.
\vspace{-2mm}
\begin{definition}[Submodularity]
	\label{def:subm}
	Let $A \subseteq B \subseteq V$, where $V$ denotes the ground set and let $x\in V\setminus B$ be an element. A set function $f:2^V\rightarrow \mathbb{R}$ is called \textit{submodular} if the following holds:
	\begin{eqnarray}
	& f(A\cup \{x\}) - f(A) \geq f(B\cup \{x\}) - f(B)
	\end{eqnarray}
\end{definition}

\vspace{-2mm}

\begin{proposition}[Submodularity]
	\label{subm}
	$F$ is submodular.
\end{proposition}

\begin{proof} For $F$ to be submodular, the following should hold:
	\begin{eqnarray}
	& F(A\cup \{x\})-F(A) \geq F(B\cup \{x\}) - F(B) \enspace.
	\label{eq:submod}
	\end{eqnarray}
	Using Definition~\ref{set-score} for $F$:
	\small
	\begin{align*}
	& F(A \cup \{x\})-F(A) \\
	& =\bigg(\sum_{i\in A\cup \{x\}} (\ell_i +1) - \frac{\alpha}{M} \smashoperator[r]{ \sum_{\substack{i,j \in A\cup \{x\} \\ i \neq j}}}K(i,j)\bigg) - \bigg(\sum_{i\in A}(\ell_i+1) - \frac{\alpha}{M} \sum_{\substack{i,j \in A\\i \neq j}} K(i,j)\bigg)
	\end{align*}
	\begin{align*}
	& = \bigg(\sum_{i\in A} \ell_i + \ell_x + \left\vert{A}\right\vert + 1 - \frac{\alpha}{M}\Big(\sum_{\substack{i,j \in A \\ i \neq j}}K(i,j) + \sum_{i\in A}K(x,i)\Big)\bigg) - 
	\\ & \quad\, \bigg(\sum_{i\in A} \ell_i + \left\vert{A}\right\vert - \frac{\alpha}{M}\sum_{ \substack{i,j \in A \\ i \neq j}}K(i,j)\bigg)
	\end{align*}
	\normalsize
	Rearranging the terms we end up with the following expression:
	\small
	\begin{equation}
	F(A\cup \{x\})-F(A) = \Big(\ell_x + 1 - \frac{\alpha}{M}\sum_{i\in A} K(x,i)\Big)\nonumber
	\end{equation}
	\normalsize
	If we do the same simplification for the right hand side of the submodularity definition, $F(B\cup \{x\}) - F(B)$ , we arrive to a similar expression for set $B$. Therefore,
	
	\small
	\begin{align*}
	&F(A\cup \{x\})-F(A) - (F(B\cup \{x\} )-F(B)) \\
	&= \Big(\ell_x + 1 - \frac{\alpha}{M} \sum_{i\in A} K(x,i) \Big) - \Big(\ell_x + 1 - \frac{\alpha}{M}\sum_{i\in B} K(x,i) \Big)\\ 
	&= \frac{\alpha}{M}\sum_{i\in B} K (x,i) - \frac{\alpha}{M} \sum_{i\in A}K(x,i) \\
	&= \bigg( \frac{\alpha}{M}\sum_{i\in B \setminus A} K(x,i) + \frac{\alpha}{M} \sum_{i\in A} K(x,i)\bigg) - \frac{\alpha}{M} \sum_{i\in A} K(x,i) \\
	&= \frac{\alpha}{M}\sum_{i\in B \setminus A} K(x,i)
	\end{align*}
	\normalsize
	Since $ K(i,j)\geq 0$ and $\alpha \in [0,1]$ and $M>0$, $ \frac{\alpha}{M}\sum_{i\in B \setminus A} K(x,i) \geq 0$. Therefore, $F(A\cup \{x\})-F(A) \geq (F(B\cup \{x\} )-F(B)) $. Hence, $F$ is submodular. \qed.
\end{proof}

To be able to apply the greedy algorithm with an approximation guarantee, we also need to show that $F$ is a monotonically non-decreasing and non-negative function under reasonable conditions. The proofs that $F$ satisfies these conditions when the selected batch size is less than or equal to $M$ are provided in Appendix \ref{app:proof1} and \ref{app:proof2}.

\begin{algorithm}[H]
	\caption{\texttt{DBALEVS}: Diverse Batch Mode Active Learning with Leverage Score Sampling}
	\label{alg:DBALEVS}
	\begin{algorithmic}
		\STATE {\bf Input:} $\mathcal D$: a training dataset of $N$ instances; $\mathcal{O}$: labeling oracle; $\tau$: eigenvalue threshold; $p$: kernel parameters; $F$: set scoring function in Definition~\ref{set-score}; $b$: batch size; $\alpha$: diversity trade-off parameter for $F$.
		\STATE {\bf Output:} $h^*$: final classifier.
		\STATE {\bf Initialize:} \\
		\STATE $\mathcal{D}_\mathrm l^{0}$ \quad \quad \quad \quad \quad \quad // initial set of labeled instances
		\STATE$\mathcal{D}_\mathrm u^{0} \gets \mathcal{D} \setminus \mathcal{D}_\mathrm l^{0}$ \quad\quad // the pool of unlabeled instances
		\STATE $t \gets 0$
		\REPEAT
		\STATE ------------------ {\bf Classification} ---------------------------
		\STATE $h^t \gets$ train($D_\mathrm l^t$)
		\STATE $\vec{\hat{y}}_\mathrm u^t \gets$ predict($h^t, \mathcal{D}_\mathrm u^t$)
		\STATE ------------------ {\bf Sampling} ---------------------------------
		\STATE Based on $\vec{\hat{y}}_\mathrm u^t$ and $\vec{y}_\mathrm l^t$, construct $\mat{X}_{+}^t$ and $\mat{X}_{-}^t$
		\STATE Based on $\vec{y}_\mathrm l^t$, construct $\mat{L}_{+}^t$ and $\mat{L}_{-}^t$ //labeled class matrices
		\STATE $\mat{K}_{+}^t \gets$ \texttt{ComputeKernel}($\mat{X}_+^t,p$) 
		\STATE $\mat{K}_{-}^t \gets$ \texttt{ComputeKernel}($\mat{X}_-^t,p$) 
		\STATE $\ell_+^t \gets$ \texttt{ComputeLeverage}($\mat{K}_+^t$,$\tau$)
		\STATE $\ell_-^t \gets$ \texttt{ComputeLeverage}($\mat{K}_-^t$,$\tau$)
		\STATE $S_\mathrm q^+ \gets$ \texttt{B-GreedyAlgorithm}($F,b/2,\mat{L}_{+}^t,\ell_+^t,\mat{K}_{+}^t,\alpha$)
		\STATE $S_\mathrm q^- \gets$ \texttt{B-GreedyAlgorithm}($F,b/2,\mat{L}_{-}^t,\ell_-^t,\mat{K}_{-}^t,\alpha$)
		\STATE $S_\mathrm q^t \gets S_\mathrm q^+ \cup S_\mathrm q^-$
		\STATE $\vec y_\mathrm q^t \gets$ query($\mathcal{O}$, $S_\mathrm q^t$)
		\STATE ------------------ {\bf Update} ----------------------------------- 
		\STATE $\mathcal{D}_\mathrm l^{t+1} \gets \mathcal{D}_\mathrm l^t \cup (S_\mathrm q^t, \vec y_\mathrm q^t)$ 
		\STATE $\mathcal{D}_\mathrm u^{t+1} \gets \mathcal{D}_\mathrm u^t \setminus S_\mathrm q^t$ 
		\STATE $t \gets t+1$
		\UNTIL stopping criterion
		\STATE $h^* \gets h^t$
	\end{algorithmic}
\end{algorithm}
\vspace{-3mm}
\begin{algorithm}[!ht]
	\caption{\texttt{B-GreedyAlgorithm}}
	\label{alg:greedy2}
	\begin{algorithmic}
		\STATE {\bf Input:} $F$: set scoring function in Definition~\ref{set-score}; $b$: batch size; $A$: initial set; $V$: ground set; $\ell$: leverage scores; $\mat{K}$: kernel matrix; $\alpha$: diversity parameter for $F$.
		\STATE {\bf Output:} $S$: selected set.
		\STATE $S_0 \gets A$
		\STATE $M \gets |A| + b$
		\STATE $i \gets 1$
		\WHILE {$i \leq b$}
		\STATE $S_i \gets S_{i-1} \cup \arg\max_{x\in V\setminus S_{i-1}} F(\ell,\mat{K},\alpha, M)$
		\STATE $i \gets i+1$
		\ENDWHILE
		\STATE $S \gets S_i \setminus A$
	\end{algorithmic}
\end{algorithm}
%

The  procedure for querying a batch is summarized in Algorithm~\ref{alg:DBALEVS}.  First, the labeled and unlabeled pool is divided based on class labels. As in \texttt{ALEVS}, the iteration $t$, the classifier, $h^t$ is exclusively trained with the labeled training examples $\mathcal{D}_\mathrm{l}^t$ with a supervised method, and the class membership of the unlabeled examples are predicted with $h^t$. The examples whose true labels are known along with the instances for which the true labels are not known but are predicted to be in the positive class based on the prediction of $h^t$ form a positive class group, $\mat{X}_+^t$. $\mat{X}_-^t$ is similarly constructed from negatively predicted and labeled examples. Having divided the pool based on class memberships, the kernel matrices for each class are computed. $\mat{K}_+^t$ is formed using $\mat{X}_+^t$, and $\mat{K}_-^t$ is formed using $\mat{X}_-^t$.  Then leverage scores of the examples are computed using the kernel matrices based on Definition~\ref{def:lev} for each class. Not we used  the leverage scores without scaling with $\frac{m}{k}$. This is necessary to ensure the submodularity of $F$. \texttt{DBALEVS} selects half of the batch from the positive examples, and half of the points from the negative examples. For this selection, the method uses the set scoring function (Definition~\ref{set-score}). For greedy maximization, the method uses the available labeled data for positive ($\mat{L}_{+}^t$) and negative ($\mat{L}_{+}^t$) class as the initial set. This allows selecting a set that is also diverse with respect to the already labeled examples. This modified greedy maximization is  given in Algorithm~\ref{alg:greedy2}.

\vspace{-4mm}
\section{Results for ALEVS}
\label{sec:alevsresults}
\vspace{-2mm}



We compare \texttt{ALEVS} with the following five approaches:
\vspace{-1mm}
\begin{itemize}
\item  {\bf Random sampling:} Selects an  unlabelled example uniformly at random.  
\item {\bf Uncertainty sampling:} Queries the example that the current classifier is most uncertain about \cite{Lewis94}, that is the one with maximal $(1 - p( y^* \given{\vec{x}} ))$ value; here $y^*$ is the predicted class label for that example. The posterior probability is estimated with Platt's algorithm \cite{platt1999probabilistic} based on SVM's output. 
\item {\bf Leverage sampling on all data (\texttt{LevOnAll})}: Computes the leverage score on the  pool of examples at the beginning of the iteration without paying attention to class membership, then at each iteration queries the unlabeled example with the largest leverage score. 
\item {\bf Transductive experimental design: } Method selects observations to maximize the quality of parameter estimates in linear regression model \cite{yu:icml06expdesign}. The model is also applicaple to classification problems. 
 {\item \bf QUIRE:} Selects an instance that is both informative and representative through optimizing a function that encodes these properties \cite{quire}. 
\end{itemize}


\begin{table}
	\caption{Datasets for \texttt{ALEVS} experiments, number of samples, features and the positive to negative class  ratio (+/-) r are listed.}
\vspace{-2mm}
	\centering
\scalebox{0.85}{
	\begin{tabular}{llll}
		\hline\noalign{\smallskip}
		{\sc Dataset} & {\sc Size} & {\sc\quad Dim.} & {\sc \quad +/-}\\ \hline
		\noalign{\smallskip}
		\textit{digit1} & 1500 & \quad 241 & \quad 1.00 \\ 
		\textit{g241c} & 1500 &\quad 241 &\quad 1.00 \\ 
		\textit{UvsV} & 1577 & \quad 16 & \quad 1.10 \\ 
		\textit{USPS} & 1500 & \quad 241 & \quad 0.25 \\ 
		\textit{twonorm} & 2000 & \quad 20 &\quad 1.00 \\ 
		\textit{ringnorm} & 2000 &\quad 20 &\quad 1.00 \\ 
		\textit{spambase} & 2000 &\quad 57 &\quad 0.66 \\ 
		\textit{3vs5} & 2000 &\quad 784 &\quad 1.20 \\ \hline
	\end{tabular}
}
\vspace{-4mm}
	\label{tbl:datasets}
\end{table}

%

We compare the methods on eight different datasets. Table~\ref{tbl:datasets} summarizes the characteristics of the datasets (for more information about the data, see Appendix \ref{app:datasets}. Each dataset is divided into training and held out test sets. We start with four randomly selected labeled examples, two from each class. At each iteration, the classifier is updated for all the methods with the training data, and the accuracy values are calculated on the same held-out test data. In all experiments, an SVM classifier with RBF kernel is trained. The experiments are repeated $50$ times with random splitting of the training and the test data and random initial selection of labeled examples. 

%


Fig.~\ref{figs:seq} shows the average classification accuracy values of \texttt{ALEVS} and other approaches at each iteration of active sampling. Table~\ref{wtl:50} and Table~\ref{wtl:100} summarize the win, tie and lost counts of \texttt{ALEVS} versus each of the competing methods on the 1-sided paired sample $t$-test at the significance level of 0.05.

\begin{table}[h]
	\centering
	\caption{Win/Tie/Loss counts of \texttt{ALEVS} against the competitor algorithms for the first 50 iterations (Sequential-mode).}
	\scalebox{0.85}{
		\begin{tabular}{|l|c|c|c|c|c|c|c|c|c|}
			\hline 
			{\sc Dataset} & vs. \textsc{QUIRE} & vs. \textsc{LevOnAll} & vs. \textsc{Random} & vs. \textsc{Uncertainty} & vs. \textsc{ExpDesign} \\ \hline
			\textit{digit1} & 8/19/23 & 29/21/0 & 27/23/0 & 46/4/0 & 33/17/0  \\
			\textit{g241c} & 0/34/16 & 32/18/0 & 30/20/0 & 34/16/0 & 30/19/1  \\
			\textit{USPS} & 0/43/7 & 32/16/2 & 33/12/5 & 0/50/0 & 32/18/0  \\
			\textit{ringnorm} & 47/3/0 & 48/2/0 & 49/1/0 & 47/3/0 & 49/1/0  \\
			\textit{spambase} & 8/21/21 & 16/27/7 & 10/29/11 & 0/46/4 & 32/7/11  \\
			\textit{MNIST-3vs5} & 3/41/6 & 42/8/0 & 44/6/0 & 48/2/0 & 43/7/0  \\
			\textit{UvsV} & 0/2/48 & 48/2/0 & 25/25/0 & 8/12/30 & 49/1/0  \\
			\textit{twonorm} & 49/1/0 & 50/0/0 & 50/0/0 & 50/0/0 & 49/1/0  \\ \hline
		\end{tabular} 
	}
	\label{wtl:50}

	\bigskip
	
	\centering
	\caption{Win/Tie/Loss counts for \texttt{ALEVS} against the competitor algorithm iterations between 50 and 100 (Sequential-mode).}
	\scalebox{0.85}{
		\begin{tabular}{|l|c|c|c|c|c|c|c|c|c|}
			\hline 
			{\sc Dataset} & vs. \textsc{QUIRE} & vs. \textsc{LevOnAll} & vs. \textsc{Random} & vs. \textsc{Uncertainty} & vs. \textsc{ExpDesign} \\ \hline 
			\textit{digit1} & 0/0/50 & 0/16/34 & 0/16/34 & 0/15/35 & 0/50/0  \\
			\textit{g241c} & 9/41/0 & 50/0/0 & 50/0/0 & 50/0/0 & 50/0/0  \\
			\textit{USPS} & 0/12/38 & 50/0/0 & 50/0/0 & 0/32/18 & 50/0/0  \\
			\textit{ringnorm} & 50/0/0 & 22/28/0 & 50/0/0 & 13/37/0 & 50/0/0  \\
			\textit{spambase} & 0/6/44 & 2/48/0 & 24/26/0 & 0/9/41 & 50/0/0  \\
			\textit{MNIST-3vs5} & 0/20/30 & 39/11/0 & 50/0/0 & 6/17/27 & 50/0/0  \\
			\textit{UvsV} & 0/0/50 & 2/48/0 & 0/50/0 & 0/0/50 & 50/0/0  \\
			\textit{twonorm} & 50/0/0 & 50/0/0 & 50/0/0 & 50/0/0 & 50/0/0  \\ \hline
		\end{tabular}
	}
	\label{wtl:100}
\end{table}

\newcommand{\newW}{0.4}
\begin{figure*}[p]
	\centering
	\vskip \newSkip
	\begin{subfigure}[b]{\newW\textwidth}
		\caption{\textit{digit1}}
		\vspace{\newSpace}
		\includegraphics[width=\textwidth]{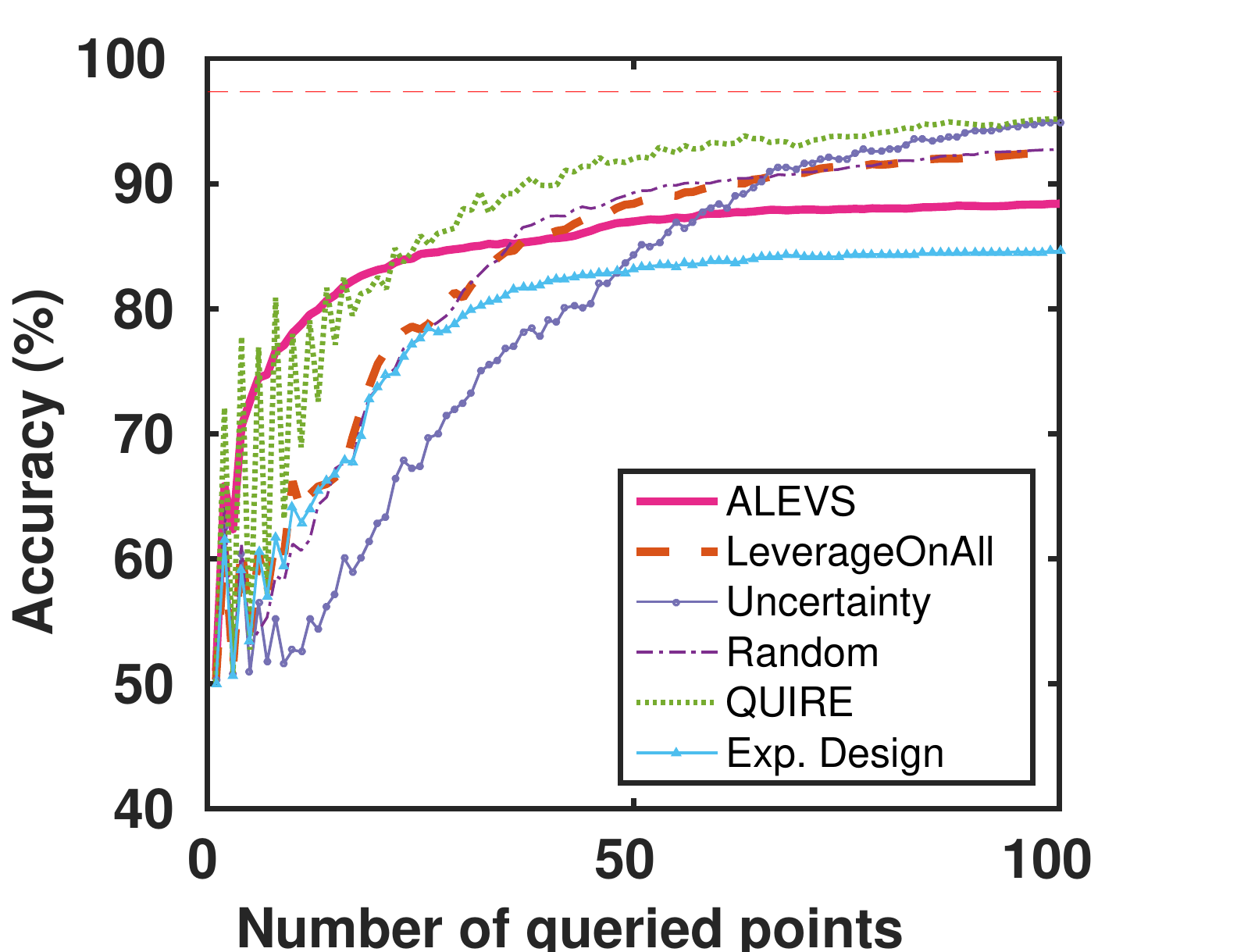}
		\label{digit1_acc:seq}
	\end{subfigure}
	~
	\begin{subfigure}[b]{\newW\textwidth}
		\caption{\textit{3vs5}}
		\vspace{\newSpace}
		\includegraphics[width=\textwidth]{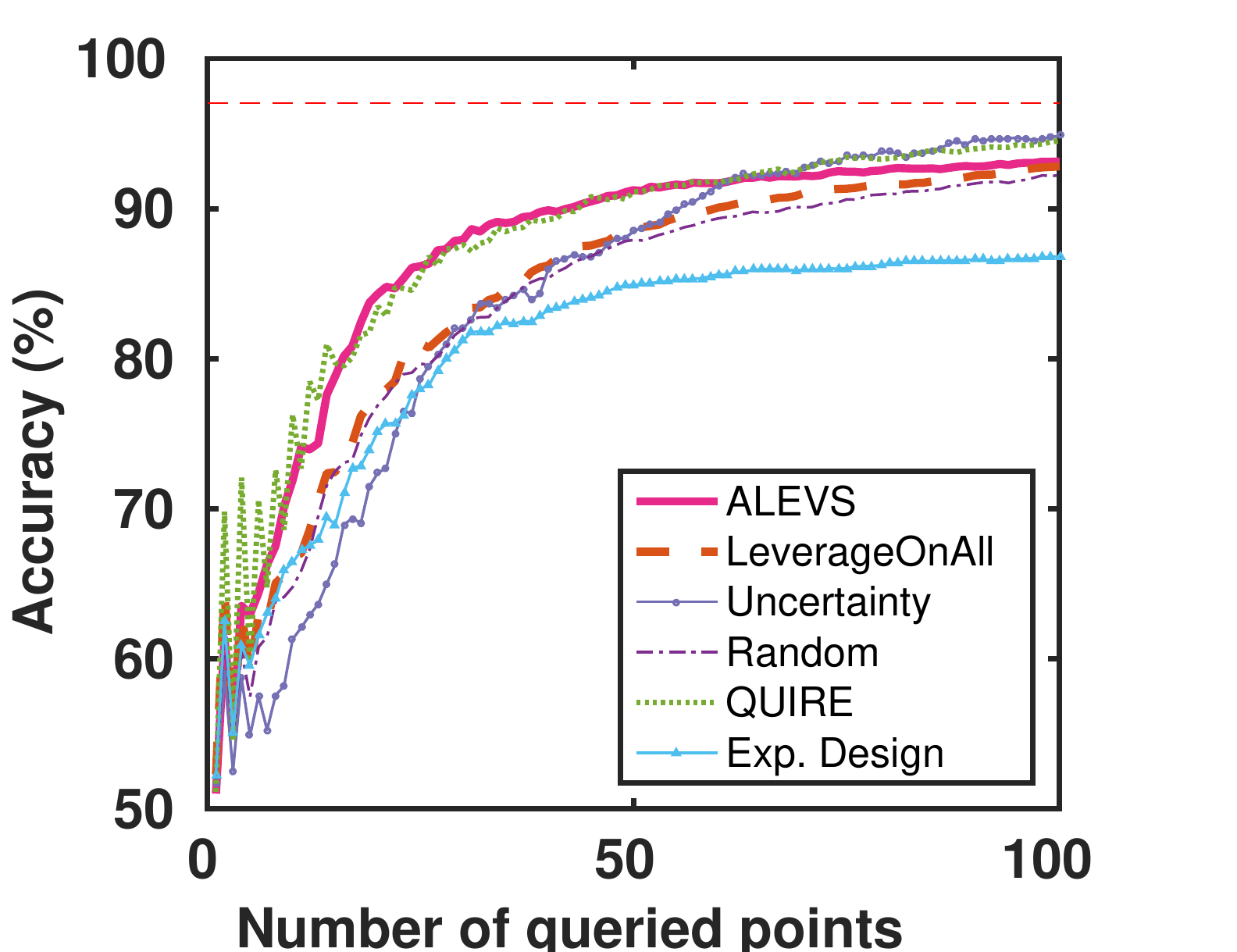}
		\label{MNIST_acc:seq}
	\end{subfigure}
	\vskip \newSkip
	\begin{subfigure}[b]{\newW\textwidth}
		\caption{\textit{g241c}}
		\vspace{\newSpace}
		\includegraphics[width=\textwidth]{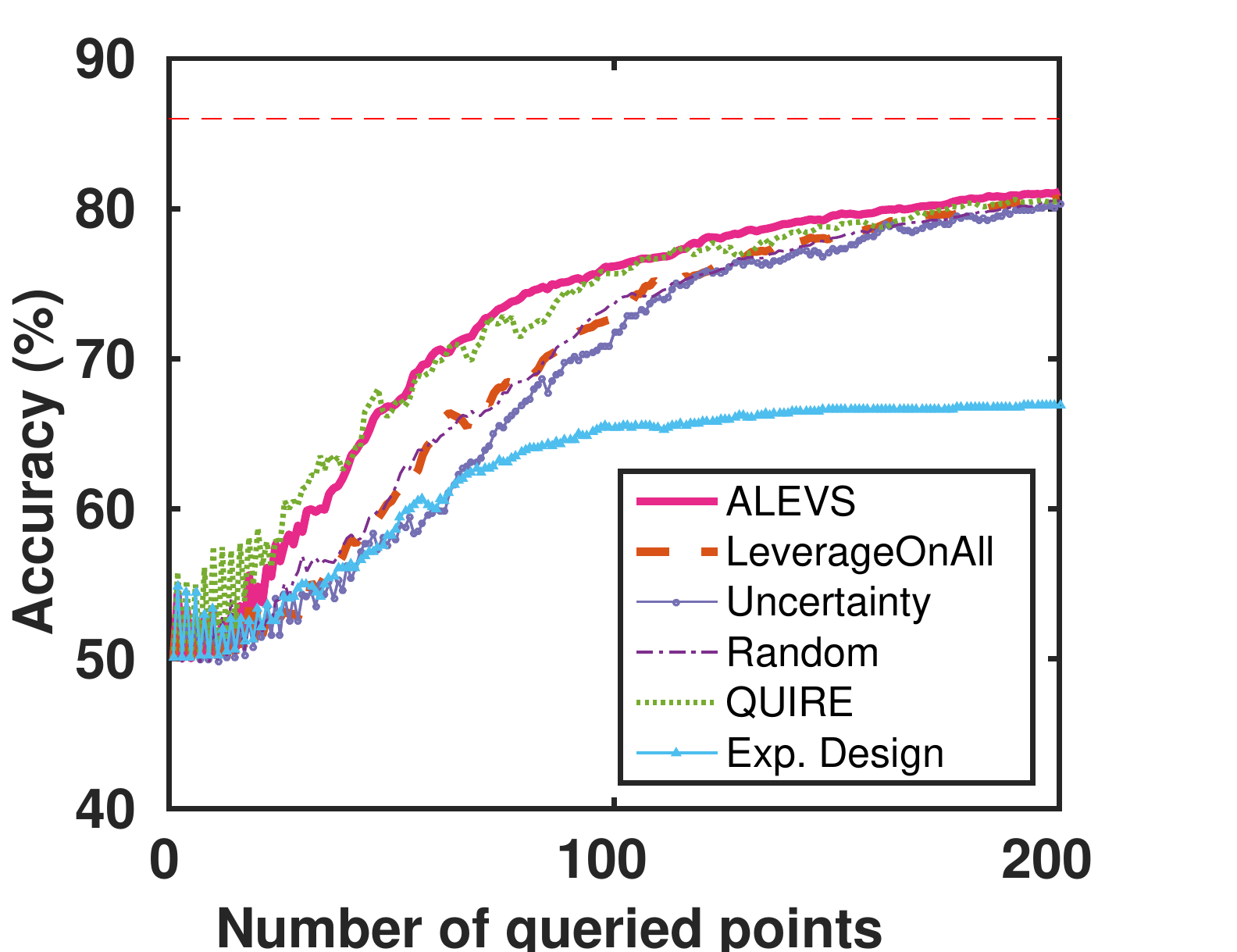}
		\label{g241c_acc:seq}
	\end{subfigure}
	~
	\begin{subfigure}[b]{\newW\textwidth}
		\caption{\textit{UvsV}}
		\vspace{\newSpace}
		\includegraphics[width=\textwidth]{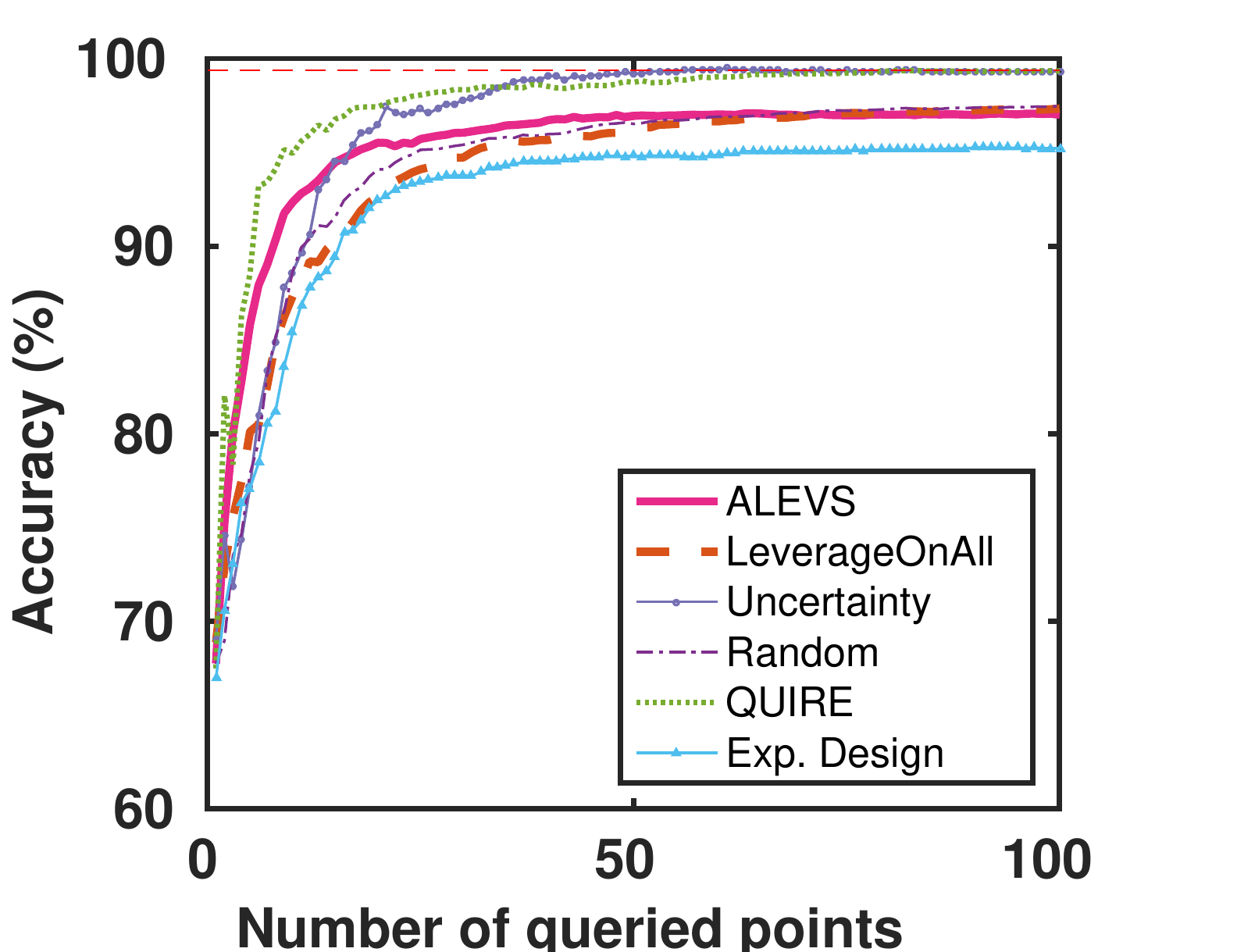}
		\label{letter2_acc:seq}
	\end{subfigure}
	\vskip \newSkip
	\begin{subfigure}[b]{\newW\textwidth}
		\caption{\textit{USPS}}
		\vspace{\newSpace}
		\includegraphics[width=\textwidth]{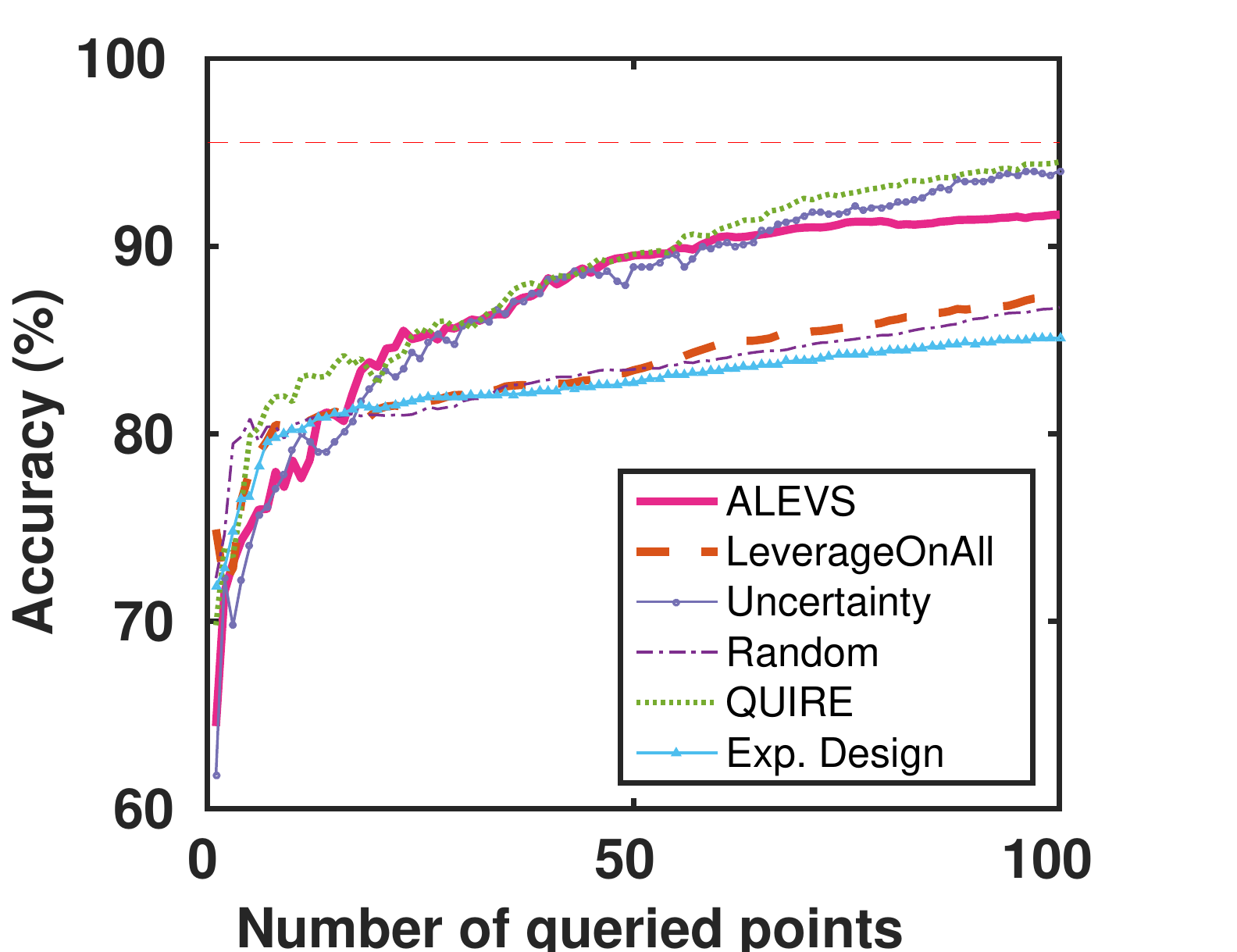}
		\label{USPS_acc:seq}
	\end{subfigure}
	~
	\begin{subfigure}[b]{\newW\textwidth}
		\caption{\textit{twonorm}}
		\vspace{\newSpace}
		\includegraphics[width=\textwidth]{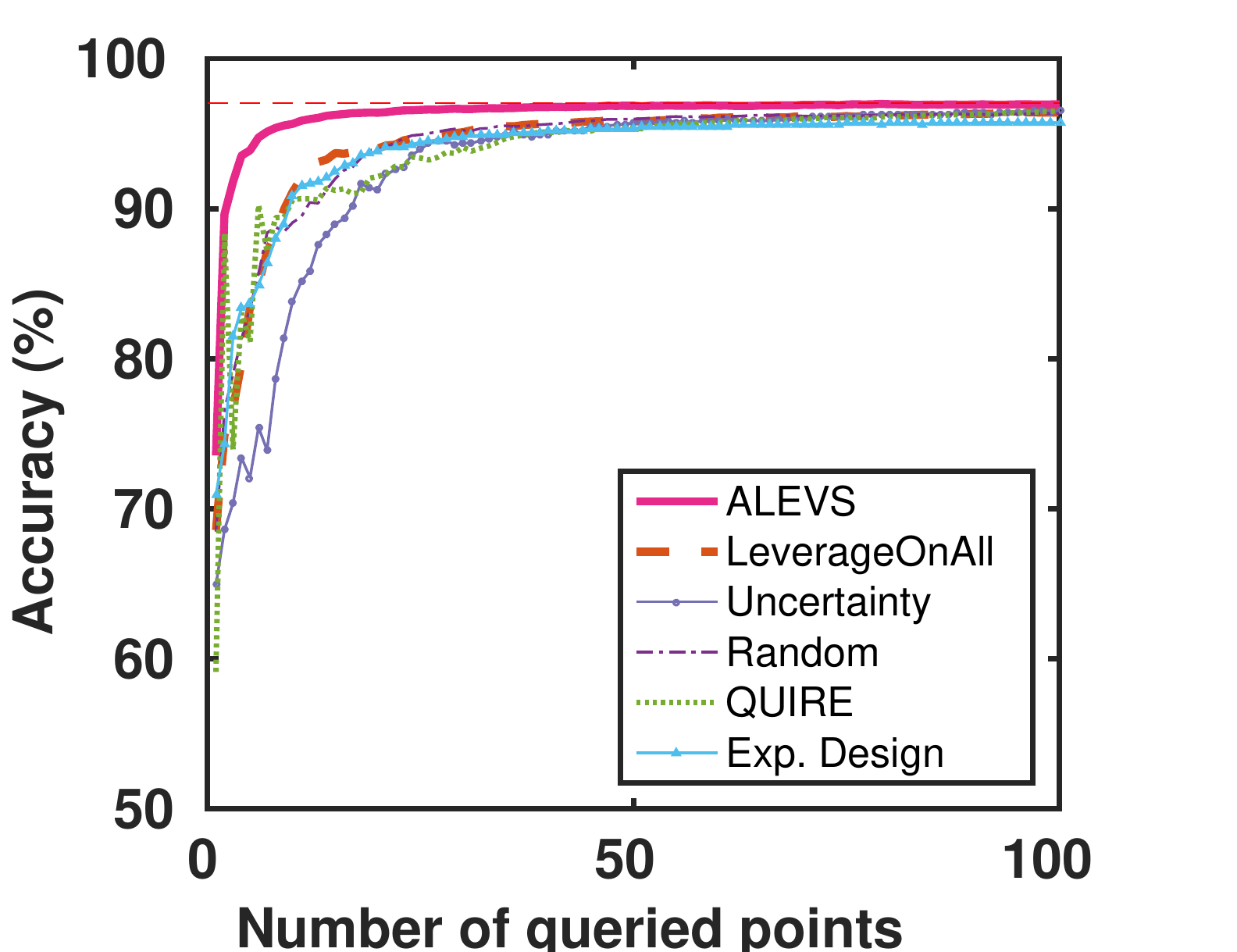}
		\label{twonorm_acc:seq}
	\end{subfigure}
	\vskip \newSkip
	\begin{subfigure}[b]{\newW\textwidth}
		\caption{\textit{ringnorm}}
		\vspace{\newSpace}
		\includegraphics[width=\textwidth]{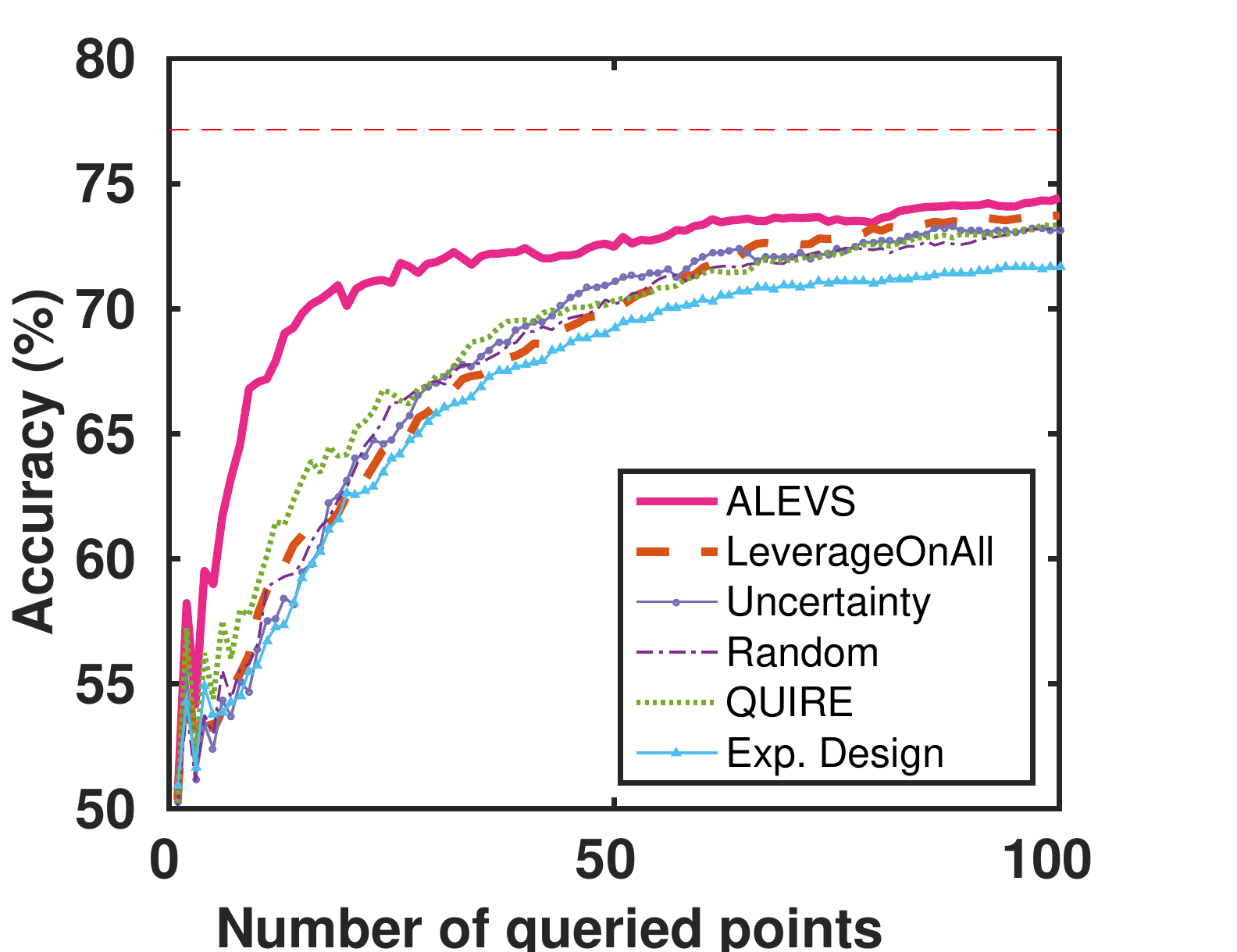}
		\label{ringnorm_acc:seq}
	\end{subfigure}
	~
	\begin{subfigure}[b]{\newW\textwidth}
		\caption{\textit{spambase}}
		\vspace{\newSpace}
		\includegraphics[width=\textwidth]{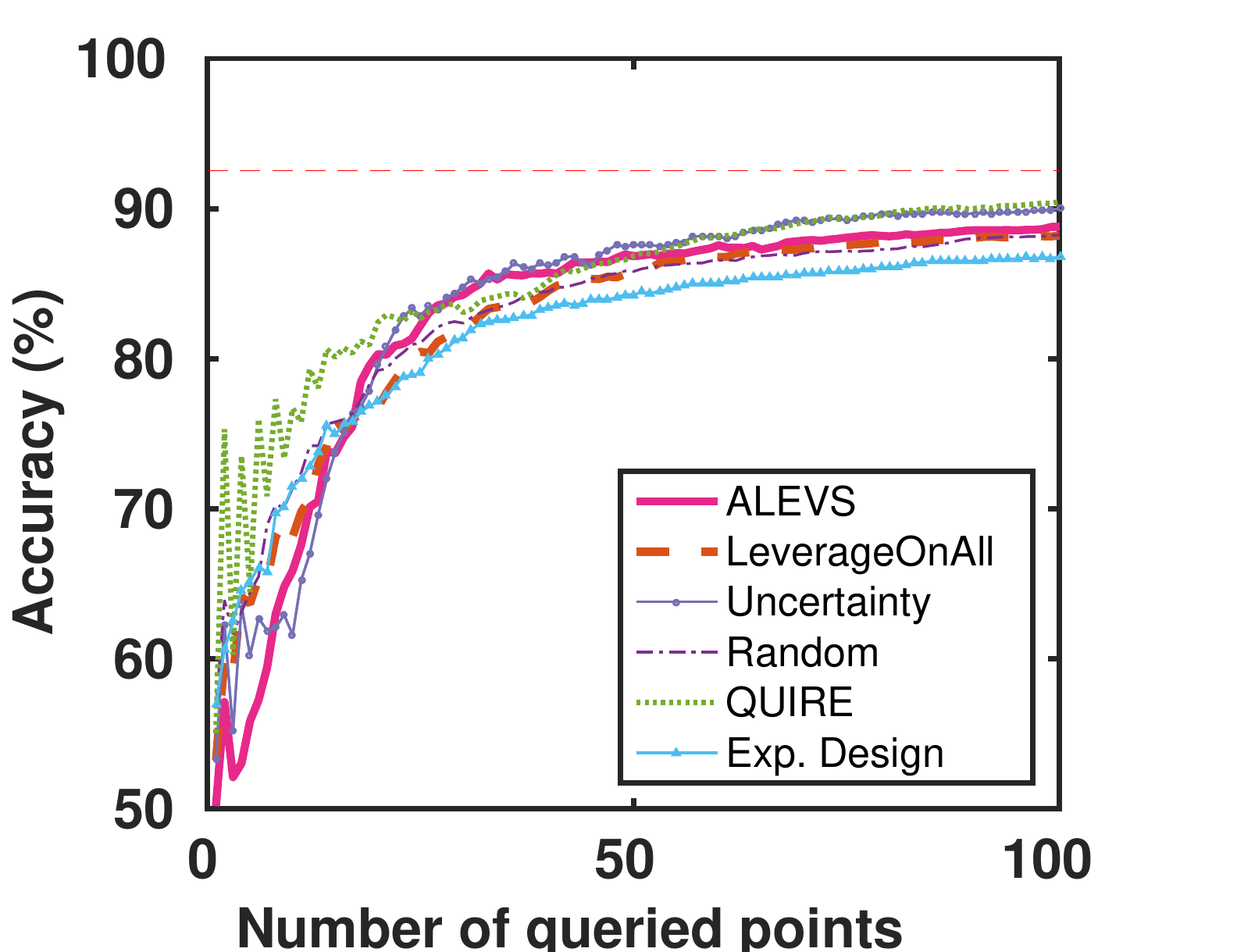}
		\label{spambase_acc:seq}
	\end{subfigure}
	\caption{Comparison of \texttt{ALEVS} with other methods on classification accuracy. The dashed line indicates the accuracy obtained when model is trained with all of the training data.}
	\label{figs:seq}
\end{figure*}

We observe that \texttt{ALEVS} outperforms random sampling and uncertainty sampling  in  almost all datasets (Fig.~\ref{figs:seq}). Exceptions to this are the \textit{USPS} and \textit{spambase} dataset, for which \texttt{ALEVS} 
performs as good as the uncertainty sampling but not better.  \texttt{ALEVS}' performance is consistently better than random sampling in the first 50 iterations of active sampling (Table~\ref{wtl:50}). For iterations between $50-100$, the two methods tie in \textit{spambase} and \textit{UvsV} datasets and random sampling performs better in  \textit{digit1} (Table~\ref{wtl:100}). In \textit{UvsV} dataset uncertainty sampling works well against all methods. When compared to transductive experimental design, \texttt{ALEVS} performs better in all datasets except  \textit{digit1} in the iterations 50-100, where the two methods tie (Table~\ref{wtl:50} and Table~\ref{wtl:100}).

When comparing the performance of \texttt{ALEVS} against \texttt{QUIRE}, there are three different groups of datasets. First group of datasets comprise \textit{ringnorm} and \textit{twonorm}, for which \texttt{ALEVS} decisively outperforms \texttt{QUIRE}. In the second group of datasets, \texttt{ALEVS} either outperforms \texttt{QUIRE} or ties with it at a subset of the iterations. In \textit{digit1},  \texttt{ALEVS} outperforms in the first 50 iterations. For the \textit{g241c} dataset, \texttt{ALEVS} either ties or performs worse than \texttt{QUIRE} in early iterations but the performance of \texttt{QUIRE} is not consistent in early iteration (Fig.~\ref{g241c_acc:seq}).  In this dataset, \texttt{ALEVS} holds up  with \texttt{QUIRE} and outperforms it at later iterations. In the \textit{3vs5} dataset \texttt{QUIRE} and \texttt{ALEVS} tie in most of the 50 iterations. There are also datasets, where \texttt{ALEVS} lags behind \texttt{QUIRE}. These include \textit{UvsV}, \textit{USPS} and \textit{spambase}. For the \textit{UvsV}, few labels are sufficient to obtain good accuracy and the performances of different methods do not differ dramatically (Fig.~\ref{USPS_acc:seq}).  For the \textit{spambase} dataset, \texttt{ALEVS} shows promising performance around iterations 30 and 40 (Fig.~\ref{spambase_acc:seq}). One observes that generally \texttt{ALEVS} manages to find effective examples for querying in the early iterations. Therefore, a strategy that combines \texttt{ALEVS} with a method that performs poorly in early iterations but do better in later iterations -- such as uncertainty sampling -- could lead to a strong active learner. Such a hybrid classifier will be explored in future work.

To understand whether computing class specific kernel matrices have any merit, we compare \texttt{ALEVS} against LevOnAll,  we observe that \texttt{ALEVS} consistently outperforms it. Thus calculating the leverage scores within each class is better at finding the influential data points than calculating them on the whole pool.

\vspace{-2mm}
\begin{figure*}[ht]
	\centering
	\vskip \newSkip
	\begin{subfigure}[b]{\newW\textwidth}
		\caption{\textit{5:1 class ratio}}
		\vspace{\newSpace}
		\includegraphics[width=1.1\textwidth]{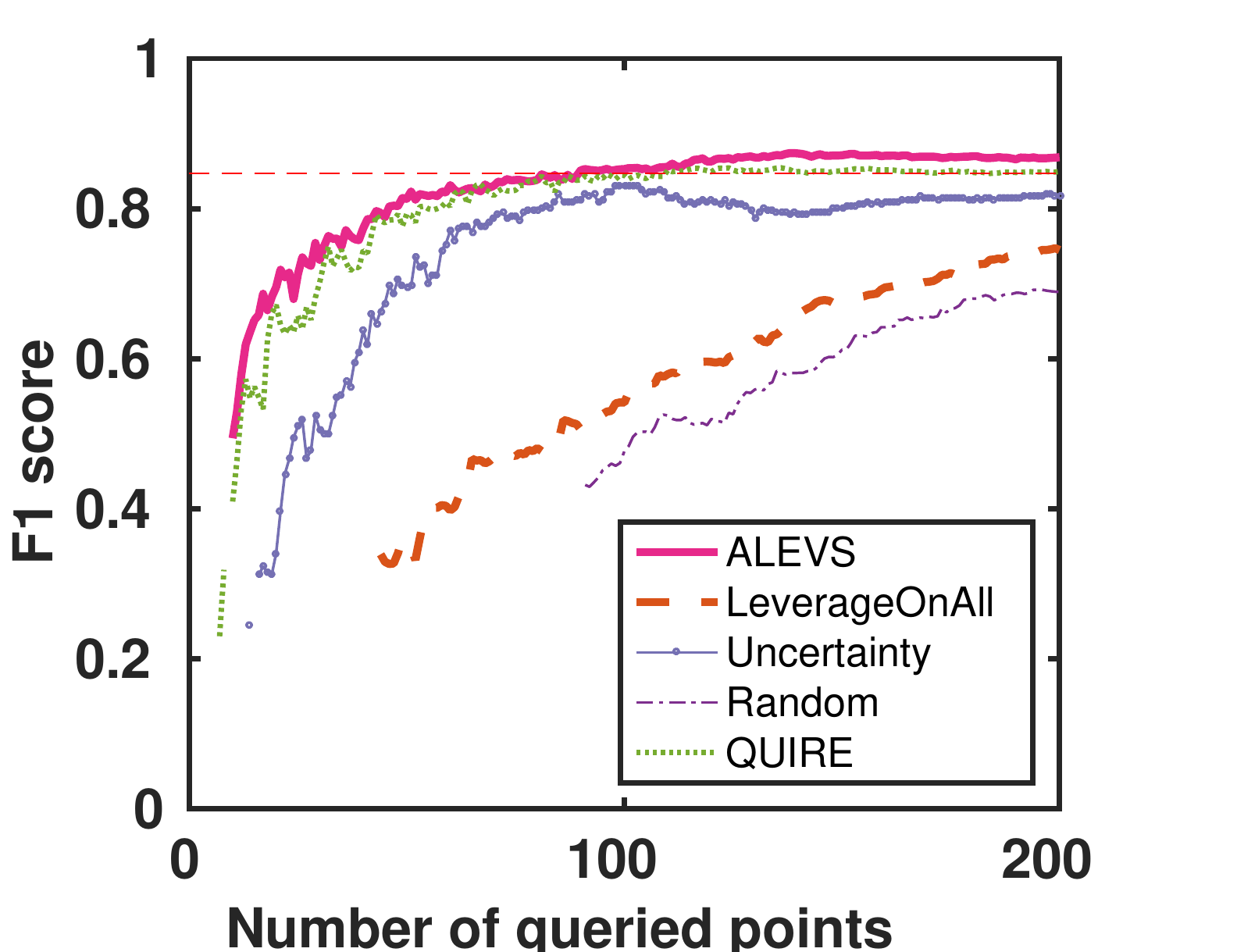}
		\label{digitimbalance5:seq}
	\end{subfigure}
	~
	\begin{subfigure}[b]{\newW\textwidth}
		\caption{\textit{10:1 class ratio}}
		\vspace{\newSpace}
		\includegraphics[width=1.1\textwidth]{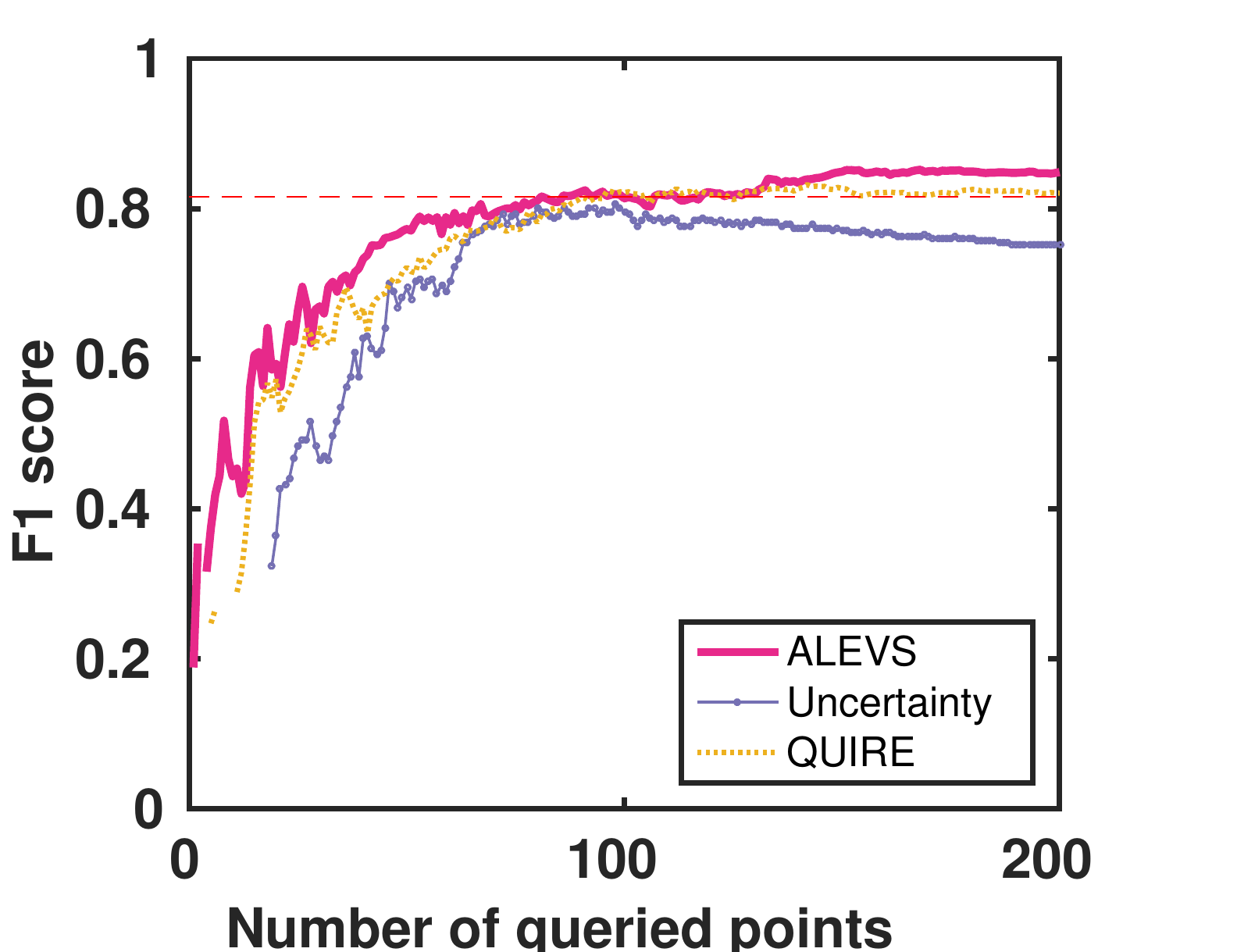}
		\label{digitimbalance10:seq}
	\end{subfigure}
	\caption{Comparison of \texttt{ALEVS} with other methods on imbalanced datasets. The dashed horizontal line indicates the F1 achieved when trained on the whole training data.}
\vspace{-2mm}
	\label{figs:imbalance}
\end{figure*}

\vskip \newSkip
\begin{table}[ht]
	\centering
	\caption{+/- class ratios of the queried sets by each method. }
\vspace{-2mm}
	\label{tbl:imbalance}
	\scalebox{0.85}{
		\begin{tabular}{ccccccc}
			\hline
			Class ratio& ALEVS & QUIRE & UNCERTAINTY & RANDOM & LEV-ON-ALL & EXP. DESIGN\\ \hline
			\hspace{0.5em}5:1              & 1.28  & 1.60  & 2.41       & \hspace{0.5em}5.32   & \hspace{0.5em}4.80   &inf*       \\ 
			10:1             & 1.31  & 1.64  & 2.30       & 10.55  & 10.19           &inf * \\ \hline
		\end{tabular}
	}
\end{table}

To understand how \texttt{ALEVS} reacts to class imbalance, we sample the \texttt{5vs3} dataset with two different class ratios, 5:1 and 10:1 and repeat the experiments on this dataset. The experimental set up is identical to that of in the previous section, the only difference being the adoption of F1 score in measuring performance.  As depicted in the Fig.~\ref{figs:imbalance}, \texttt{ALEVS} successfully copes with the class imbalance and outperforms other methods. The transductive experimental design method is unable to handle the class imbalance and returns F1 scores of 0; therefore, we exclude its results from each figure. Similarly, with 10:1 class ratio, Lev-On-All, and random sampling return very poor F1 scores and are excluded from the graph. To understand why different methods would handle the class imbalance differently, we analyze the class label distribution of the queried set. Table~\ref{tbl:imbalance} illustrates that those that are robust to class imbalance sample equally from both classes whereas those that fail sample close to the original class distribution. Since the classifier is provided with a balanced dataset, the overall training process is not hurt by the unequal class distribution. As the class label balance becomes detorioted, if one adopts cost-sensititve training methods, this could solve the problem but adds an extra complexity  layer to the problem. 

We also compare methods in term of their running times. The querying step of \texttt{ALEVS} involves the calculation of eigenvalue decomposition of the kernel matrices; however, in practice, this does not cause a computational bottleneck. We summarize the average CPU times for selecting one example in a single iteration from the unlabeled data pool in Appendix Fig.~\ref{figs:time-seq}.  As the figures show, \texttt{ALEVS} is as fast as uncertainty sampling.

\vspace{-5mm}
\section{Results for DBALEVS}
\label{sec:dbalevsresults}

We compare \texttt{DBALEVS} with the following approaches: 
\vspace{-1mm}
\begin{itemize}
	\item \textbf{Random sampling:} Randomly selects $b$ examples uniformly at random from the unlabelled pool.
	\item \textbf{Uncertainty sampling:} Selects  $b$ examples with maximal uncertainty. 
	\item \textbf{Top leverage sampling \texttt{(Top-Lev)}}: Computes the leverage score on the whole pool at the beginning without paying attention to class membership and selects the top $b$ examples based on their leverage scores.
	\item \textbf{Near-optimal batch mode active learning (\texttt{NearOpt}) :}  \texttt{NearOpt} \cite{chen2013near} selects a batch of instances using adaptive submodular optimization. 
\end{itemize}
\vspace{-2mm}
%
To evaluate the performance of \texttt{DBALEVS} we run experiments on six different datasets (Table~\ref{tbl:datasets-batch}). The details of these datasets can be found in Appendix Section X. Each dataset is divided into training and held out test sets. We start with four randomly selected labeled examples, two from each class. Batch size, $b$, is set to 10 and diversity tradeoff parameter $\alpha$ is set to 0.5 for each dataset except \textit{ringnorm}, wherein that dataset it is set to 0.1. At each iteration, the classifier is updated for all the methods with the training data, and the accuracy values are calculated on the same held-out test data. In all experiments, an SVM classifier with RBF kernel is used. For each dataset, the experiment is repeated $50$ times with random splitting of the training and the test data and random initial selection of labeled examples. For the set function defined in Definition~\ref{set-score} to be submodular, the kernel function should satisfy: $0\leq K(i,j)\leq 1$. We use RBF kernel in calculating the set scoring function, which is in this range. However, the method is compatible with other kernel functions as long as they are normalized within the range $[0, 1]$. For optimizing the set scoring function in Definition~\ref{set-score}, we use submodular function optimization toolbox \cite{krause2010sfo}.


\begin{table}
	\caption{Datasets for \texttt{DBALEVS} experiments, number of samples, features and the positive to negative class ratio (+/-)  are listed.}
	\centering
	\begin{tabular}{llllll}
		\hline\noalign{\smallskip}
		{\sc Dataset\quad} & {\sc Size} & {\sc \quad +/-} \\ \hline
		\noalign{\smallskip}
		\textit{autos} & 1986 x 11009 &\quad 1.00 \\
		\textit{hardware} & 1945 x 9877 &\quad 1.00 \\
		\textit{sport} & 1993 x 11148 &\quad 1.00 \\
		\textit{ringnorm} & 7400 x 20 &\quad 1.00 \\
		\textit{3vs5} & 13454 x 784 &\quad 1.20 \\
		\textit{4vs9} & 13782 x 784 &\quad 1.20 \\ \hline
	\end{tabular}
	\label{tbl:datasets-batch}
\end{table}




Fig.~\ref{figs-batch} shows the average classification accuracy values of \texttt{DBALEVS} and other approaches at each iteration of active sampling. Table~\ref{wtl:batch} summarizes the win, tie and lost counts of \texttt{DBALEVS} versus each of the competing methods on the 1-sided paired sample $t$-test at the significance level of 0.05.

\vspace{-3mm}
\begin{table}[!ht]
	\centering
	\caption{Win/Tie/Loss counts for \texttt{DBALEVS} against the competitor algorithm.}	
\vspace{-2mm}
\scalebox{0.85}{
	\begin{tabular}{|c|c|c|c|c|c|c|}
		\hline
		{\sc Dataset} & vs.\textsc{NearOpt} & vs.\textsc{Top-Lev} & vs.\textsc{Random} & vs.\textsc{Uncertainty} \\ \hline 
		\textit{ringnorm} & 20/0/0 & 0/1/19 & 20/0/0 & 20/0/0 \\ 
		\textit{autos} & 18/11/1 & 49/8/3 & 49/11/0 & 51/6/3 \\ 
		\textit{hardware} & 25/5/0 & 57/3/0 & 54/6/0 & 57/3/0 \\ 
		\textit{sport} & 25/5/0 & 51/9/0 & 46/14/0 & 54/6/0 \\ 
		\textit{4vs9} & 18/2/0 & 20/0/0 & 18/2/0 & 19/1/0 \\ 
		\textit{3vs5} & 17/3/0 & 18/2/0 & 18/2/0 & 19/1/0 \\ \hline
	\end{tabular}
\vspace{-2mm}
	\label{wtl:batch}
}
\end{table}
\vspace{-2mm}

\renewcommand{\newW}{0.4}
\begin{figure*}[!ht]
	\centering
	\vskip \newSkip
	\begin{subfigure}[b]{\newW\textwidth}
		\caption{\textit{autos}}
		\vspace{\newSpace}
		\includegraphics[width=\textwidth]{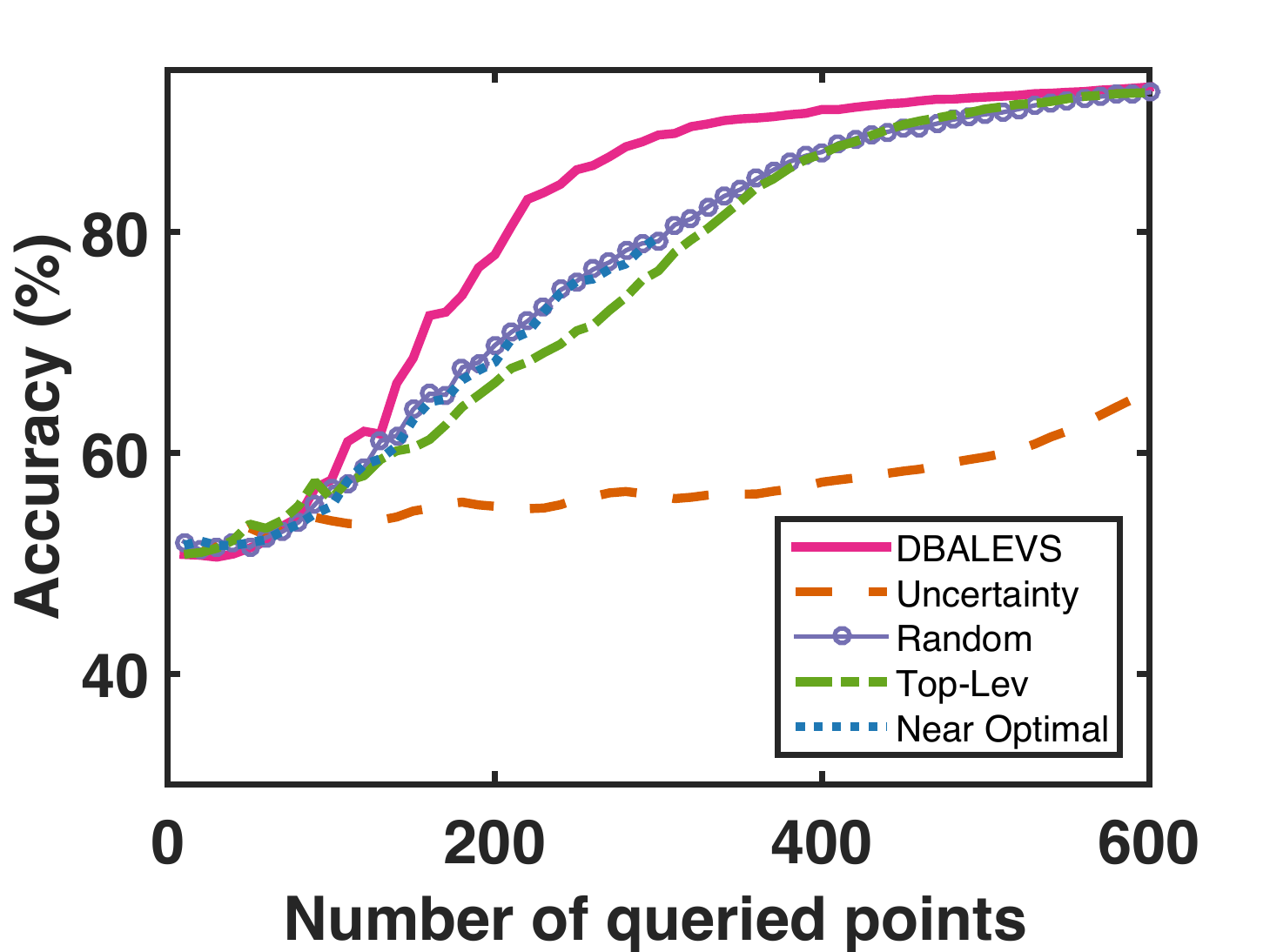}
		\label{autos}
	\end{subfigure}
	~
	\begin{subfigure}[b]{\newW\textwidth}
		\caption{\textit{hardware}}
		\vspace{\newSpace}
		\includegraphics[width=\textwidth]{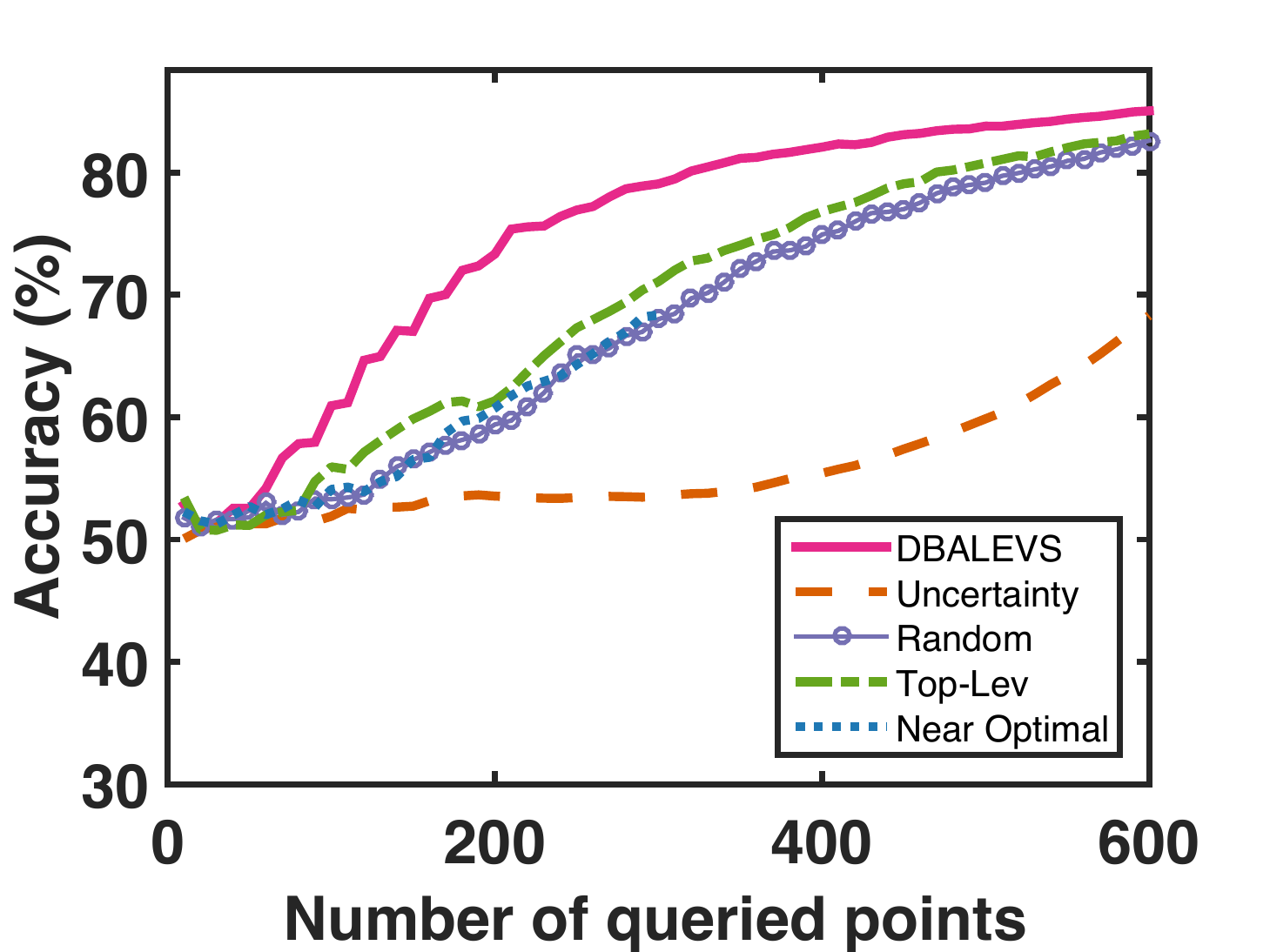}
		\label{hardware}
	\end{subfigure}
	\vskip \newSkip
	\begin{subfigure}[b]{\newW\textwidth}
		\caption{\textit{sport}}
		\vspace{\newSpace}
		\includegraphics[width=\textwidth]{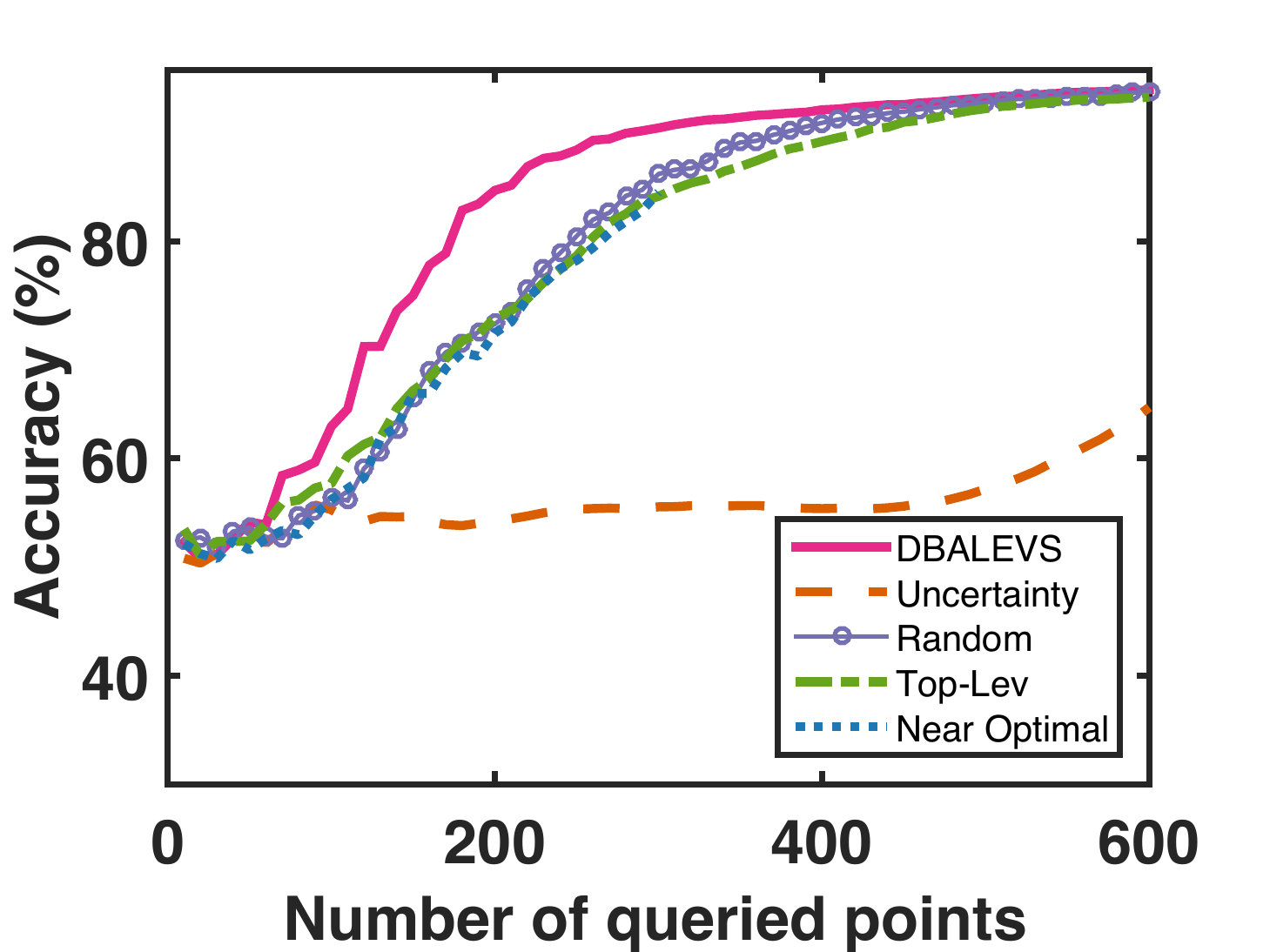}
		\label{sport}
	\end{subfigure}
	~
	\begin{subfigure}[b]{\newW\textwidth}
		\caption{\textit{ringnorm}}
		\vspace{\newSpace}
		\includegraphics[width=\textwidth]{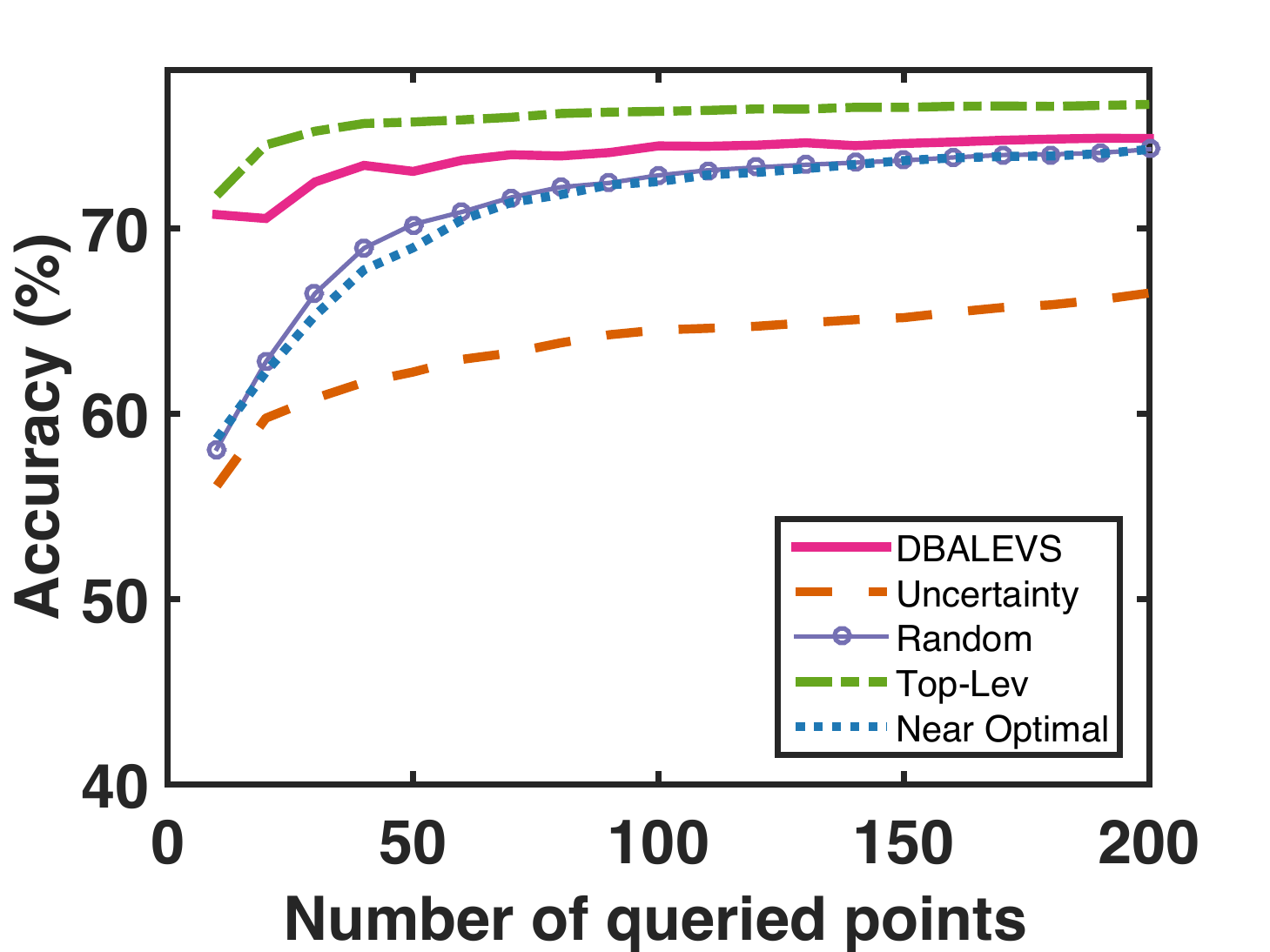}
		\label{ringnorm-batch}
	\end{subfigure}
	\vskip \newSkip
	\begin{subfigure}[b]{\newW\textwidth}
		\caption{\textit{3vs5}}
		\vspace{\newSpace}
		\includegraphics[width=\textwidth]{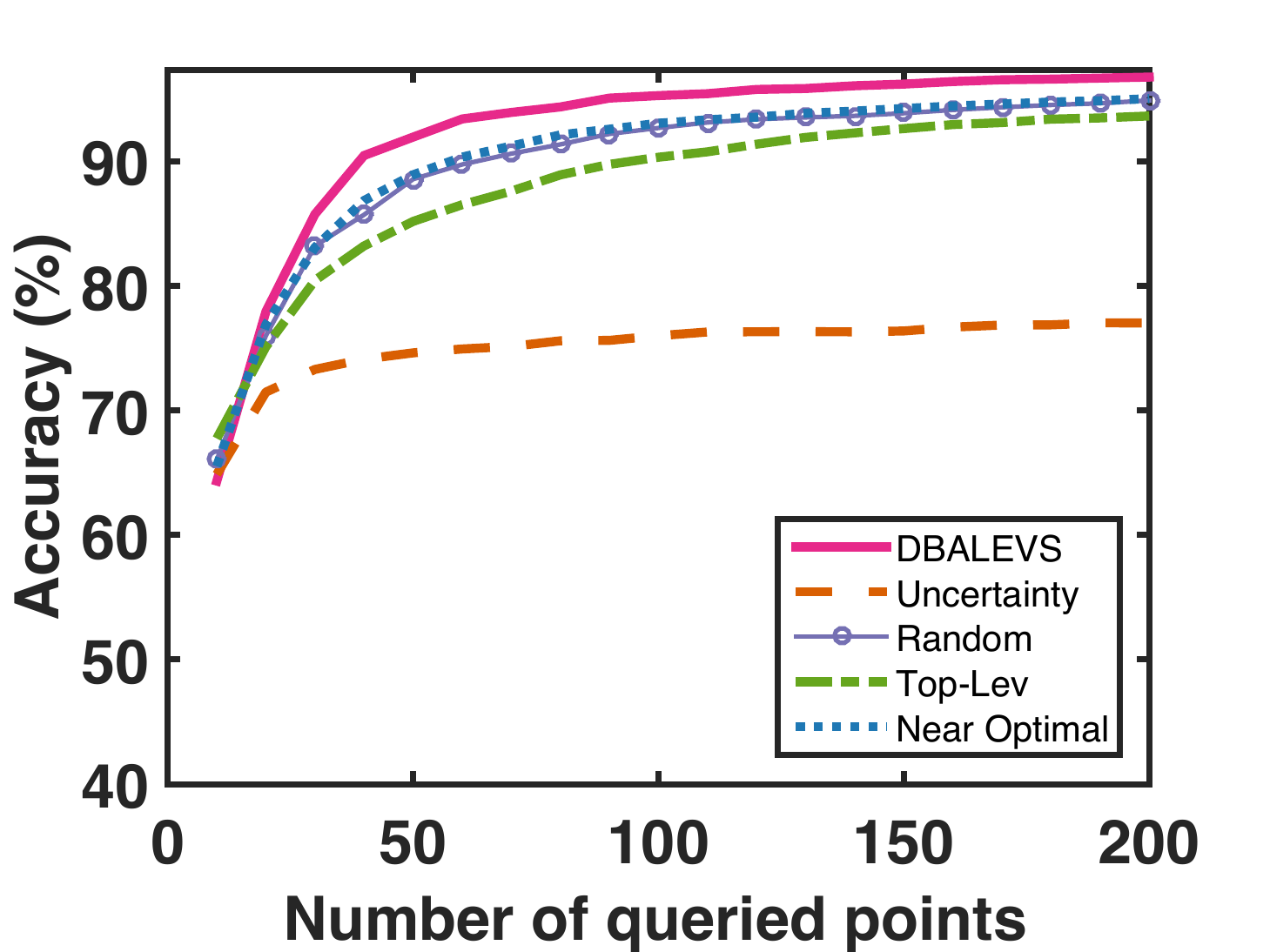}
		\label{3v5}
	\end{subfigure}
	~
	\begin{subfigure}[b]{\newW\textwidth}
		\caption{\textit{4vs9}}
		\vspace{\newSpace}
		\includegraphics[width=\textwidth]{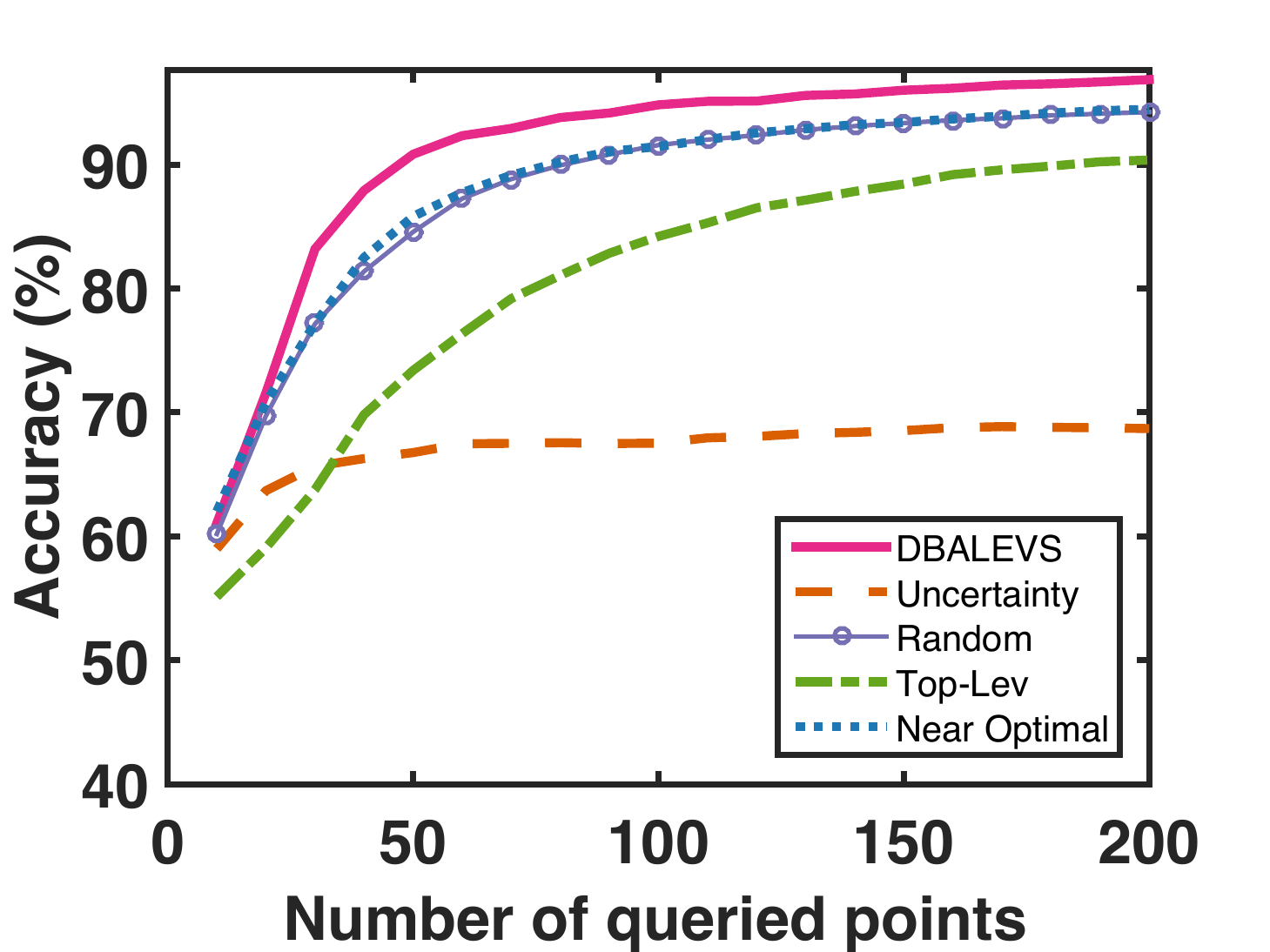}
		\label{4vs9}
	\end{subfigure}
\vspace{-4mm}
	\caption{Comparison of \texttt{DBALEVS} with other methods on classification accuracy.}
	\label{figs-batch}
\end{figure*}
\vspace{-4mm}

\texttt{DBALEVS} outperforms \texttt{NearOpt} in almost all cases. For all of the datasets, \texttt{DBALEVS} wins over \texttt{NearOpt} except for few tie cases  (See Table~\ref{wtl:batch}), and \texttt{DBALEVS} never loses against \texttt{NearOpt}. We also observe that the performances of \texttt{NearOpt} and random sampling are comparable. The results indicate that \texttt{DBALEVS} outperforms random sampling and uncertainty sampling approaches in all of the datasets. One interesting observation is, uncertainty sampling performs poorly in the batch mode setting. The main cause for this poor performance is the strong dependency of uncertainty sampling on the initial hypothesis. If the initial hypothesis formed from the initially labeled examples is not reasonable, then the query selection is affected by it because non-informative instances are selected at the subsequent iterations. One other drawback of this approach is that it fails to model the dependencies among the selected instances.

Another result is that random sampling performance is not bad and outperforms uncertainty sampling in all of the datasets. This surprising performance is also noted by others \cite{gu2014batch}. It could be attributed to the fact that datasets that are available for running these experiments are not truly random; therefore, if the batch size is large enough, the examples provide valuable information to learn the datasets. 

When compared to Top-Lev baseline, we observe that \texttt{DBALEVS} consistently other contenders, only exception being the \textit{ringnorm} dataset. One possible explanation pertinent to this data is the structure of the dataset: since \textit{ringnorm} dataset is artificially created from multivariate Gaussians, by querying points with maximal leverage scores, the learner receives labels from dense regions in different clusters without the need for diversification.
\vspace{-4mm}

\section{Conclusions and Future Work}
\label{sec:conclusion}
\vspace{-2mm}
In this study, we present a new query measure for active learning that is based on statistical leverage scores. We propose two novel algorithms based on this querying strategy: a sequential-mode algorithm \texttt{ALEVS} and a batch mode algorithm \texttt{DBALEVS}. Our experimentation on 8 different datasets shows that the  use of statistical leverage scores as an alternative strategy to find examples that are influential. \texttt{ALEVS} achieves superior performance compared to common baselines such as uncertainty sampling and random selection, an older method transductive experimental design. More importantly, \texttt{ALEVS} performs better or equally well in terms of classification accuracy when compared to a state-of-the-art approach, \texttt{QUIRE} \cite{quire}, which is documented to outperform other methods. Moreover, \texttt{ALEVS} runs much faster than \texttt{QUIRE}.We also show emprically that ALEVS is also robust to class imbalance. 

The second proposed algorithm, \texttt{DBALEVS} is designed to address the batch mode active learning setting.  We formulate a set scoring function that rewards examples with high-leverage scores and penalizes the inclusion of similar examples into the set.  We prove that this function is submodular, monotone and non-negative, which enables us to use a greedy algorithm for solving the submodular maximization problem that produces a solution that is constant factor approximate to the optimal solution. Our experiments on 6 different datasets  show that the idea of incorporating leverage scores and kernel matrix entries to find an influential and diverse batch of points is an effective strategy. \texttt{DBALEVS} performs well against common baselines random sampling and uncertainty sampling; and against \texttt{NearOpt} \cite{chen2013near}, which employs a newly introduced framework called adaptive submodularity. These results show that statistical leverage score is an effective measure for detecting which examples to query in a pool of unlabeled examples.

The work presented here can be extended in different directions. We observe \texttt{ALEVS} and \texttt{DBALEVS} are especially effective in early iterations a stage where many of the existing algorithms inadequately perform. Therefore, a future direction can be developing a hybrid strategy where the proposed method works in cooperation with other methods. For example, the framework proposed in this study does not incorporate any knowledge about the uncertainty of the class labels. 
One possible future direction would be to study the adaptability of these methods to stream-based selective sampling active learning approaches. Finally, to understand the effectiveness of the methods, we did not consider the more complex active learning scenarios such as the non-uniform cost of labels and noisy, reluctant experts. The general framework presented here can be further investigated to incorporate these alternative settings. 
\vspace{-4mm}

\begin{acknowledgements}
O.T. acknowledges support from Bilim Akademisi - The Science Academy, Turkey under the BAGEP program.
\end{acknowledgements}

\small
\bibliographystyle{abbrv}

\bibliography{ActiveLearning}   

\normalsize

%
%

\newpage
\appendix

\section*{APPENDIX}
\setcounter{section}{1}
\subsection{Notation table}
\label{app:notation}
\begin{table}[!ht]
	\centering
	\caption{Notation used throughout the article.
		\label{tbl:notations}}{
		\begin{tabular}{|l|l|}
			\hline
			{\bf Symbol} & {\bf Explanation} \\ \hline
			$\mathcal X$ & input space \\
			$\mathcal Y$ & output space \\
			$\mathcal D$ & dataset \\ \hline
			$\vec x$ & feature vector of a data point \\
			$ q$ & queried example \\
			$y$ & class label \\
			$\mathcal O$ & labeling oracle \\ 
			$h$ & classifier \\ \hline
			$\tau$ & threshold parameter for rank parameter selection \\ 
			$b$ & batch size\\
			$S_\mathrm q$ & queried batch \\
			$\alpha$ & diversity trade-off parameter \\ \hline
			$\ell$ & statistical leverage score\\
			$ \lambda$ & eigenvalue \\
			$k$ & rank parameter for calculating truncated leverage scores \\ \hline
			$\mat X$ & input feature matrix \\
			$\Phi$ & feature mapping \\
			$K(i,j)$ & kernel function evaluated on examples $i$ and $j$ \\
			$\mat K$ & kernel matrix \\
			$\mat K_{(i)}$ & $i$-th row of $\mat K$\\
			$\mat K_{ij}$ & element in $i$-th row and $j$-th column of matrix $\mat K$ \\
			$p$ & kernel parameter configuration \\ 
			$\sigma$ & scale parameter of RBF kernel \\
			$d$ & degree of polynomial kernel \\
			$c$ & coefficient of polynomial kernel \\
			\hline
	\end{tabular}}
\vspace{-10mm}
\end{table}
\vspace{-10mm}
\subsection{Proving $F$ is monotone}
\normalsize
\label{app:proof1}

\begin{proposition}[Monotonicity]
	\label{mon}
	$F$ is a monotonically non-decreasing set function when input set size is at most $M$ and for kernel values $K(i, j) \in [0, 1]$ $\forall i $ and $j$.
\end{proposition}
\vspace{-6mm}
\begin{proof}
	Consider two arbitrary sets, $A$ and $B$, where $A\subseteq B\subseteq V$ . And let $ t = \left\vert{B}\right\vert - \left\vert{A}\right\vert$, and $\left\vert{A}\right\vert \leq M$ and $\left\vert{B}\right\vert \leq M$. We need to show that the following inequality holds:
	\begin{eqnarray}
	& F(A) \leq F(B)
	\label{eq:mono}
	\end{eqnarray}
	Using Definition~\ref{set-score} for $F$:\\
	{\footnotesize
		\begin{align*}
		F(A) - F(B) 
		&= \bigg(\sum_{i\in A}(\ell_i+1) - \frac{\alpha}{M}\sum_{\substack{i,j \in A \\ i \neq j} } K(i,j) \bigg) 
		-
		\bigg(\sum_{i\in B}(\ell_i+1) - \frac{\alpha}{M} \sum_{\substack{i,j \in B \\ i \neq j} }K(i,j)\bigg) \\
		&=\bigg(\sum_{i\in A} \ell_i + \left\vert{A}\right\vert - \frac{\alpha}{M}\sum_{\substack{i,j \in A \\ i \neq j }}K(i,j)\bigg) 
		-
		\bigg(\sum_{i\in B}\ell_i + \left\vert{B}\right\vert - \frac{\alpha}{M}\sum_{\substack{i,j \in B \\ i \neq j}}K(i,j)\bigg) &&\\
		&=\frac{\alpha}{M}\bigg( \sum_{\substack{i,j \in B\\ i \neq j}}K(i,j) - \sum_{\substack{i,j \in A\\ i \neq j}}K(i,j)\bigg)
		-
		\sum_{i\in B}\ell_i + \sum_{i\in A}\ell_i - \left\vert{B}\right\vert + \left\vert{A}\right\vert &&\\
		&=\frac{\alpha}{M}\bigg( \sum_{\substack{i \in B\setminus A \\ j\in A}} K(i,j) + \sum_{\substack{i,j \in B\setminus A \\ i \neq j}}K(i,j) \bigg) - \sum_{ i \in B\setminus A} \ell_i - t 
		\end{align*}
	}
	The leftmost summation calculated over $t \left\vert{A}\right\vert$ terms and the second one sums over $t^2$ terms. Using the fact $K(i,j) \leq 1$, these terms can be at most $t \left\vert{A}\right\vert$ and $t^2$, respectively. Additionally, since $ 0 \leq \ell_i \leq 1$ the minimum value that $\sum_{ i \in B\setminus A} \ell_i$ can take is 0. Then the following inequality holds:
	\begin{align*}
	F(A) - F(B) &\leq \frac{\alpha}{M} \bigg( t \left\vert{A}\right\vert +t^2 \bigg) - t\\
	&\leq t \bigg( \frac{\alpha}{M} ( \left\vert{A}\right\vert + t ) - 1\bigg)\\
	&\leq t (\alpha - 1)\\
	&\leq 0
	\end{align*}
	Passing from the second line to the third line we used the fact that, $\left\vert{B}\right\vert \leq M$. As $F(A) - F(B) \leq 0$, $F$ is monotonically non-decreasing for sets with sizes smaller than or equal to $M$.
\end{proof}

\subsection{Proving $F$ is non-negative }
\label{app:proof2}
\begin{proposition}[Non-negativity]
	$F$ is a non-negative set function for sets with cardinality smaller than or equal to $M$.
\end{proposition}

\begin{proof}
	For $F$ to be non-negative, the following statement should hold for sets with cardinality at most $M$:
	\begin{eqnarray} 
	& \forall S\subseteq V, F(S) \geq 0\nonumber
	\end{eqnarray}
	$F(S)$ is defined as follows:
	\begin{align*}
	F(S) & = \sum_{i\in S}(\ell_i+1) - \frac{\alpha}{M}\sum_{\substack{i,j \in S \\ i \neq j}}K(i,j)  \\ 
	& =\sum_{i\in S}\ell_i + M - \frac{\alpha}{M}\sum_{\substack{i,j \in S \\ i \neq j}}K(i,j) \\
	& \geq 0 + M - \frac{\alpha}{M}( M^2 - M)\\
	& \geq 0 + M - \alpha(M-1)\\
	& \geq M (1 - {\alpha} ) + \alpha
	\end{align*}
	Moving from equality to inequality (line 2 to 3), we use the facts that the minimum value  $l_i$ can take is zero, the maximum value of $K(i,j)$  is 1, and $|S| \le M$. Thus, the summation of the leverage scores is minimum 0 and  kernel terms can be at most $M^2 - M$, which is the number of elements in the matrix excluding the diagonals. 
	Since $\alpha \in [0,1]$, $(1 - {\alpha} ) \geq 0$; thereby, $F(S) \geq 0 $. This completes the proof that $F(S)$ is non-negative when the chosen set size is bounded with $M$.
\end{proof}

\subsection{Datasets}
\label{app:datasets}
The following datasetsets are used in the ALEVS experiments.
The \textit{digit1}, \textit{g241c}, \textit{USPS} datasets are from\footnote{\url{http://olivier.chapelle.cc/ssl-book/benchmarks.html}} \cite{chapelle2010semi}. The \textit{spambase} and \textit{letter} datasets are obtained from \cite{Lichman:2013}. The \textit{letter} dataset is a multi-class dataset; we select a letter pair that are difficult to distinguish: \textit{UvsV}. Similarly, we sample 3 and 5 digits from the \textit{MNIST} dataset as \textit{3vs5}, since they are one of the most confused pairs in the \textit{MNIST} dataset \cite{lecun1998gradient}, obtained from\footnote{\url{http://yann.lecun.com/exdb/mnist/}}. Finally, \textit{twonorm} and \textit{ringnorm} are culled from\footnote{\url{http://www.cs.toronto.edu/~delve/data/twonorm/desc.html}} and\footnote{\url{http://www.cs.toronto.edu/~delve/data/ringnorm/desc.html}} which are implementations of \cite{breiman1996bias}. We use a random subsample of 2000 examples for \textit{ringnorm}, \textit{twonorm}, \textit{spambase}, and \textit{3vs5} because the running time for \texttt{QUIRE} is prohibitively long. The description of these datasets are given in Table~\ref{tbl:datasets}.

The following datasetsets are used in the DBALEVS experiments. We used the \textit{autos}, \textit{hardware} and \textit{sport} tasks in the 20-newsgroups dataset\footnote{\url{http://qwone.com/~jason/20Newsgroups/}}. These subtopics that are picked because they are harder to differentiate. \textit{autos} involves classification of \textit{rec.autos} and \textit{rec.motorcycles} topics. \textit{hardware} involves classifying \textit{comp.sys.ibm.pc.hardware} and \textit{comp.sys.m\\ac.hardware} topics and lastly the \textit{sport} dataset involves classification of \textit{rec.spor\\t.baseball} and \textit{rec.sport.hockey} topics. We use a bag-of-words representation for features in these datasets. Similarly, 3-5 and 4-9 digit pairs from the \textit{MNIST} dataset \cite{lecun1998gradient} are sampled to create the \textit{3vs5} and \textit{4vs9} classification tasks. Finally, \textit{ringnorm} is culled from\footnote{\url{http://www.cs.toronto.edu/~delve/data/ringnorm/desc.html}} which is an implementation of \cite{breiman1996bias}. The description of these datasets and the parameters chosen for each of the dataset are listed in Table~\ref{tbl:datasets-batch}. 

\subsection{Effect of target rank k}

\newcommand{\newWk}{0.48}
\begin{figure*}[!h]
	\centering
	\begin{subfigure}[b]{\newWk\textwidth}
		\caption{\textit{digit1}}
		\vspace{\newSpace}
		\includegraphics[width=\textwidth]{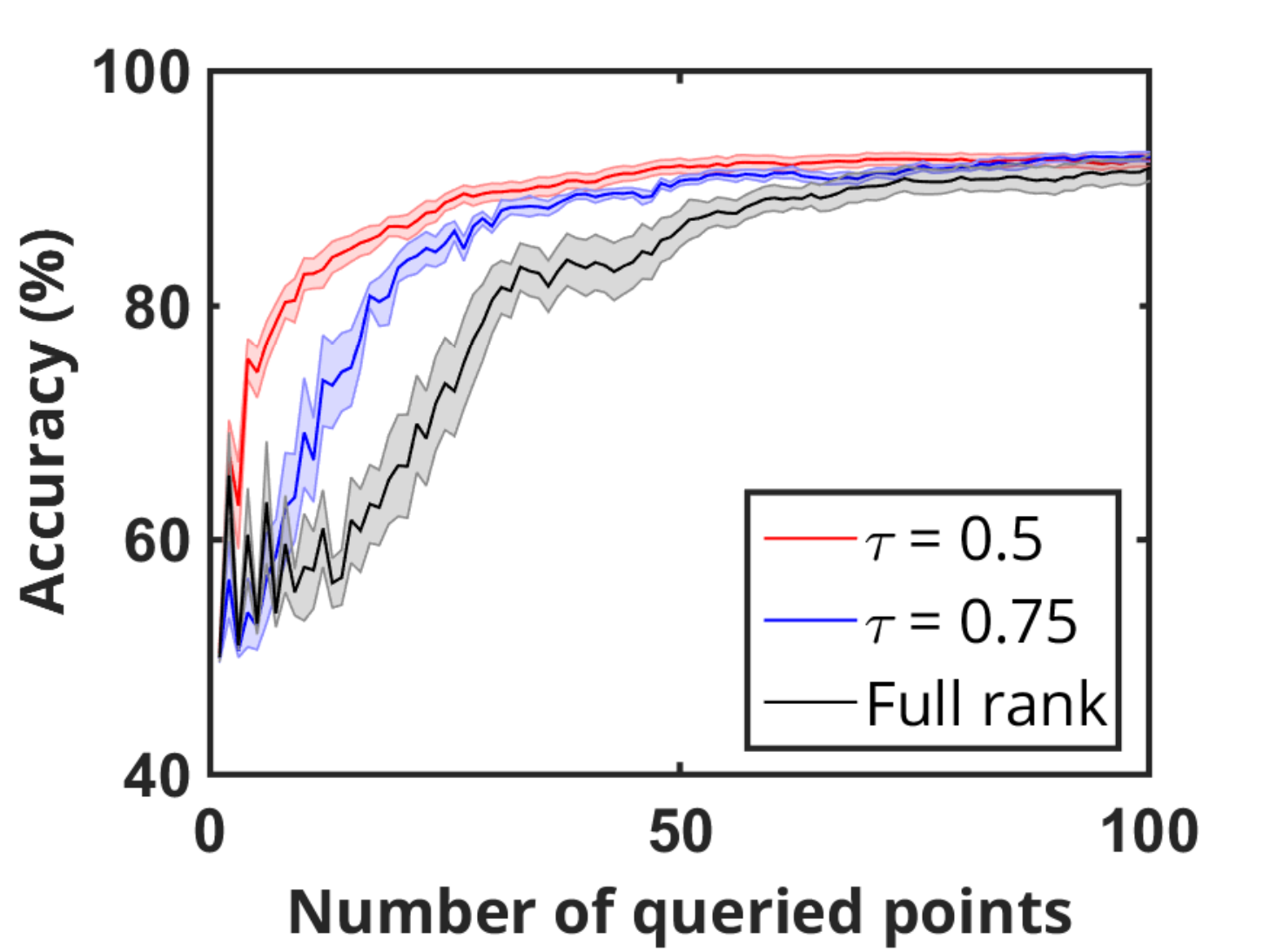}
		\label{digit1_k}
	\end{subfigure}
	~
	\begin{subfigure}[b]{\newWk\textwidth}
		\caption{\textit{twonorm}}
		\vspace{\newSpace}
		\includegraphics[width=\textwidth]{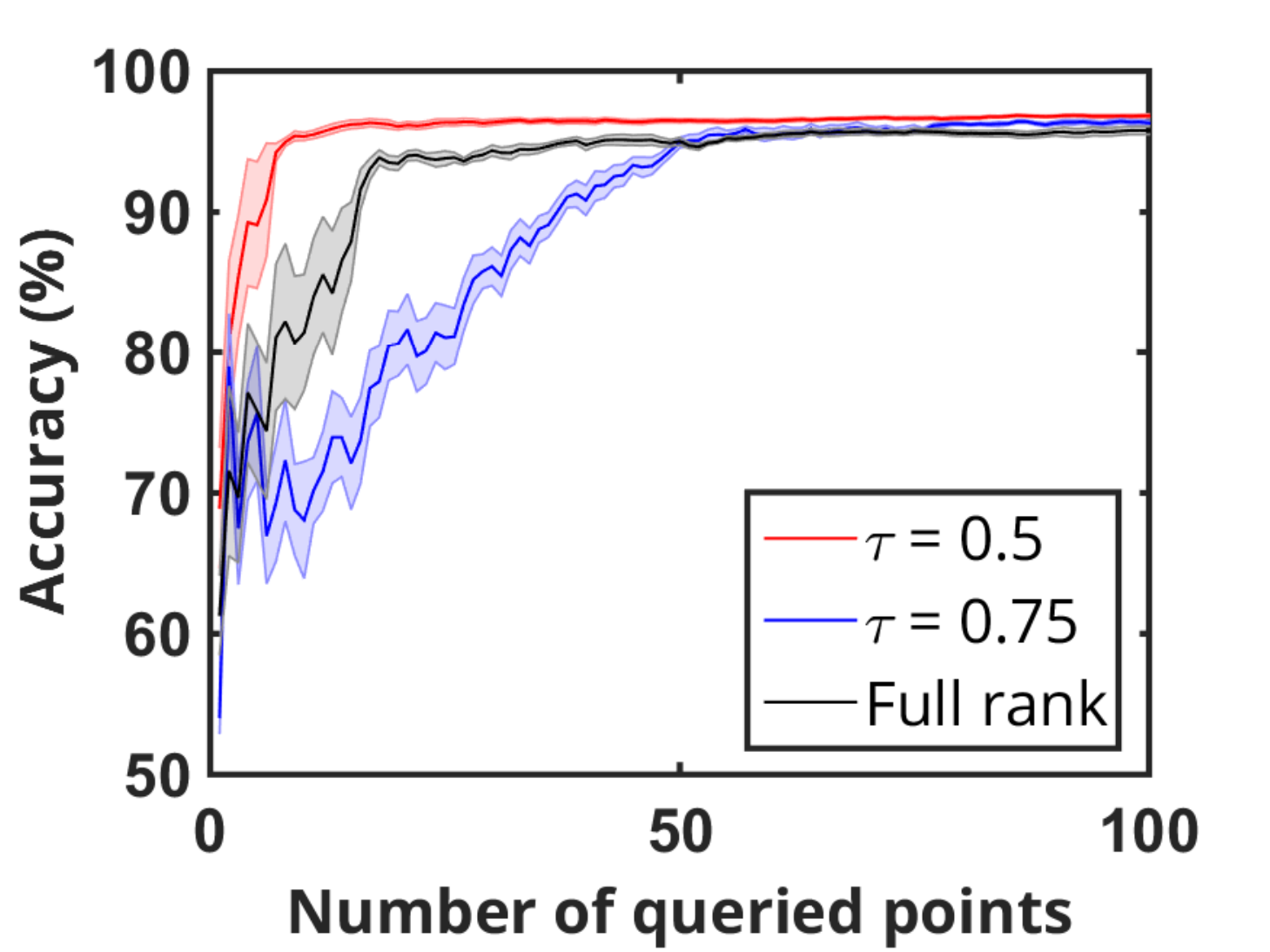}
		\label{twonorm_k}
	\end{subfigure}
	\begin{subfigure}[b]{\newWk\textwidth}
		\caption{\textit{ringnorm}}
		\vspace{\newSpace}
		\includegraphics[width=\textwidth]{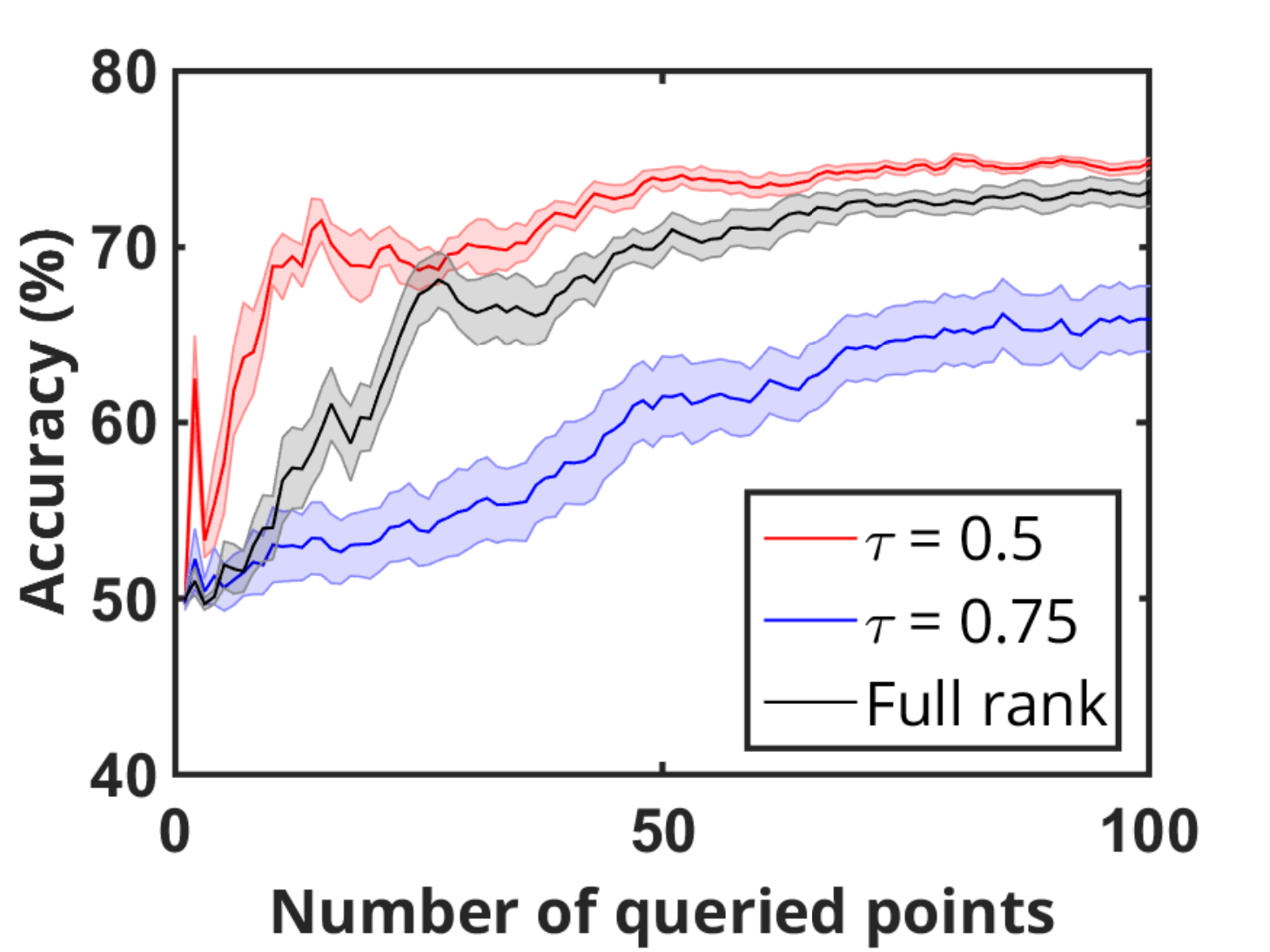}
		\label{ringnorm_k}
	\end{subfigure}
	~ 
	\begin{subfigure}[b]{\newWk\textwidth}
		\caption{\textit{3vs5}}
		\vspace{\newSpace}
		\includegraphics[width=\textwidth]{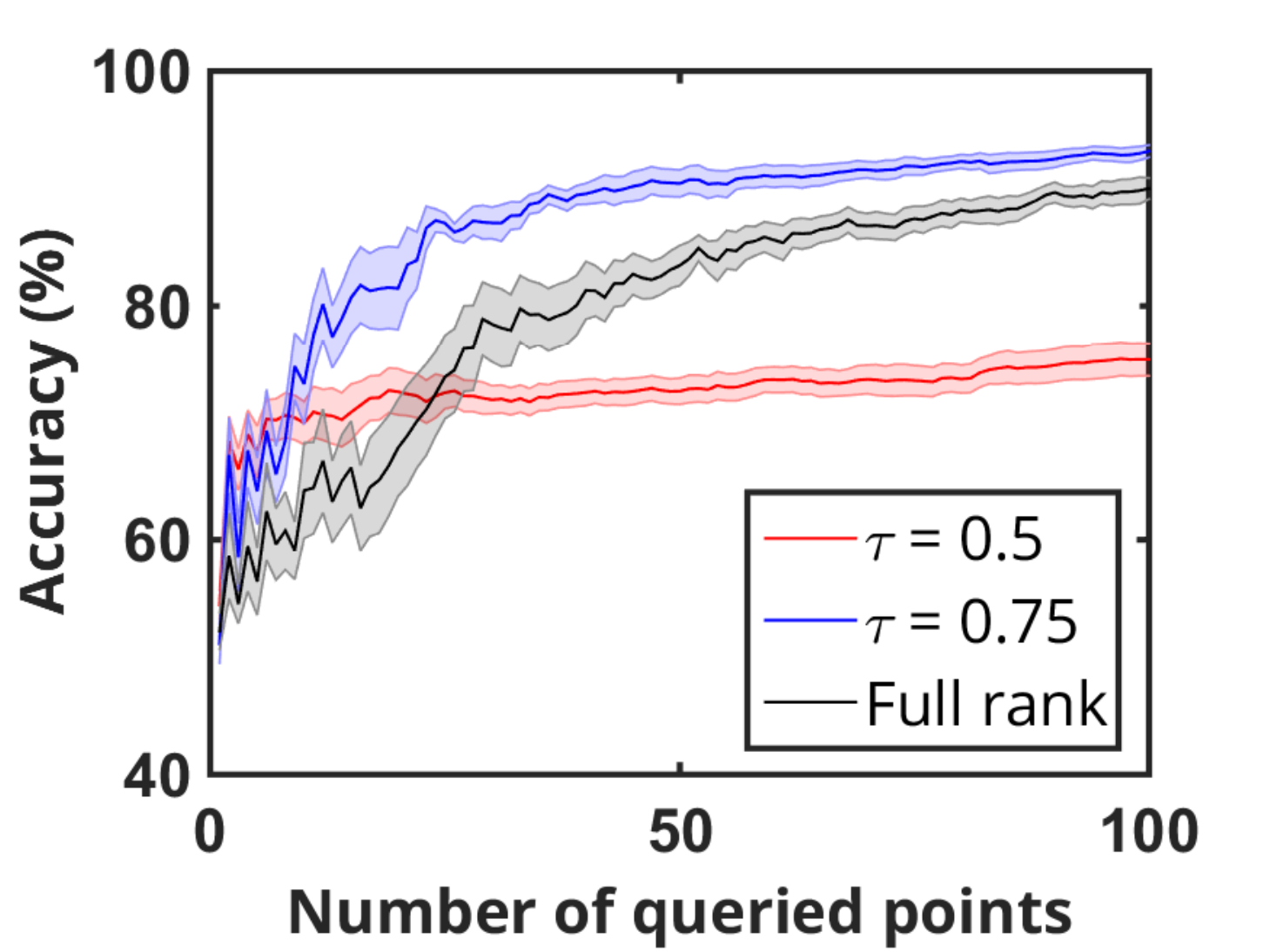}
		\label{MNIST_k}
	\end{subfigure}
	\caption{Effect of target rank $k$ selected by threshold $\tau$ on test set accuracy for \texttt{ALEVS} (Sequential-mode).}
	\label{figs:k}

\end{figure*}

One parameter that has a large impact on the performance of \texttt{ALEVS} is the target low-rank parameter $k$. In this work, we adaptively select the value of $k$ for negative and positive kernel matrices at each iteration by setting a threshold on the variance for top-$k$ dimensional eigenspace as described in \texttt{RankSelector} algorithm. We further analyze the effect of $\tau$ by varying these thresholds; experimented on four datasets with three different $\tau$ values. Accuracies shown in Fig.~\ref{figs:k} are averages computed over 10 random experiments. Selecting the full rank option for computing leverage scores does not necessarily provide the best performance. The low-rank representation acts as a regularizer and focuses on the core dimensions that matter in the datasets. For the datasets \textit{digit1}, \textit{twonorm}, \textit{ringnorm}, $\tau=0.5$ works best, whereas for \textit{3vs5} threshold value 0.75 is a better choice and $\tau=0.5$ is the worst choice. This difference is expected, as the eigenvalue spectra of the matrices are different.

\subsection{ALEVS runtime performance}


\begin{figure*}[p]
	\centering
	\begin{subfigure}[b]{\newW\textwidth}
		\caption{\textit{digit1}}
		\vspace{\newSpace}
		\includegraphics[width=\textwidth]{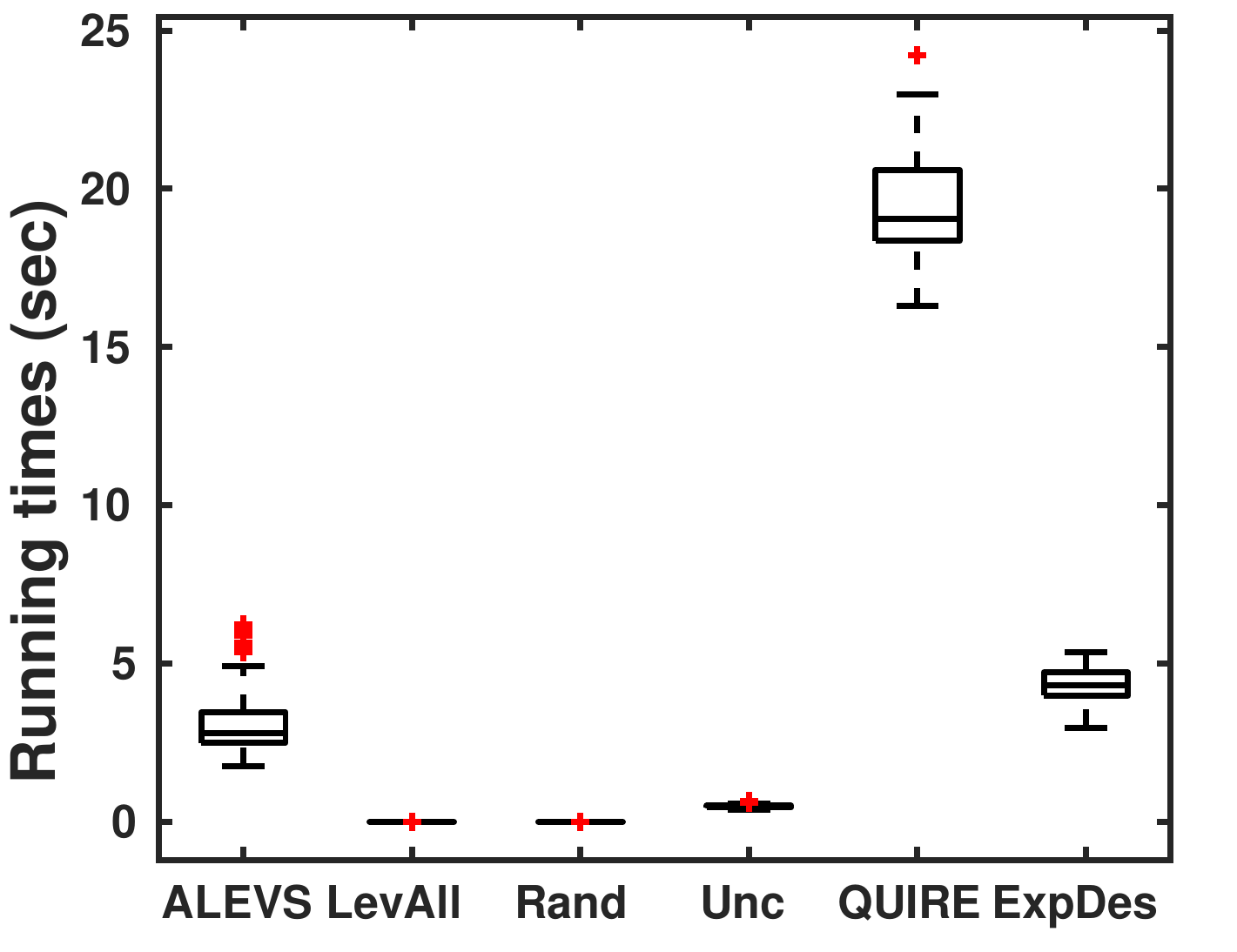}
		\label{digit1_timebox}
	\end{subfigure}
	~
	\begin{subfigure}[b]{\newW\textwidth}
		\caption{\textit{3vs5}}
		\vspace{\newSpace}
		\includegraphics[width=\textwidth]{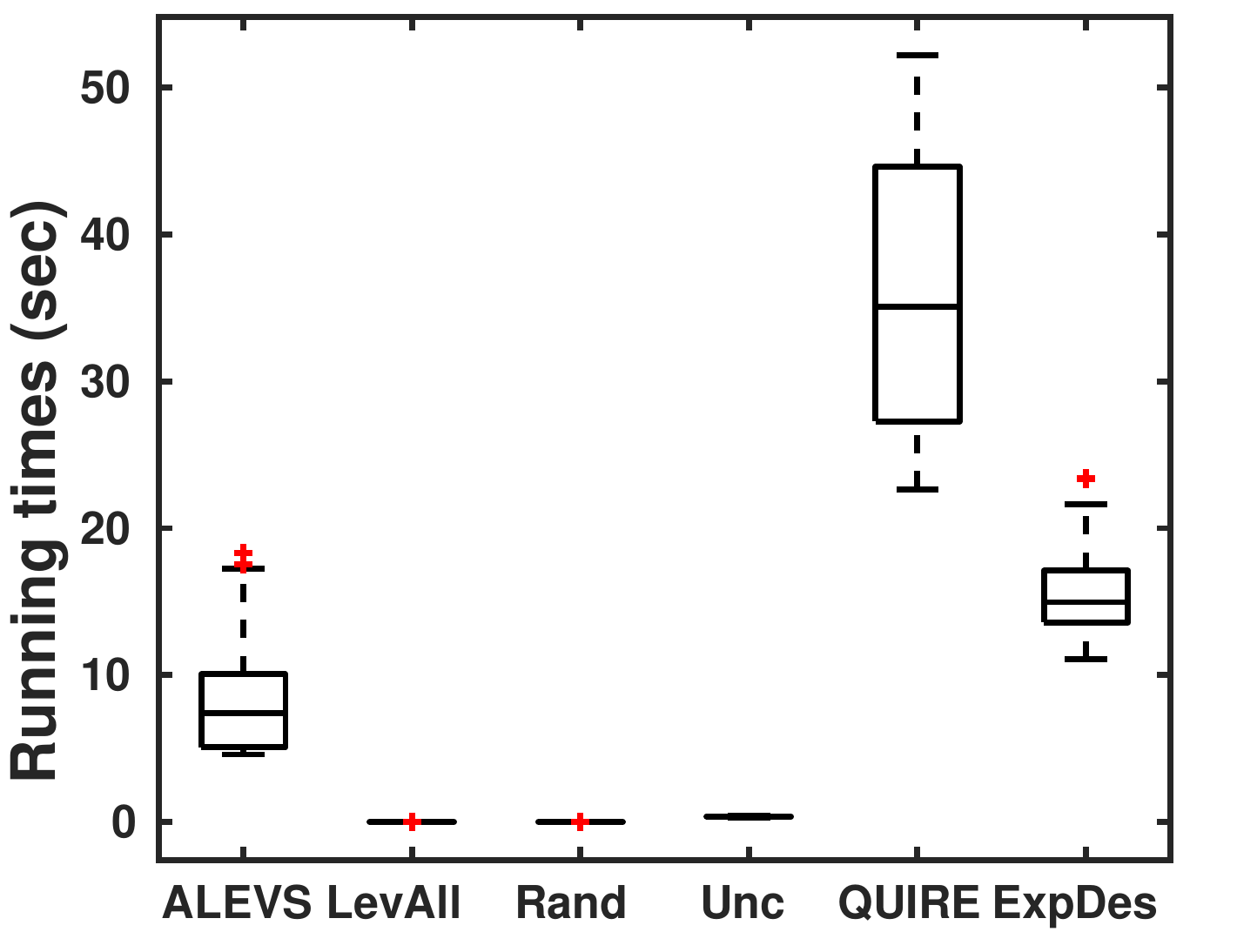}
		\label{MNIST_timebox}
	\end{subfigure}
	\vskip \newSkip
	\begin{subfigure}[b]{\newW\textwidth}
		\caption{\textit{g241c}}
		\vspace{\newSpace}
		\includegraphics[width=\textwidth]{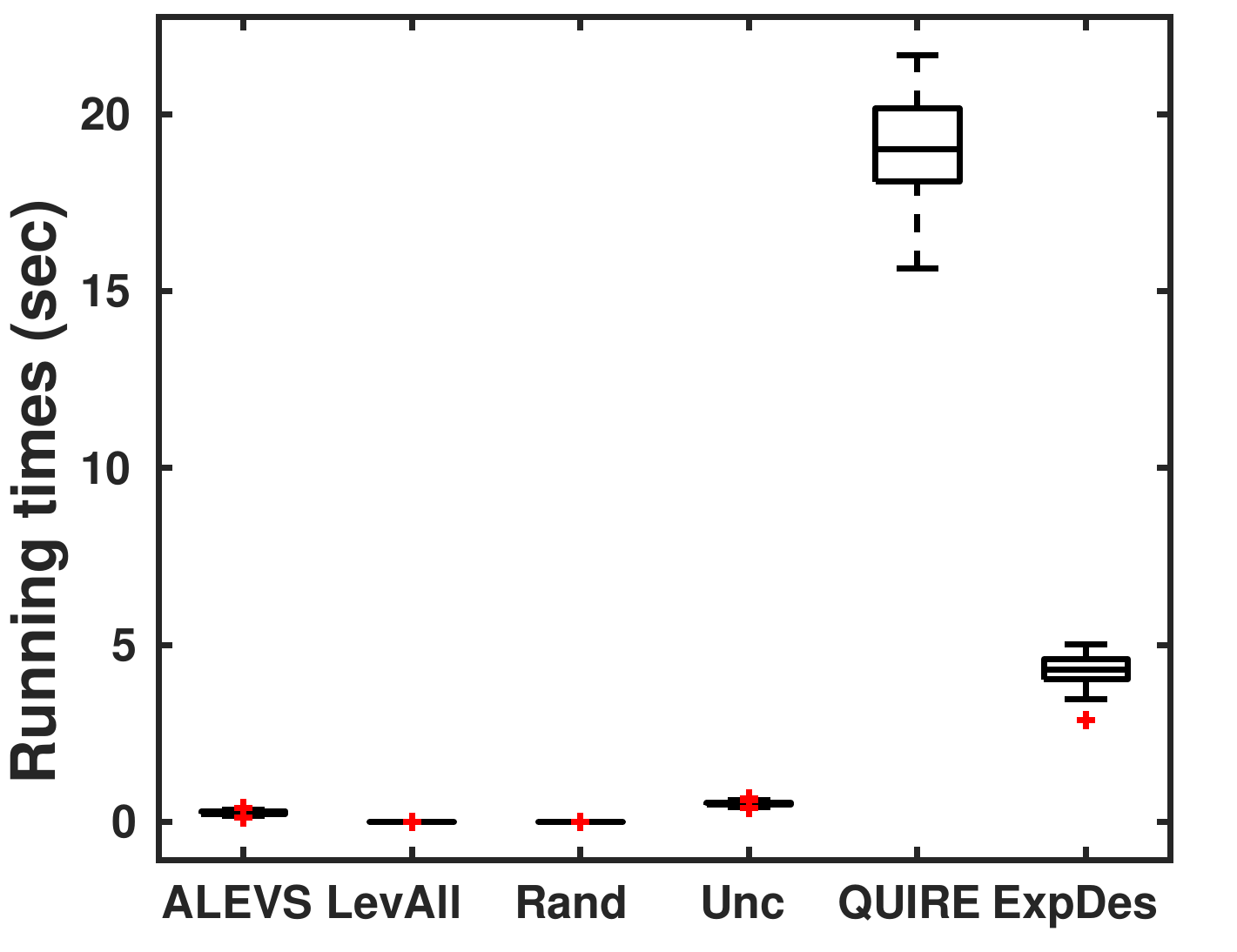}
		\label{g241c_timebox}
	\end{subfigure}
	~
	\begin{subfigure}[b]{\newW\textwidth}
		\caption{\textit{UvsV}}
		\vspace{\newSpace}
		\includegraphics[width=\textwidth]{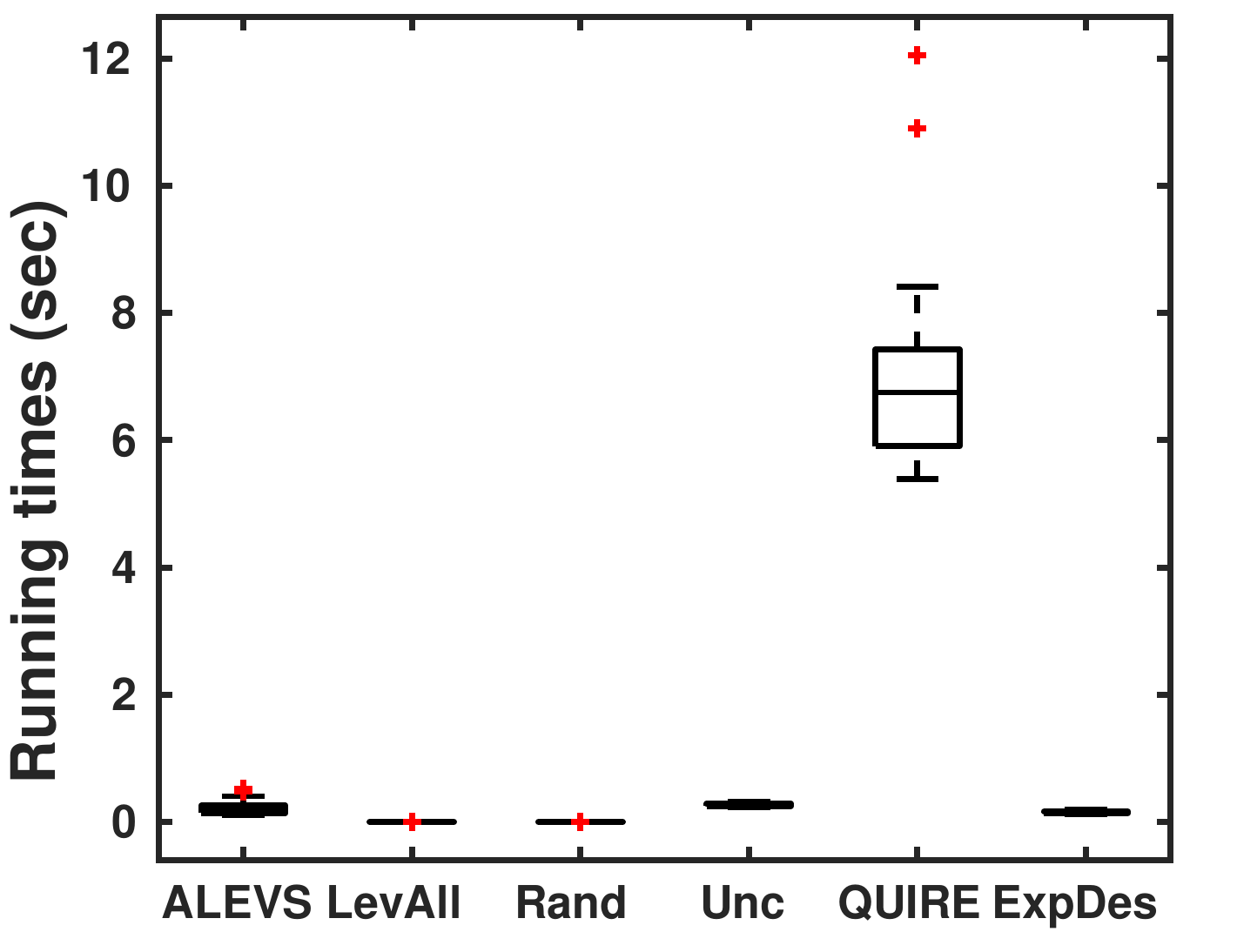}
		\label{letter2_timebox}
	\end{subfigure}
	\vskip \newSkip
	\begin{subfigure}[b]{\newW\textwidth}
		\caption{\textit{USPS}}
		\vspace{\newSpace}
		\includegraphics[width=\textwidth]{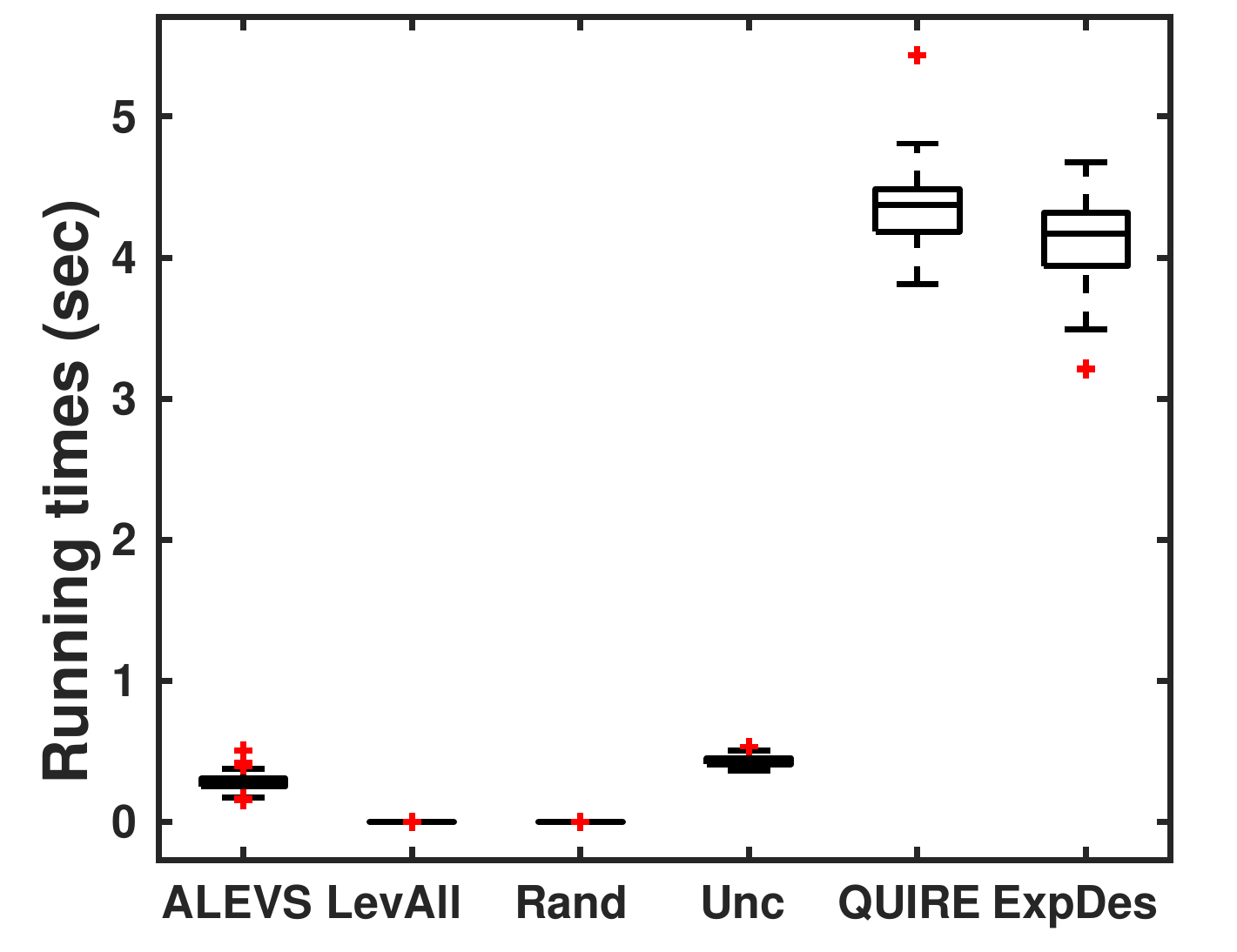}
		\label{USPS_timebox}
	\end{subfigure}
	~
	\begin{subfigure}[b]{\newW\textwidth}
		\caption{\textit{twonorm}}
		\vspace{\newSpace}
		\includegraphics[width=\textwidth]{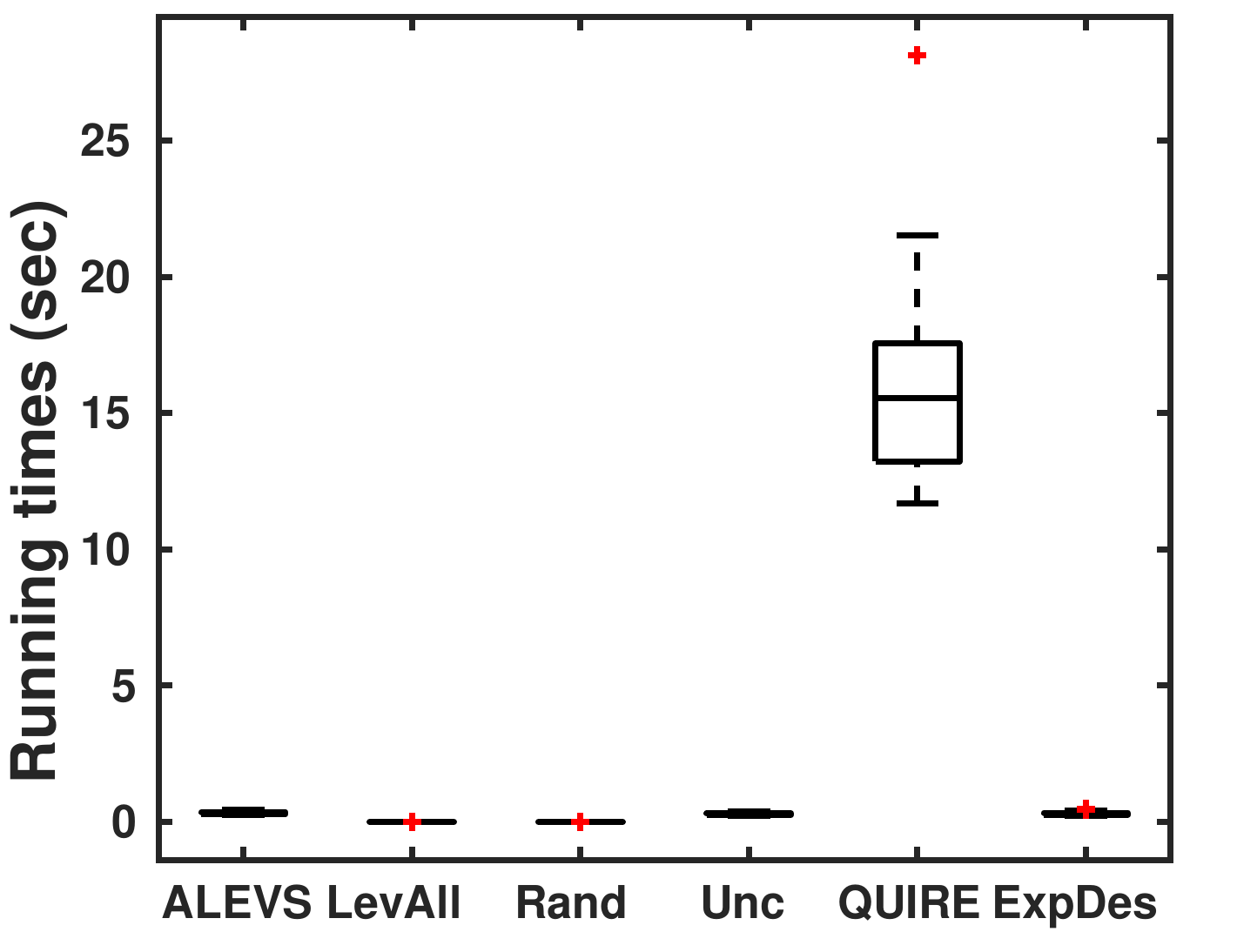}
		\label{twonorm_timebox}
	\end{subfigure}
	\vskip \newSkip
	\begin{subfigure}[b]{\newW\textwidth}
		\caption{\textit{ringnorm}}
		\vspace{\newSpace}
		\includegraphics[width=\textwidth]{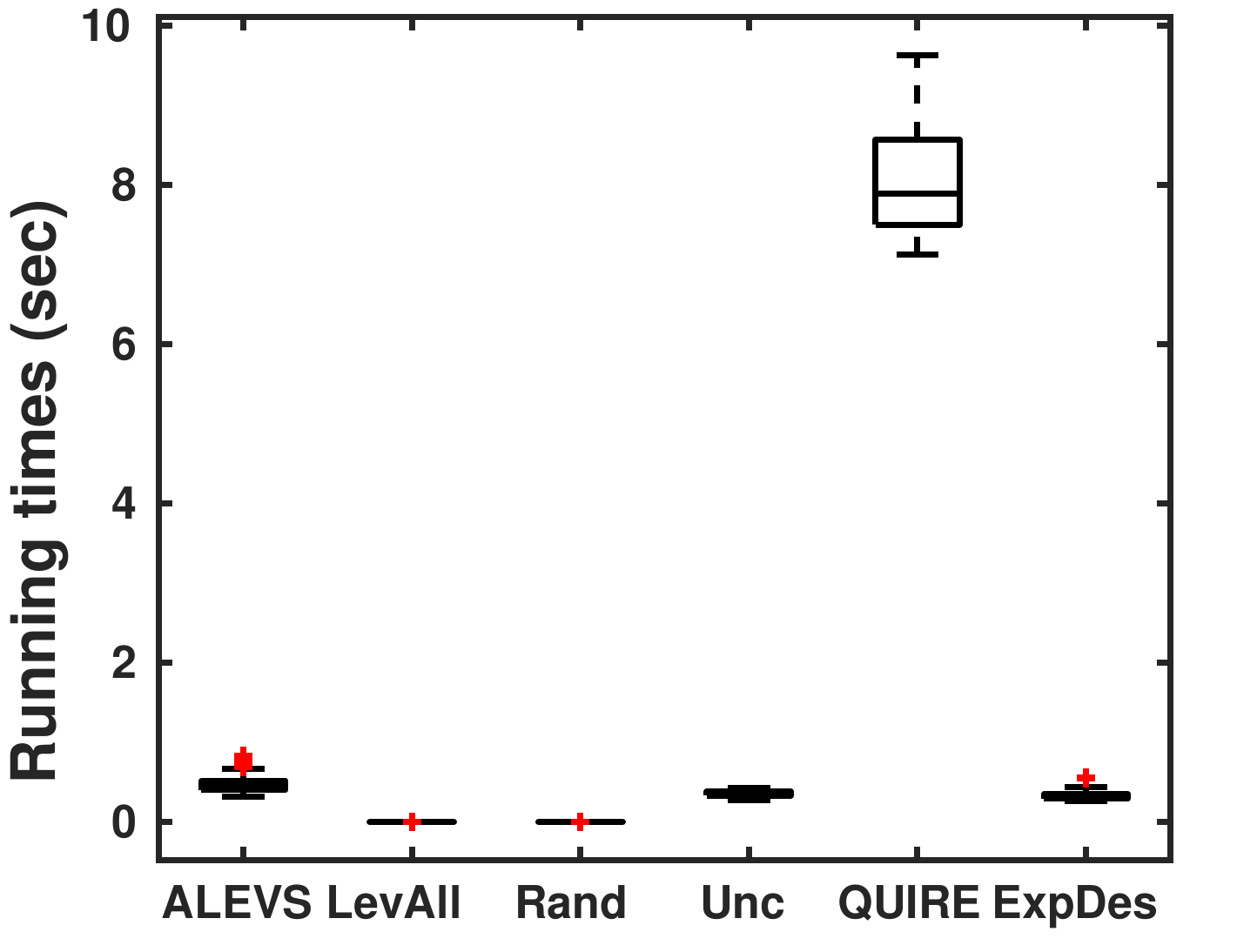}
		\label{ringnorm_timebox}
	\end{subfigure}
	~
	\begin{subfigure}[b]{\newW\textwidth}
		\caption{\textit{spambase}}
		\vspace{\newSpace}
		\includegraphics[width=\textwidth]{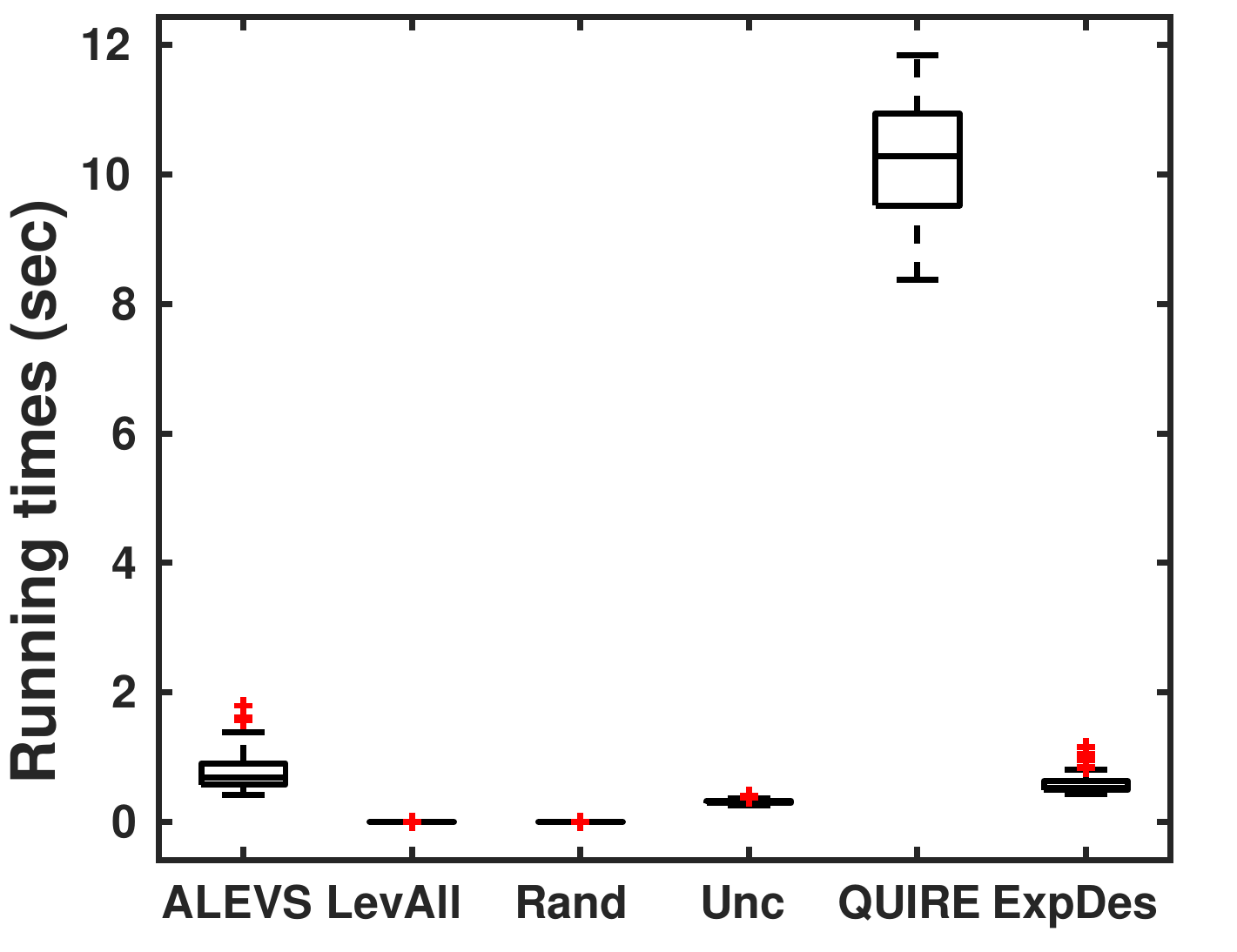}
		\label{spambase_timebox}
	\end{subfigure}
	\caption{Comparison of \texttt{ALEVS} with other methods on running times (Sequential-mode).}
	\label{figs:time-seq}
\end{figure*}
We performed the experiments in Matlab on a computer with 2.6 GHz CPU (24-core) and 64 GB of memory running Ubuntu 14.04 LTS operating system. ALEVS is as fast as almost uncertainity sampling.

\subsection{DBALEVS runtime performance}
\begin{figure*}[p]
	\centering
	\vskip \newSkip
	\begin{subfigure}[b]{\newW\textwidth}
		\caption{\textit{autos}}
		\vspace{\newSpace}
		\includegraphics[width=\textwidth]{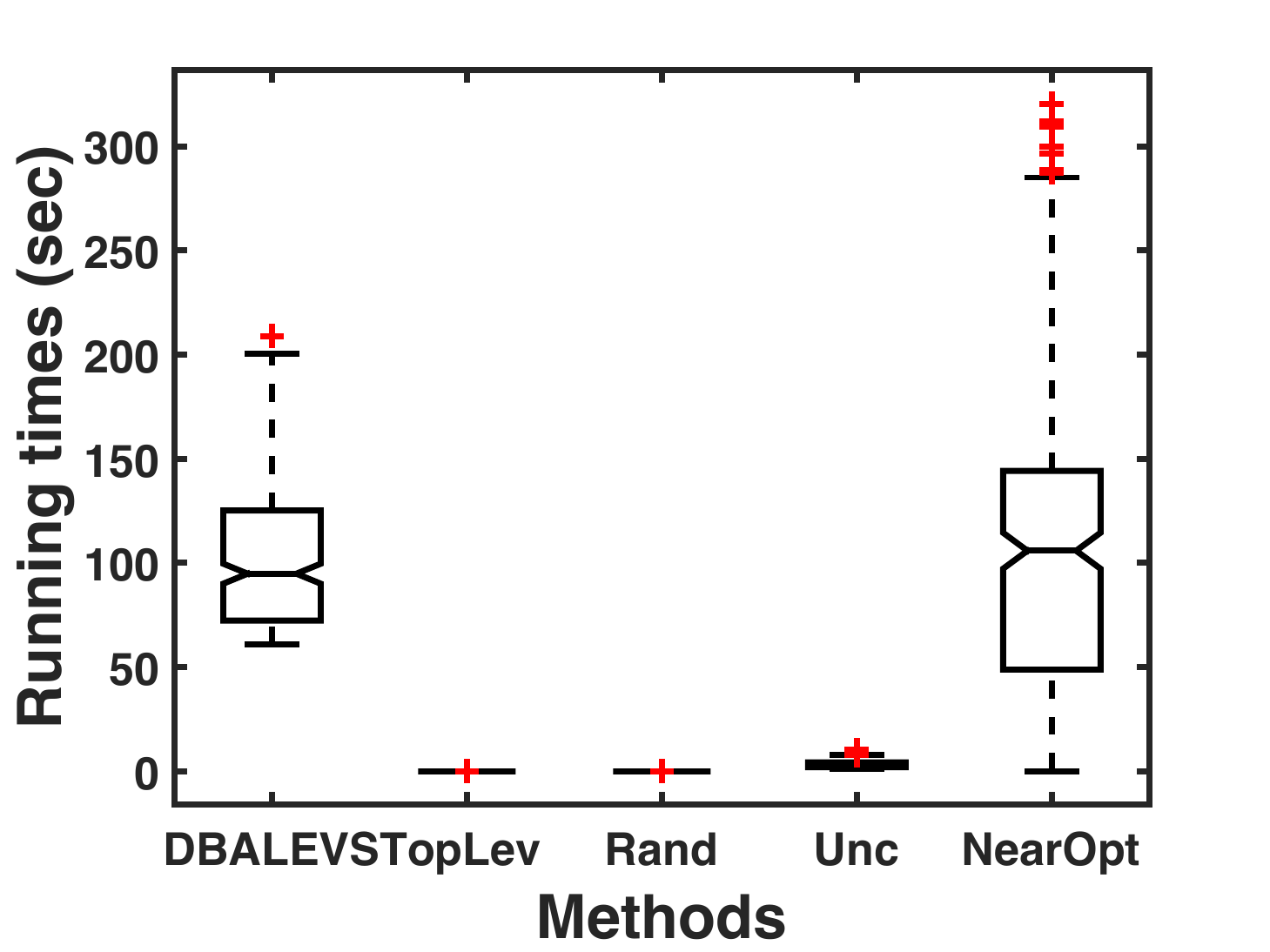}
		\label{autos_time}
	\end{subfigure}
	~
	\begin{subfigure}[b]{\newW\textwidth}
		\caption{\textit{hardware}}
		\vspace{\newSpace}
		\includegraphics[width=\textwidth]{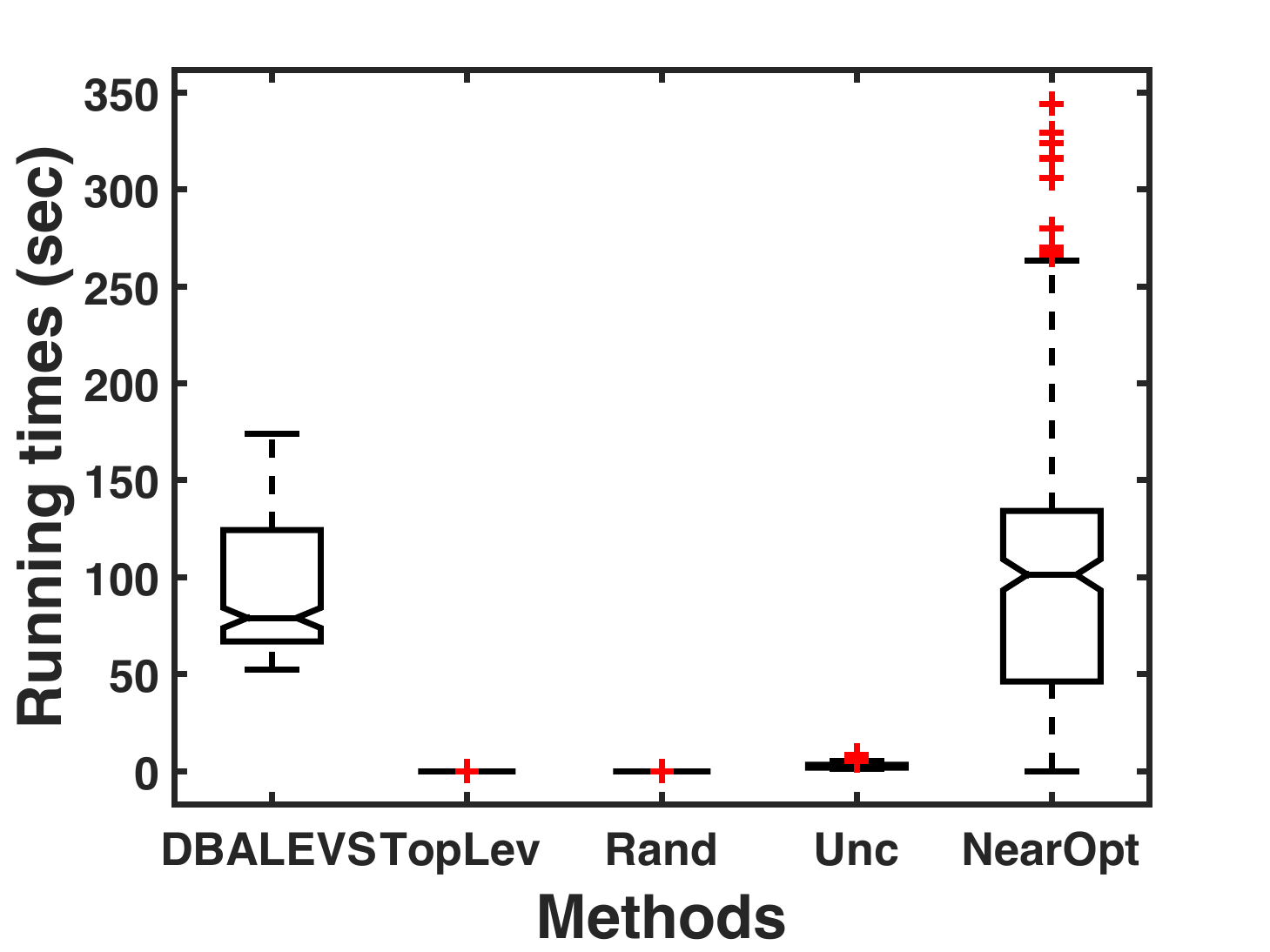}
		\label{hardware_time}
	\end{subfigure}
	\vskip \newSkip
	\begin{subfigure}[b]{\newW\textwidth}
		\caption{\textit{sport}}
		\vspace{\newSpace}
		\includegraphics[width=\textwidth]{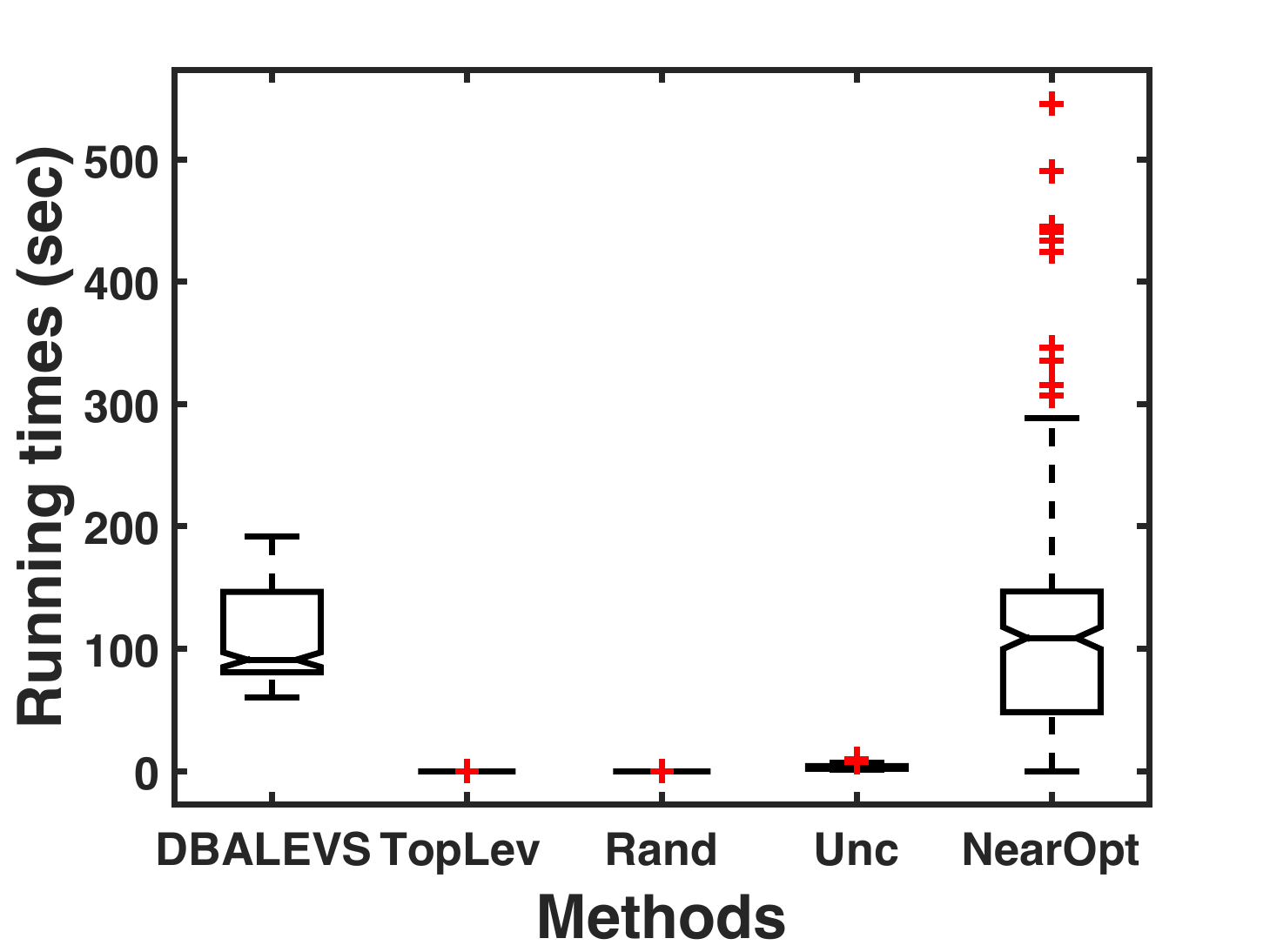}
		\label{sport_time}
	\end{subfigure}
	~
	\begin{subfigure}[b]{\newW\textwidth}
		\caption{\textit{ringnorm}}
		\vspace{\newSpace}
		\includegraphics[width=\textwidth]{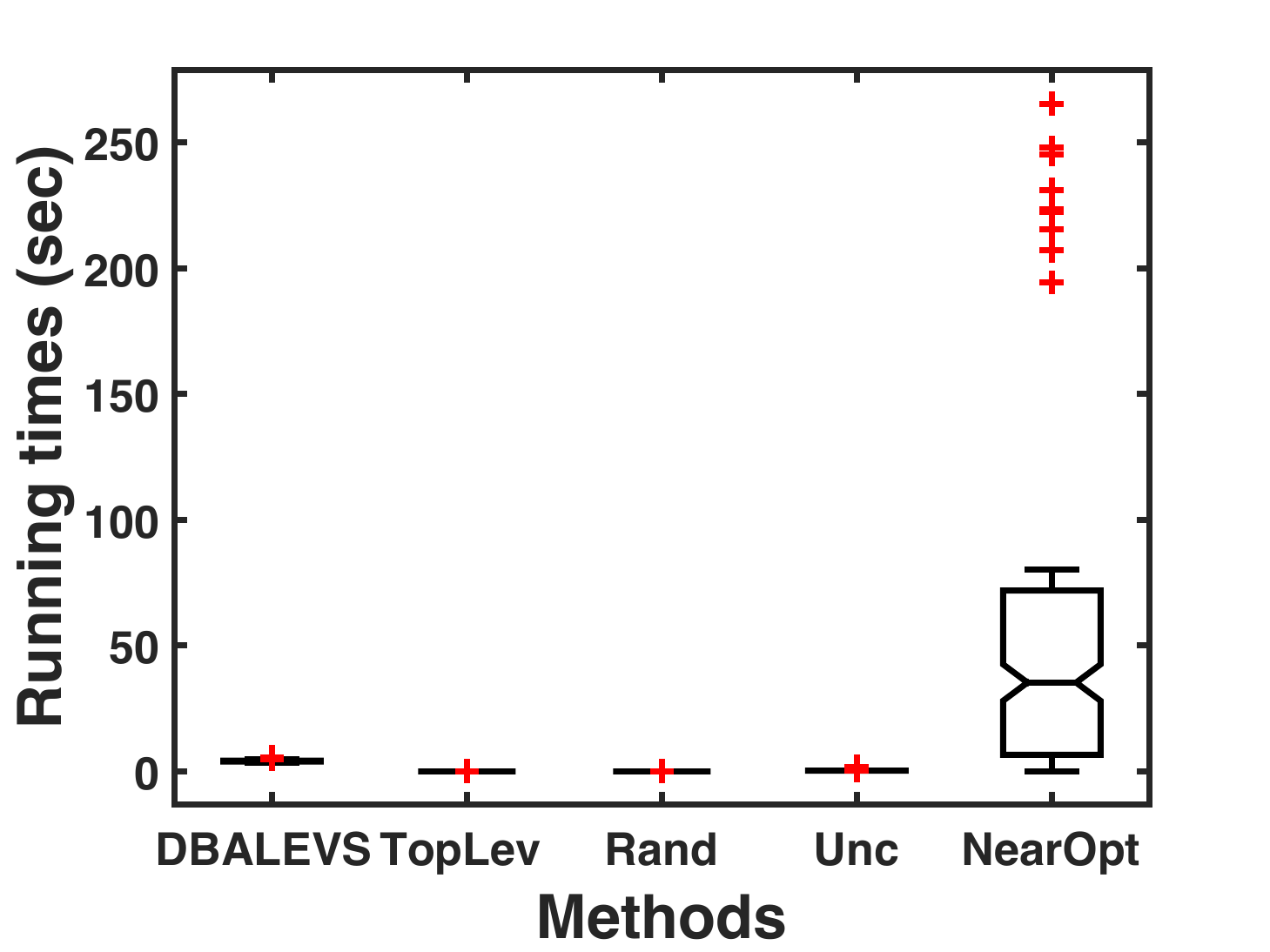}
		\label{rng_time}
	\end{subfigure}
	\vskip \newSkip
	\begin{subfigure}[b]{\newW\textwidth}
		\caption{\textit{3vs5}}
		\vspace{\newSpace}
		\includegraphics[width=\textwidth]{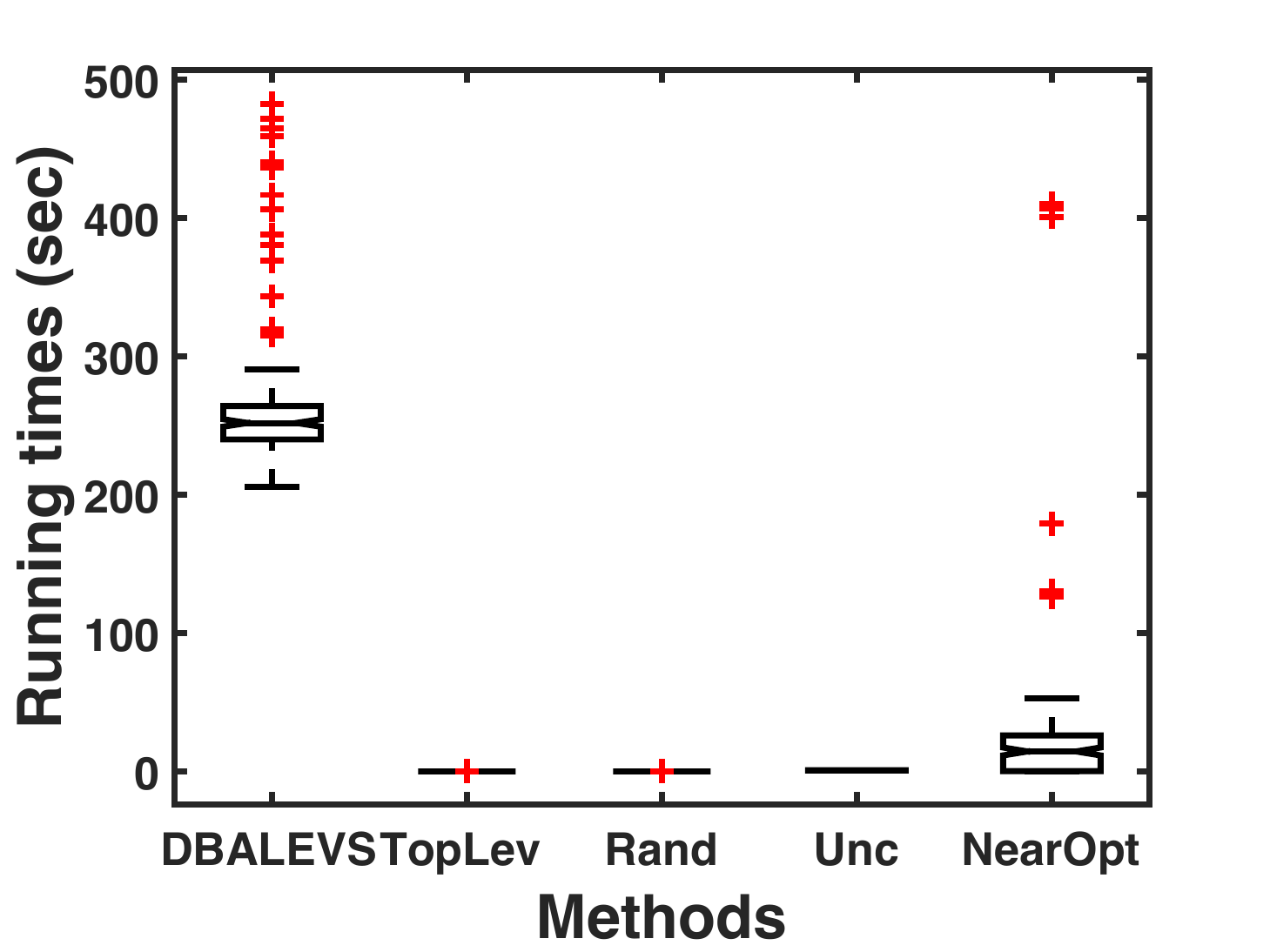}
		\label{3vs5_time}
	\end{subfigure}
	~
	\begin{subfigure}[b]{\newW\textwidth}
		\caption{\textit{4vs9}}
		\vspace{\newSpace}
		\includegraphics[width=\textwidth]{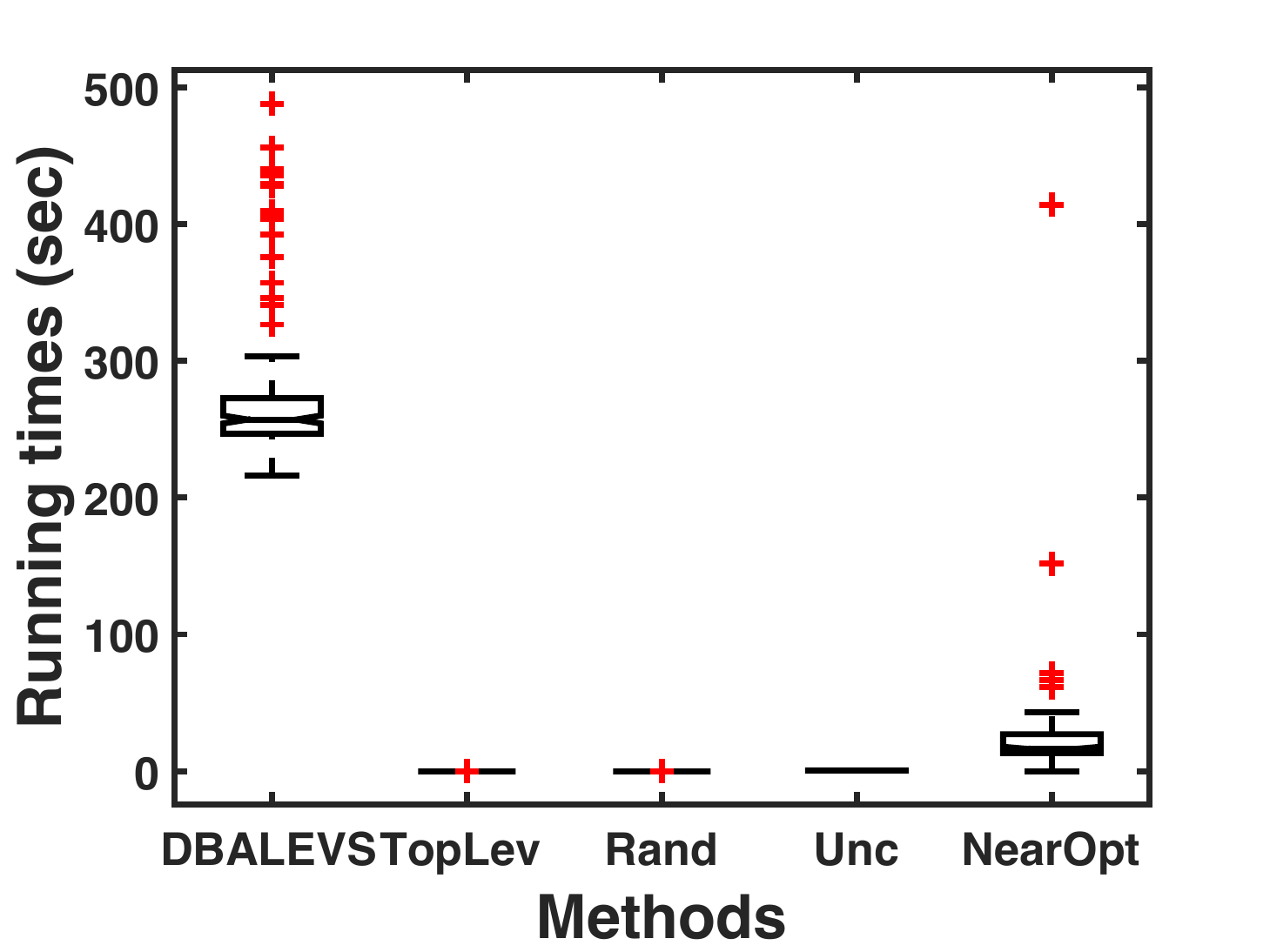}
		\label{4vs9_time}
	\end{subfigure}
	\caption{Comparison of \texttt{DBALEVS} with other methods on runtimes (Batch mode).}
	\label{figs:time-batch}
\end{figure*}

The querying step of \texttt{DBALEVS} involves the calculation of eigenvalue decomposition of the kernel matrices and the greedy maximization procedure.  Fig.~\ref{figs:time-batch} displays the average CPU times for selecting a batch in a single iteration from the unlabeled data pool. DBALEVS have comparable runtimes with the Near-Opt method. Experiments are conducted in Matlab on a computer with 2.6 GHz CPU (24-core) and 64 GB of memory running Ubuntu 14.04 LTS operating system.


\end{document}